\documentclass[lettersize,journal]{IEEEtran}

    \usepackage{jhoagg}
    \usepackage{amsmath}
    \usepackage{mathtools}
    \usepackage{amsfonts}
    \usepackage{amssymb,latexsym}
    \usepackage[psamsfonts]{eucal}
    \usepackage{amsthm}
    \usepackage{graphicx}
    \usepackage[inline]{enumitem}
    \usepackage{indentfirst}
    \usepackage{setspace}
    \usepackage{microtype}
    \usepackage{threeparttable}
    \usepackage{float}
    \usepackage{cleveref}
    \usepackage{afterpage}
    \usepackage{placeins}
    \usepackage{cancel}
    \usepackage{cite}
    \usepackage{leftidx}
    \usepackage{breqn}
    \usepackage[export]{adjustbox}
    \usepackage{comment}
    \usepackage[ruled, linesnumbered, lined]{algorithm2e}
    \usepackage[dvipsnames]{xcolor}
    \usepackage{xcolor}
    \usepackage{soul}
    \usepackage{siunitx}
    \usepackage{mwe}
    \usepackage{changepage}
    \usepackage{breqn}
    \usepackage{multirow}

    \newcommand{\wfrac}[3][3pt]{%
  \frac{\wfracterm{depth}{\dp}{#1}{#2}}{\wfracterm{height}{\ht}{#1}{#3}}%
}
\newcommand{\wfracterm}[4]{%
  \sbox0{$\displaystyle#4$}%
  \vrule width 0pt #1 \dimexpr #20+#3\relax
  \usebox{0}%
}

    \let\originalleft\left
    \let\originalright\right
    \renewcommand{\left}{\mathopen{}\mathclose\bgroup\originalleft}
    \renewcommand{\right}{\aftergroup\egroup\originalright}

    \newcounter{thm} 
    \newtheorem{theorem}[thm]{\indent Theorem}
    
    \newtheorem{assumption}{\indent Assumption}
    
    \newtheorem{proposition}{\indent Proposition}
    \newenvironment{prop}{\begin{proposition}$\!\!\!${\bf }\rm }{\end{proposition}}
    \newtheorem{lemma}{\indent Lemma}
    
    \newtheorem{corollary}{\indent Corollary}
    
    \newtheorem{definition}{\indent Definition}

    \newtheorem{examplex}{\indent Example}
    \newenvironment{example}
  {\pushQED{\qed}\examplex}
  {\popQED\endexamplex}
    
    \newtheorem{remark}{\indent Remark}

    \newtheorem{Simulation}{Simulation}

    \newtheorem{fact}{\indent Fact}
    
    \newtheorem{conjecture}{\indent Conjecture}
    
    \newtheorem{experiment}{\indent Experiment}
    
    \allowdisplaybreaks

    \renewcommand{\theenumi}{\textit{(\alph{enumi})}}
    \renewcommand{\labelenumi}{\theenumi}

    \usepackage{cite}
    \usepackage{accents}
    
    \newlength\figureheight 
    \newlength\figurewidth

    \allowdisplaybreaks
    
    \graphicspath{ {Figures_Journal2/} }
    \DeclareGraphicsExtensions{.png}
    
    \crefname{equation}{}{}
    
    \SetCommentSty{mycommfont}
    \newcommand{\ubar}[1]{\underaccent{\bar}{#1}}
    \DeclareMathAlphabet{\mathcal}{OMS}{cmsy}{m}{n} 

    \def\BibTeX{{\rm B\kern-.05em{\sc i\kern-.025em b}\kern-.08em
    T\kern-.1667em\lower.7ex\hbox{E}\kern-.125emX}}
\usepackage{balance}

\begin{document}

\title{Time-Varying Soft-Maximum Barrier Functions for Safety in Unmapped and Dynamic Environments}

\author{Amirsaeid Safari and Jesse B. Hoagg \vspace{-3ex}
\thanks{A. Safari and J. B. Hoagg are with the Department of Mechanical and Aerospace Engineering, University of Kentucky, Lexington, KY, USA. (e-mail: amirsaeid.safari@uky.edu, jesse.hoagg@uky.edu).}
\thanks{This work is supported in part by the National Science Foundation (1849213) and Air Force Office of Scientific Research (FA9550-20-1-0028).}
}
\maketitle

\begin{abstract}
We present a closed-form optimal feedback control method that ensures safety in an \textit{a priori} unknown and potentially dynamic environment.
This article considers the scenario where local perception data (e.g., LiDAR) is obtained periodically, and this data can be used to construct a local control barrier function (CBF) that models a local set that is safe for a period of time into the future. 
Then, we use a smooth time-varying soft-maximum function to compose the $N+1$ most recently obtained local CBFs into a single barrier function that models an approximate union of the $N+1$ most recently obtained local sets. 
This composite barrier function is used in a constrained quadratic optimization, which is solved in closed form to obtain a safe-and-optimal feedback control. 
We also apply the time-varying soft-maximum barrier function control to 2 robotic systems (nonholonomic ground robot with nonnegligible inertia, and quadrotor robot), where the objective is to navigate an \textit{a priori} unknown environment safely and reach a target destination. 
In these applications, we present a simple approach to generate local CBFs from periodically obtained perception data.
\end{abstract}


\section{Introduction}

Safe autonomous robotic navigation in an \textit{a priori} unmapped environment has applications in a variety of domains including search and rescue~\cite{hudson2021heterogeneous}, environmental monitoring~\cite{kress2009temporal}, and transportation~\cite{schwarting2018planning}. 
A techniques for safe navigation include potential field methods\cite{tang2010novel,kirven2021autonomous}, collision cones \cite{sunkara2019collision}, reachable sets \cite{chen2018hamilton}, and barrier function approaches~\cite{prajna2007framework,panagou2015distributed,tee2009barrier,ames2014control,ames2016control,ames2019control,wabersich2022predictive,seiler2021control, breeden2023robust}.

Control barrier functions (CBFs) provide a set-theoretic method to obtain forward invariance (e.g., safety) with respect to a specified safe set \cite{wieland2007constructive}.
CBFs can be implemented as constraints in real-time optimization-based control methods (e.g., quadratic programs) in order to guarantee forward invariance and thus, safety \cite{ames2019control}. 
A variety of extensions have been developed, including higher-relative degree CBFs (e.g., \cite{tan2021high, nguyen2016exponential,xiao2021high}), CBFs for discrete-time \cite{zeng2021safety}, and CBFs with time variation or adaptation (e.g., \cite{xiao2021high,taylor2020adaptive}).

Traditionally, CBFs are assumed to be given offline, that is, constructed offline using \textit{a priori} known information regarding the safe set.
However, in situations where the environment is unknown or changing, online construction of CBFs could enable safe navigation. 
In this scenario, the objective is to construct a CBF in real time based on the current state of the system (e.g., robot) as well as information gathered from the environment (e.g., perception information). 
For example, \cite{srinivasan2020synthesis,abuaish2023geometry} use support vector machine classifiers to construct barrier functions using LiDAR data, and \cite{long2021learning} synthesizes barrier functions using a deep neural network trained with sensor data. 
However, when new sensor data is obtained, the barrier function must be updated, which typically results in discontinuities that can be problematic for ensuring forward invariance.

Another challenge to online barrier function construction is that it can be difficult to synthesize one barrier function that models a complex environment. 
Thus, there is interest in composing multiple barrier functions into one function \cite{backupautomatica, rabiee2024closed, compositionACC, lindemann2018control,srinivasan2018control,glotfelter2017nonsmooth,glotfelter2019hybrid}.
For example, \cite{glotfelter2017nonsmooth} uses Boolean composition, which result in non-smooth barrier functions that use minimum and maximum functions. 
This approach is extended in \cite{glotfelter2019hybrid} to allow for hybrid non-smooth barrier functions, which can be useful for addressing discontinuities that arise when a barrier function is updated using new information (e.g., perception data). 
\textcolor{black}{Non-smooth time-varying barrier functions are used in \cite{long2024sensor} to define safety constraints from signed distance functions, allowing safety to be evaluated directly from sensor data in dynamic environments.
However, \cite{srinivasan2018control, glotfelter2017nonsmooth, glotfelter2019hybrid, long2024sensor} are not applicable for relative degree greater than one.}
Thus, these approaches cannot be directly applied to higher-relative-degree systems such as position-based safe sets (e.g., avoiding obstacles) and ground robots with nonnegligible inertia or unmanned aerial vehicles.
In contrast to nonsmooth compositions, \cite{backupautomatica, rabiee2024closed, compositionACC, lindemann2018control} use smooth soft-minimum and soft-maximum functions to compose multiple barrier functions into one function.
However, \cite{backupautomatica, compositionACC, rabiee2024closed, lindemann2018control} do not address online barrier function construction using real-time perception. 
\textcolor{black}{
Non-smooth barrier functions with a maximum operator are used in \cite{lutkus2024incremental} to compose local CBFs that are learned from discrete state-action pairs. 
However, \cite{lutkus2024incremental} requires a discretization and sampling of the state-space to determine safe state-action pairs, which can make it computationally intractable for high-dimensional systems.}


This article's main contribution is a new closed-form optimal feedback control that ensures safety in an \textit{a priori} unknown and potentially dynamic environment.
We consider the scenario where periodically obtained local perception data can be used to construct a local CBF that, for some time interval into the future, models a local subset of the unknown time-varying safe set.
These local CBFs can have arbitrary relative degree $r$.
We use a smooth time-varying soft-maximum construction to compose the $N+1$ most recently obtained local CBFs into a single barrier function whose zero-superlevel set approximates the union of the $N+1$ most recently obtained local subsets. 
This construction uses not only the soft maximum but also a homotopy in time to smoothly move the oldest local CBF out of the soft maximum and the newest local CBF into the soft maximum. 
This homotopy is designed to ensure that the composite soft-maximum function is $r$-times continuously differentiable in time and space.
The time-varying soft-maximum barrier function is used with a higher-relative-degree approach to construct a relaxed control barrier function (R-CBF), that is, a function that satisfies the CBF criteria on the boundary of its zero-superlevel set but not necessarily on the interior. 
This R-CBF is used in a constrained quadratic optimization, which is solved in closed form to obtain a safe-and-optimal feedback control. 
Closed-form solutions to other CBF-based optimizations appear in \cite{ames2016control,wieland2007constructive, cortez2022compatibility,rabiee2024closed}.

We also present applications of the new control method to a nonholonomic ground robot with nonnegligible inertia, and to a quadrotor aerial robot. 
These robots are equipped with sensing capability (e.g., LiDAR) that periodically detects points on obstacles that are near the robot. 
The objective is to safely navigate the environment, which may be dynamic, and reach a target destination. 
We present a simple approach to generate local CBFs from the perception data. 
Specifically, for each detected point, we construct a spheroid, whose semi-major axis is the line from the detected point to the boundary of the perception system's detection area.  
The interior of each spheroid is a region to avoid because it is not necessarily safe. 
We use a soft minimum to compose all spheroidal functions and functions that model the perception detection area.  
The zero-superlevel set of this perception-based composite soft-minimum CBF is a subset of the unknown safe set. 
Some preliminary ideas related to this work are presented in the conference article \cite{safari2023time}.


\section{Notation}

The interior, boundary, and closure of the set $\mathcal{A} \subseteq \mathbb{R}^n$ are denoted by $\mbox{int}~\mathcal{A}$, $\mbox{bd}~\mathcal{A}$, and $\mbox{cl}~\mathcal{A}$ respectively.
Let $\mathbb{N} \triangleq \{ 0, 1, 2, \ldots \}$, and let $\| \cdot \|$ denote the $2$ norm on $\BBR^n$.

Let $\zeta:[0,\infty) \times \BBR^n \to \BBR$ be continuously differentiable. 
Then, the partial Lie derivative of $\zeta$ with respect to $x$ along the vector fields of $\nu:\mathbb{R}^n \to \mathbb{R}^{n \times \ell}$ is define as $L_\nu \zeta(t,x) \triangleq \frac{\partial \zeta(t,x)}{\partial x} \nu(x)$.
In this paper, all functions are sufficiently smooth such that all derivatives that we write exist and are continuous.

A continuous function $a \colon \BBR \to \BBR$ is an \textit{extended class-$\SK$ function} if it is strictly increasing and $a(0)=0$.

\section{Soft Minimum and Soft Maximum}

Let $\kappa>0$, and consider $\mbox{softmin}_\kappa \colon \mathbb{R}^N \to \mathbb{R}$ and $\mbox{softmax}_\kappa \colon \mathbb{R}^N \to \mathbb{R}$ defined by
\begin{gather}
\mbox{softmin}_\kappa (z_1,\cdots,z_N) \triangleq -\frac{1}{\kappa}\log\sum_{i=1}^Ne^{-\kappa z_i},\label{eq:softmin}\\
\mbox{softmax}_\kappa (z_1,\cdots,z_N) \triangleq \frac{1}{\kappa}\log\sum_{i=1}^Ne^{\kappa z_i} - \frac{\log N }{\kappa},\label{eq:softmax}
\end{gather}
which are the log-sum-exponential \textit{soft minimum} and \textit{soft maximum}, respectively.
The next result relates the soft minimum to the minimum and the soft maximum to the maximum. 
The proof is in Appendix~\ref{appendix:proposition proofs}.

\begin{proposition}
\label{fact:softmin_limit}
\rm
Let $z_1,\cdots, z_N \in \mathbb{R}$. 
Then,
\begin{align}
  \min \, \{z_1,\cdots,z_N\} - \frac{\log N }{\kappa} 
    &\le \mbox{softmin}_\kappa(z_1,\cdots,z_N) \nn \\
    &\le \min \, \{z_1,\cdots,z_N\}\label{eq:softmin_inequality},
\end{align}
and
\begin{align}
    \max\,\{z_1,\cdots,z_N\}  - \frac{\log N}{\kappa}
 &\le \mbox{softmax}_\kappa(z_1,\cdots,z_N) \nn \\
 &\le \max \, \{z_1,\cdots,z_N\}\label{eq:softmax_inequality}.
\end{align}
\end{proposition}

\Cref{fact:softmin_limit} shows that $\mbox{softmin}_\kappa$ and $\mbox{softmax}_\kappa$ lower bound minimum and maximum, respectively.
\Cref{fact:softmin_limit} also shows that $\mbox{softmin}_\kappa$ and $\mbox{softmax}_\kappa$ approximate the minimum and maximum in the sense that they converge to minimum and maximum, respectively, as $\kappa \to \infty$.

The next result is a consequence of \Cref{fact:softmin_limit}.
The result shows that soft minimum and soft maximum can be used to approximate the intersection and the union of zero-superlevel sets, respectively. 
The proof is in Appendix~\ref{appendix:proposition proofs}.

\begin{proposition}
\label{prop:softmin_softmax_sets}
\rm
For $i \in \{ 1,2,\ldots,N\}$, let $\zeta_i \colon \BBR^n \to \BBR$ be continuous, and define 
\begin{gather}
\SD_i \triangleq \{ x \in \BBR^n \colon \zeta_i(x) \ge 0 \}\label{prop2.1},\\
\SX_{\kappa} \triangleq \{ x \in \BBR^n \colon \mbox{softmin}_\kappa(\zeta_1(x),\cdots,\zeta_N(x)) \ge 0 \}\label{prop2.2},\\
\SY_{\kappa} \triangleq \{ x \in \BBR^n \colon \mbox{softmax}_\kappa(\zeta_1(x),\cdots,\zeta_N(x)) \ge 0 \}\label{prop2.3}.
\end{gather}
Then, 
\begin{equation*}
    \SX_{\kappa} \subseteq \bigcap_{i=1}^N \SD_i, \qquad \SY_{\kappa} \subseteq \bigcup_{i=1}^N \SD_i.
\end{equation*}
Furthermore, as $\kappa \to \infty$, $\SX_{\kappa} \to \bigcap_{i=1}^N \SD_i$ and $\SY_{\kappa} \to \bigcup_{i=1}^N \SD_i$.
\end{proposition}

If at least one argument of the soft minimum \eqref{eq:softmin} is negative and $\kappa >0$ is large, then \eqref{eq:softmin} can involve the exponential of a large positive number, which can be poorly conditioned for numerical computation.  
A similar phenomenon can occur with \eqref{eq:softmax} 
if at least one of its arguments is positive. 
Hence, we present the next result, which provides expressions for $\mbox{softmin}_\kappa$ and $\mbox{softmax}_\kappa$ that are advantageous for numerical computation. 
These expressions are used in \Cref{sec:GroundRobot,sec:Quadrotor} for implementation.
The proof of the next result is in Appendix~\ref{appendix:proposition proofs}.

\begin{prop}\label{prop:softmin_max_numerical}
Let $z_1,\cdots, z_N \in \mathbb{R}$, and define $\ubar{z} \triangleq \min \, \{z_1,\cdots,z_N\}$ and $\bar z \triangleq \max \, \{z_1,\cdots,z_N\}$. 
Then, 
    \begin{gather*}
  \mbox{softmin}_\kappa(z_1,\cdots,z_N) =\ubar{z} + \mbox{softmin}_\kappa(z_1-\ubar{z},\cdots,z_N-\ubar{z}),\\
   \mbox{softmax}_\kappa(z_1, \cdots, z_N) = \bar{z} + \mbox{softmax}_\kappa(z_1-\bar{z}, \cdots, z_N-\bar{z}).
\end{gather*}
\end{prop}

\section{Problem Formulation}\label{sec:problem formulation}

Consider
\begin{equation}\label{eq:affine control}
\dot x(t) = f(x(t))+g(x(t)) u(t), 
\end{equation}
where $x(t) \in \mathbb{R}^n $ is the state, $x(0) = x_0 \in \mathbb{R}^n$ is the initial condition, $f: \mathbb{R}^n \to \mathbb{R}^n$ and $g: \mathbb{R}^n \to \mathbb{R}^{n \times m}$ are locally Lipschitz continuous on $\BBR^n$, and $u:[0, \infty) \to \BBR^m$ is the control.

The following definition is a relaxation of the standard CBF definition.
See \cite{ames2016control,lindemann2018control,xiao2021high} for standard CBF definitions.

\begin{definition}\label{def:RCBF}
\rm
Let $\zeta \colon [0,\infty) \times \BBR^n \to \BBR$ be continuously differentiable, and for all $t \ge 0$, define $\SD(t) \triangleq \{ x \in \BBR^n \colon \zeta(t,x) \ge 0 \}$.
Then, $\zeta$ is a \textit{relaxed control barrier function} (R-CBF) for~\cref{eq:affine control} on $\SD(t)$ if for all $(t,x) \in [0,\infty) \times  \mbox{bd }\SD(t)$,
\begin{equation}
\sup_{u \in \BBR^m} \left [ \textstyle\frac{\partial \zeta(t,x)}{\partial t} + L_f \zeta(t,x) + L_g \zeta(t,x) u \right ]  \ge 0.\label{def:RCBF.1}
\end{equation} 
\end{definition}

\Cref{def:RCBF} implies that an R-CBF need only satisfy 
\eqref{def:RCBF.1} on the boundary of its zero-superlevel.
In contrast, a CBF (e.g., \cite{ames2016control,lindemann2018control,xiao2021high}) must satisfy \eqref{def:RCBF.1} on the boundary and a related condition on the interior of $\SD(t)$. 
Under additional assumptions (e.g., $\SD(t)$ is compact), an R-CBF may also be a CBF.
CBF definitions often take the supremum over a subset of $\BBR^m$; however, \Cref{def:RCBF} is sufficient for this article.

For all $t \ge 0$, let $\SSS_{\rm{s}}(t) \subset \BBR^n$ be the \textit{safe set}, that is, the set of states that are safe at time $t$. 
The time-varying safe set $\SSS_{\rm{s}}(t)$ is not assumed to be known. 
Instead, we assume that a real-time perception/sensing system provides a subset of the safe set at update times $0,T,2T,3T,\ldots$, where $T > 0$.
In particular, for all $k \in \BBN$, we obtain perception feedback at time $kT$ in the form of a continuously differentiable function $b_k : \mathbb{R}^n \to \mathbb{R}$ such that its zero-superlevel set
\begin{equation}\label{eq:Sk}
\SSS_{k} \triangleq \{x \in \mathbb{R}^n : b_k(x) \ge 0\},
\end{equation}
is nonempty, contains no isolated points, and is a subset of $\SSS_\rms(kT)$.
The perception information $b_k$ can be obtained from a variety of perception synthesis methods (e.g., \cite{srinivasan2020synthesis,long2021learning, abuaish2023geometry}).
Sections \ref{sec:GroundRobot} and \ref{sec:Quadrotor} demonstrate a simple approach for using LiDAR and the soft minimum to synthesize $b_k$.
We assume that there exist known positive integers $N$ and $r$ such that for all $k \in \mathbb{N}$, the following hold:

\begin{enumerate}[leftmargin=0.9cm]
	\renewcommand{\labelenumi}{(A\arabic{enumi})}
	\renewcommand{\theenumi}{(A\arabic{enumi})}

\item\label{con1}
For all $t \in [kT, (k+N+1)T]$, $\SSS_{k} \subseteq \SSS_{\rm{s}}(t)$.

\item\label{con3}
For all $x \in \BBR^n$, $L_gb_k(x)=L_gL_fb_k(x)=\cdots=L_gL_f^{r-2}b_k(x)=0$.

\item\label{con4}
For almost all $x \in \cup_{i=k}^{k+N} \SSS_{i}$, $L_gL_f^{r-1}b_k(x) \ne 0$.

\item\label{con2}
$\SSS_{k} \cap \left ( \cup_{i=k+1}^{k+N} \SSS_i \right )$ is nonempty and contains no isolated points.

\end{enumerate}

Assumption~\ref{con1} implies that $\SSS_k$ is a subset of the safe set $\SSS_\rms(t)$ for $(N+1)T$ time units into the future. 
For example, if $N=1$, then \ref{con1} implies that $\SSS_k$ is a subset of the safe set over the interval $[kT,(k+2)T]$. 
The choice $N=1$ is appropriate if the safe set is changing quickly. 
On the other hand, if the safe set is changing more slowly, then $N>1$ may be appropriate. 
If the safe set is time-invariant, then \ref{con1} is satisfied for any positive integer $N$ because $\SSS_k \subseteq \SSS_\rms$. 
The next section presents a construction for a time-varying R-CBF that uses the $N+1$ most recent perception feedback $b_k,\ldots,b_{k-N}$.

Assumptions~\ref{con3} and~\ref{con4} implies that $b_k$ has relative degree $r$ with respect to \eqref{eq:affine control} almost everywhere on $\cup_{i=k}^{k+N} \SSS_{i}$.
Assumption~\ref{con2} is a technical condition on the perception data that the zero-superlevel set of $b_{k}$ is connected to the union of the zero-superlevel sets of $b_{k-1},\ldots,b_{k-N}$.

\textcolor{black}{
Section~\ref{sec:GroundRobot} demonstrates one method to satisfy \ref{con1} and \ref{con2} in a dynamic environment using a known upper bound $\bar{v} > 0$ on the speed of dynamic obstacles.
In this case, $\SSS_k$ is synthesized from LiDAR, and it excludes the possible positions that an obstacle detected at time $t = kT$ could occupy during the time interval $[kT,(k+N+1)T]$. 
More specifically, the upper bound $\bar{v}$ on speed is used to ensure that $\SSS_k$ does not intersect with a disk of radius $\bar{v}(N+1)T$ centered at each LiDAR point where an obstacle is detected. 
For $N=1$, the radius of the exclusion disks is $2\bar{v}T$. 
Since the perception system cannot update arbitrarily quickly (i.e., $T$ is lower bounded), it follows that the exclusion disk must have a larger radius to accommodate faster obstacles. 
Furthermore, if the exclusion disk is too large, then $\SSS_k$ may be empty, which violates \ref{con2}. 
Thus, the perception hardware impose limits on the allowable rate of change of a time-varying environment.
}

Next, we consider the cost function $J:[0,\infty) \times \mathbb{R}^n \times \mathbb{R}^m \to \mathbb{R}$ defined by
\begin{equation}\label{eq:J}
    J(t,x,u) \triangleq \tfrac{1}{2}u^\rmT Q(t,x)u + c(t,x)^\rmT u, 
\end{equation}
where $Q:[0,\infty) \times \mathbb{R}^n \to \BBR^m$ and $c: [0,\infty) \times \mathbb{R}^n \to \mathbb{R}^m$ are continuous in $t$ and locally Lipschitz in $x$; and for all $(t,x) \in [0,\infty) \times \BBR^n$, $Q(t,x)$ is positive definite.

The objective is to design a feedback control such that for each $t \ge 0$, the cost $J(t, x(t), u)$ is minimized subject to the safety constraint that $x(t) \in \SSS_{\rm{s}}(t)$.

A special case of this problem is the minimum-intervention problem, where there is desired control $u_\rmd : [0,\infty) \times \BBR^n \to \BBR^m$, which is designed to satisfy performance requirements but may not account for safety. 
In this case, the objective is to design a feedback control such that the minimum-intervention cost $\| u - u_\rmd(t, x(t)) \|^2$ is minimized subject to the constraint that $x(t) \in \SSS_{\rm{s}}(t)$.  
The cost $\| u - u_\rmd(t, x) \|^2$ is minimized by $u_\rmd(t,x)$, which is equal to the minimizer of \eqref{eq:J} where $Q(t,x) = I_m$ and $c(t,x) = -u_\rmd(t, x)$. 
Thus, the minimum-intervention problem is addressed by letting $Q(t,x) = I_m$ and $c(t,x) = -u_\rmd(t, x)$. 
All statements in this paper that involve the subscript $k$ are for all $k \in \mathbb{N}$.

\section{Time-Varying Soft Maximum R-CBF} 
\label{sec:Method}

This section presents a method for constructing a time-varying R-CBF from the real-time perception feedback $b_k$. 

Let $\eta:\mathbb{R} \to [0,1]$ be $r$-times continuously differentiable such that the following conditions hold:

\begin{enumerate}[leftmargin=0.9cm]
	\renewcommand{\labelenumi}{(C\arabic{enumi})}
	\renewcommand{\theenumi}{(C\arabic{enumi})}

 \item\label{con:con1_g}
For all $t\in (-\infty,0]$, $\eta(t) = 0$.

\item\label{con:con2_g}
For all $t\in [1,\infty)$, $\eta(t) = 1$.

\item\label{con: con4_g}
For all $i \in \{1,\cdots,r\}$, $\left . {\frac{{\rm d}^i \eta(t)}{{\rm d}t^i}}\right |_{t=0} = 0$ and $\left . {\frac{{\rm d}^i \eta(t)}{{\rm d}t^i}} \right |_{t=1} = 0$. 
\end{enumerate}

The following example provides a choice for $\eta$ that satisfies \ref{con:con1_g}--\ref{con: con4_g} for any positive integer $r$.

\begin{example}\label{ex:g}\rm
Let $\nu \ge 1$, and let
\begin{equation}\label{eq:smoothstep}
\eta(t) =
\begin{cases}
        0, & t \in (-\infty,0), \\
        \left(\nu t\right)^{r+1} \sum_{j=0}^{r} \binom{r+j}{j}\binom{2r+1}{r-j}(-\nu t)^j, & t \in \left [ 0,\frac{1}{\nu} \right ], \\
        1, & t \in \left (\frac{1}{\nu},\infty \right ),
\end{cases}
\end{equation}
which satisfies \ref{con:con1_g}--\ref{con: con4_g}.
\Cref{fig:eta} is a plot of \Cref{eq:smoothstep} with $r=2$ for different $\nu$. 
Note that $\nu$ is a tuning parameter, which influences the rate at which $\eta$ transitions from zero to one. 
\end{example}

\begin{figure}[t!]
\center{\includegraphics[width=0.47\textwidth,clip=true,trim= 0.42in 0.25in 1.1in 0.6in] {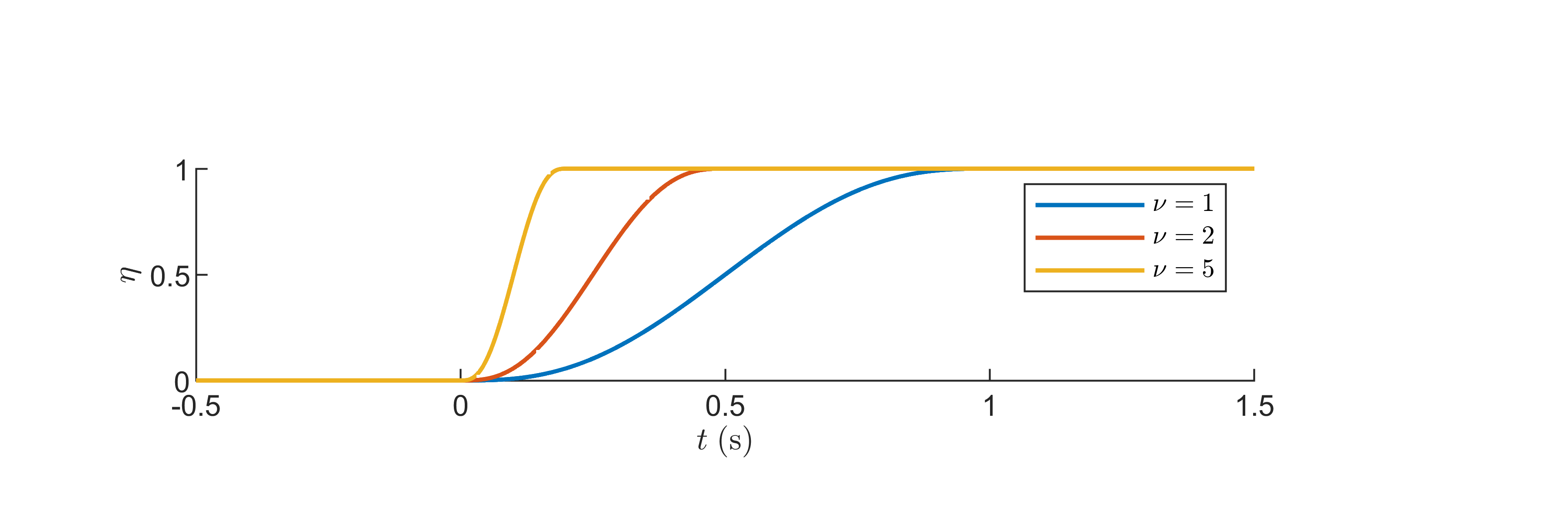}}
\caption{$\eta$ given by Example~\ref{ex:g} with $r=2$.}\label{fig:eta}
\end{figure}

The next example provides another choice for $\eta$ that satisfies \ref{con:con1_g}--\ref{con: con4_g} in the case where $r \in \{ 1 ,2 \}$.

\begin{example}\label{ex:g2}\rm
Assume $r \in \{ 1 ,2 \}$. 
Let $\nu \ge 1$, and let 
\begin{equation}\nn
\eta(t) =
\begin{cases}
        0, & t \in (-\infty,0), \\
        \nu t - \frac{1}{2 \pi}\sin 2\pi \nu t, & t \in \left [0,\frac{1}{\nu} \right ], \\
        1, & t \in \left (\frac{1}{\nu},\infty \right ),
\end{cases}
\end{equation}
which satisfies \ref{con:con1_g}--\ref{con: con4_g}.
\end{example}

Next, let $\kappa > 0$, and consider $\psi_0:[0, \infty) \times \mathbb{R}^n  \to \mathbb{R}$ such that for all $k\in\BBN$ and all $(t,x) \in [kT, (k+1)T) \times \BBR^n$,
\begin{align}\label{eq:softmax_h}
\psi_0(t,x) &\triangleq \mbox{softmax}_\kappa \bigg (b_{k-1}(x), \cdots, b_{k-N+1}(x),\nonumber\\
&\qquad \eta\left(\textstyle\frac{t}{T}-k\right)b_k(x)  + \left[ 1-\eta\left(\textstyle\frac{t}{T}-k\right)\right]b_{k-N}(x) \bigg ),
\end{align}
where for $i\in \{1,2,\ldots,N\},$ $b_{-i}\triangleq b_0$.
The function $\psi_0$ is constructed from the $N+1$ most recently obtained perception feedback functions $b_k,\ldots,b_{k-N}$. 
If $N=1$, then 
\begin{equation}
\psi_0(t,x) = \eta\left(\textstyle\frac{t}{T}-k\right)b_k(x)  + \left[ 1-\eta\left(\textstyle\frac{t}{T}-k\right)\right ] b_{k-1}(x).    \label{eq:psi0_N=1}
\end{equation}
In this case, $\psi_0$ is a convex combination of $b_k$ and $b_{k-1}$, where $\psi_0$ smoothly transitions from $b_{k-1}$ to $b_{k}$ over the interval $[kT,(k+1)T]$. 
In general, the final argument of \eqref{eq:softmax_h} involves the convex combination of $b_k$ and $b_{k-N}$, which allows for the smooth transition from $b_{k-N}$ to $b_{k}$ over the interval $[kT,(k+1)T]$. 
In other words, this convex combination is the mechanism by which the newest perception feedback $b_{k}$ is smoothly incorporated into $\psi_0$ and the oldest perception feedback $b_{k-N}$ is smoothly removed from $\psi_0$.

The zero-superlevel set of $\psi_0$ is defined by
\begin{equation} \label{eq:safe set final}
    \SC_0(t) \triangleq \{x \in \mathbb{R}^n \colon \psi_0(t,x) \geq 0 \}.
\end{equation}
\Cref{prop:softmin_softmax_sets} implies that $\SC_0(t)$ is a subset of the union of the zero-superlevel sets of the arguments of \eqref{eq:softmax_h}.
The next result is an immediate consequence of \Cref{prop:softmin_softmax_sets}.

\begin{proposition} \label{prop:C0_at_kT}\rm
For all $k \in \BBN$, $\SC_0(kT) \subseteq \cup_{i=k-N}^{k-1} \SSS_i$.
\end{proposition}

\textcolor{black}{\Cref{prop:softmin_softmax_sets} also implies that for sufficiently large $\kappa>0$, $\SC_0(t)$ approximates the union of the zero-superlevel sets of the arguments of the soft maximum in \eqref{eq:softmax_h}.
Specifically, $\SC_0(kT) \to \cup_{i=k-N}^{k-1} \SSS_i$ as $\kappa \to \infty$. 
In other words, if $\kappa >0$ is sufficiently large, then $\psi_0$ is a lower-bound approximation of 
\begin{align*}
\psi_*(t,x) &\triangleq \max \, \{ b_{k-1}(x), \cdots, b_{k-N+1}(x),  \nn\\
&\qquad \eta\left(\textstyle\frac{t}{T}-k\right)b_k(x)  + \left[ 1-\eta\left(\textstyle\frac{t}{T}-k\right)\right]b_{k-N}(x)\}.
\end{align*}
In fact, \Cref{fact:softmin_limit} shows that $\psi_*-\psi_0 \le \frac{\log N}{\kappa}$, which implies that $\psi_0$ converges to $\psi_*$ as $\kappa \to \infty$.
However, if $\kappa >0$ is large, then $\textstyle\frac{\partial\psi_0(t,x)}{\partial x}$ has large magnitude at points where $\psi_*$ is not differentiable. 
Thus, choice of $\kappa$ is a trade off between the magnitude of $\textstyle\frac{\partial\psi_0(t,x)}{\partial x}$ and the how well $\SC_0(t)$ approximates the zero-superlevel set of $\psi_*$.}

\Cref{prop:C0_at_kT} and \ref{con1} imply that $\SC_0$ is a subset of the safe set $\SSS_{\rm{s}}$ at sample times $kT$.
In fact, the convex combination of $b_k$ and $b_{k-N}$ used in \eqref{eq:softmax_h} ensures that $\SC_0$ is a subset of $\SSS_{\rm{s}}$ not only at sample times but also for all time between samples. 
The following result demonstrates this property. 

\begin{proposition}\label{fact:S(t)}\rm
Assume \ref{con1} is satisfied. 
Then, for all $k \in \BBN$ and all $t \in [kT, (k+1)T]$, $\SC_0(t) \subseteq  \cup_{i=k-N}^k \SSS_i  \subseteq \SSS_{\rm{s}}(t)$. 
\end{proposition}

\begin{proof}[\indent Proof]
Let $k_1 \in \mathbb{N}$, $t_1 \in [k_1T, (k_1+1)T]$, and $x_1 \in \SC_0(t_1)$. 
Assume for contradiction that $x_1 \not \in \cup_{i=k_1-N}^{k_1} \SSS_i$, which implies that for $i \in \{ k_1-N,k_1-N+1,\ldots,k_1 \}$, $b_i(x_1) < 0$. 
Since, in addition, $\eta (t_1/T-k_1) \in [0,1]$, it follows from~\eqref{eq:softmax_h} and \Cref{fact:softmin_limit} that $\psi_0(t_1,x_1) < 0$. 
Next, \eqref{eq:safe set final} implies that $x_1 \not \in \SC_0(t_1)$, which is a contradiction. 
Thus, $x_1 \in \cup_{i=k_1-N}^{k_1} \SSS_i$, which implies that $\SC_0(t_1) \subseteq \cup_{i=k_1-N}^{k_1} \SSS_i$, which combined with \ref{con1} implies that $\SC_0(t_1) \subseteq \cup_{i=k_1-N}^{k_1} \SSS_i \subseteq \SSS_\rms(t_1)$.
\end{proof}

\Cref{fact:S(t)} demonstrates that $\SC_0(t) \subseteq \SSS_{\rm{s}}(t)$; however, we may not be able to use $\psi_0$ directly as a candidate R-CBF because $\psi_0$ is not necessarily relative degree one.
The next result concerns the relative degree of $\psi_0$.
The proof is in Appendix~\ref{appendix:proposition proofs}. 
The following notation is needed.
For all $k\in\BBN$ and all $(t,x) \in [kT, (k+1)T) \times \BBR^n$,
define
\begin{align}
\mu_{0}(t,x) &\triangleq \frac{\eta\left(\textstyle\frac{t}{T}-k\right)}{N} \exp \kappa \Big ( \eta\left(\textstyle\frac{t}{T}-k\right)b_k(x)  \nn\\
&\qquad + \left[ 1-\eta\left(\textstyle\frac{t}{T}-k\right)\right ] b_{k-N}(x) - \psi_0(t,x) \Big ), \label{eq:mu0}\\
\mu_{N}(t,x) &\triangleq \frac{1-\eta\left(\textstyle\frac{t}{T}-k\right)}{N} \exp \kappa \Big ( \eta\left(\textstyle\frac{t}{T}-k\right)b_k(x)  \nn\\
&\qquad + \left[ 1-\eta\left(\textstyle\frac{t}{T}-k\right)\right ] b_{k-N}(x) - \psi_0(t,x) \Big ),\label{eq:muN}
\end{align}
and for all $j \in \{ 1,\ldots,N-1 \}$, define
\begin{equation}
    \mu_{j}(t,x) \triangleq \textstyle \frac{1}{N} \exp \kappa \Big ( b_{k-j}(x) - \psi_0(t,x) \Big ).\label{eq:muj}
\end{equation}

\begin{proposition}\label{prop.h}\rm
Assume \ref{con3} is satisfied.
Then, the following statements hold: 
\begin{enumerate}
\item \label{prop.h.1}
$\psi_0$ is $r$-times continuously differentiable on $[0,\infty) \times \BBR^n$.

\item \label{prop.h.2}
For all $(t,x) \in [0,\infty) \times \BBR^n$, 
\begin{equation*}
L_g\psi_0(t,x) = L_gL_f\psi_0(t,x) = \cdots = L_gL_f^{r-2}\psi_0(t,x)=0.    
\end{equation*}

\item \label{prop.h.3}
For all $k\in\BBN$ and all $(t,x) \in [kT, (k+1)T) \times \BBR^n$,
\begin{equation}
    L_gL_f^{r-1}\psi_0(t,x) = \sum_{j=0}^{N} \mu_{j}(t,x) L_gL_f^{r-1} b_{k-j}(x).\label{eq:LgLfr-1psi0}
\end{equation}

\item \label{prop.h.4}
For all $(t,x) \in [0,\infty) \times \BBR^n$, $\mu_0(t,x),\ldots,\mu_N(t,x) \ge 0$, and $\sum_{j=0}^N \mu_j(t,x) = 1$.

\end{enumerate}
\end{proposition}

\Cref{prop.h} shows that if $L_gL_f^{r-1}\psi_0$ is nonzero, then $\psi_0$ has relative degree $r$ with respect to \eqref{eq:affine control}.
In addition, \ref{prop.h.3} and \ref{prop.h.4} of \Cref{prop.h} show that $L_gL_f^{r-1}\psi_0$ is a convex combination of $L_gL_f^{r-1}b_k,\ldots,L_gL_f^{r-1}b_{k-N}$, which are nonzero from \ref{con4}.

Since $\psi_0$ has relative degree $r$ provided that $L_gL_f^{r-1}\psi_0$ is nonzero, we use a higher-order approach to construct a candidate R-CBF from $\psi_0$.
Specifically, for all $i \in \{0, 1, \cdots, r-2\}$, let $\alpha_i:\mathbb{R} \to \mathbb{R}$ be an $(r-1-i)$-times continuously differentiable extended class-$\mathcal{K}$ function, 
and consider $\psi_i:[0, \infty) \times \mathbb{R}^n  \to \mathbb{R}$ defined by
\begin{equation}\label{eq:HOCBF}
\psi_{i+1}(t,x) \triangleq \textstyle\frac{\partial \psi_{i}(t,x)}{\partial t} + L_f \psi_{i}(t,x) +\alpha_{i}(\psi_{i}(t,x)).
\end{equation}
For all $i \in \{1,\cdots,r-1\}$, the zero-superlevel set of $\psi_i$ is defined by
\begin{equation}\label{eq:HOCBF set}
    \SC_i(t) \triangleq \left \{x \in \mathbb{R}^n \colon \psi_{i}(t,x) \ge 0 \right \}.
\end{equation}
Next, we define
\begin{equation}
\SC(t) \triangleq \bigcap_{i=0}^{r-1} \SC_i(t), \label{eq:Common set}
\end{equation}
and assume that $\SC(0)$ is nonempty and contains no isolated points. 
Since $\SC(t) \subseteq \SC_0(t)$, it follows from \Cref{fact:S(t)} that for all $t \ge 0$, $\SC(t) \subseteq \SSS_\rms(t)$.

The next result shows that $L_g \psi_{r-1} = L_g L_f^{r-1} \psi_0$, which implies that if $\psi_0$ has relative degree $r$, then $\psi_{r-1}$ has relative degree one.
The proof is in Appendix~\ref{appendix:proposition proofs}.

\begin{proposition}\label{prop.psi_i_properties}\rm
Assume \ref{con3} is satisfied.
Then, for all $(t,x) \in [0,\infty) \times \BBR^n$, $L_g \psi_{r-1}(t,x) = L_g L_f^{r-1} \psi_0(t,x)$ and $L_g\psi_0(t,x) = L_g\psi_1(t,x) = \cdots = L_g \psi_{r-2}(t,x)=0$.
\end{proposition}

Next, we define
\begin{align}
    \SB(t) &\triangleq \left \{ x \in \mbox{bd } \SC_{r-1}(t) \colon \textstyle\frac{\partial \psi_{r-1}(t,x)}{\partial t} + L_f \psi_{r-1}(t,x) \leq 0 \right \}, \label{eq: bd cr-1}     
\end{align}
which is the set of all states on the boundary of the zero-superlevel set of $\psi_{r-1}$ such that if $u=0$, then the time derivative of $\psi_{r-1}$ is nonpositive along the trajectories of \eqref{eq:affine control}.
We assume that on $u$ directly affects the time derivative of $\psi_{r-1}$ on $\SB$. 
Specifically, we make the following assumption:

\begin{enumerate}[leftmargin=0.9cm]
\renewcommand{\labelenumi}{(A\arabic{enumi})}
\renewcommand{\theenumi}{(A\arabic{enumi})}
\setcounter{enumi}{4}

\item\label{con6}
For all $(t,x) \in [0,\infty) \times \SB(t)$, $L_g\psi_{r-1}(t,x) \ne 0$.
\end{enumerate}

Assumption \ref{con6} is related to the constant-relative-degree assumption often invoked with CBF approaches. 
In this work, $L_g \psi_{r-1}$ is assumed to be nonzero on $\SB$, which is a subset of the boundary of the zero-superlevel set of $\psi_{r-1}$.

\begin{remark}\label{remark:Lgpsi}\rm
\Cref{prop.psi_i_properties,prop.h} imply that $L_g \psi_{r-1}$ is nonzero if and only if the convex combination \eqref{eq:LgLfr-1psi0} is nonzero. 
This observation can be used to obtain sufficient conditions for \ref{con6} that may be verifiable in certain applications. 
For example, this observation implies that \ref{con6} is satisfied if for all $k\in\BBN$ and all $(t,x) \in [kT, (k+1)T) \times \SB(t)$, the convex hull of $L_gL_f^{r-1}b_k,\ldots,L_gL_f^{r-1}b_{k-N}$ does not contain the origin. 
Note that this condition on the convex hull of $L_gL_f^{r-1}b_k,\ldots,L_gL_f^{r-1}b_{k-N}$ is sufficient for \ref{con6}, but in general, it is not necessary. 
In the case where $N=1$, it follows from \eqref{eq:psi0_N=1} and \Cref{prop.h,prop.psi_i_properties} that 
\begin{align*}
    L_g \psi_{r-1}(t,x) &= \eta\left(\textstyle\frac{t}{T}-k\right) L_g L_f^{r-1} b_k(x)  \\
    &\qquad + \left[ 1-\eta\left(\textstyle\frac{t}{T}-k\right)\right ] L_g L_f^{r-1}  b_{k-1}(x).
\end{align*}
In this case, \ref{con6} is satisfied if for all $k\in\BBN$ and all $(t,x) \in [kT, (k+1)T) \times \SB(t)$, the origin is not on the line segment connecting $L_g L_f^{r-1} b_k(x)$ and $L_g L_f^{r-1} b_{k-1}(x)$.
\end{remark}

The next result demonstrates that that $\psi_{r-1}$ is an R-CBF.

\begin{proposition}\label{prop:CFI}
\rm
Assume \ref{con6} is satisfied. 
Then, $\psi_{r-1}$ is an R-CBF for~\cref{eq:affine control} on $\SC_{r-1}(t)$. 
\end{proposition}

\begin{proof}[\indent Proof]
Let $(t_1,x_1) \in [0,\infty)\times \BBR^n$ be such that $x_1 \in \mbox{bd } \SC_{r-1}(t_1)$. 
We consider two case: $x_1 \in \SB(t_1)$ and $x_1 \not \in \SB(t_1)$. 
First, consider the case where $x_1 \in \SB(t_1)$, and \ref{con6} implies that $L_g \psi_{r-1}(t_1,x_1) \ne 0$. 
Thus, $\sup_{u \in \BBR^m} [ \frac{\partial \psi_{r-1}(t,x)}{\partial t} + L_f \psi_{r-1}(t,x) + L_g \psi_{r-1}(t,x) u ] = \infty$, which confirms \eqref{def:RCBF.1} for $x_1 \in \SB(t_1)$.
Next, consider the case where $x_1 \not \in \SB(t_1)$, and it follows from \eqref{eq: bd cr-1} that $\frac{\partial \psi_{r-1}(t,x)}{\partial t}|_{(t,x)=(t_1,x_1)} + L_f \psi_{r-1}(t_1,x_1) > 0$, which confirms \eqref{def:RCBF.1} for $x_1 \not \in \SB(t_1)$.
\end{proof}

Since $\psi_{r-1}$ is an R-CBF on $\SC_{r-1}(t)$, Nagumo's theorem \cite[Corollary~4.8]{blanchini2008set} suggests that it may be possible to construct a control such that $\SC_{r-1}(t)$ is forward invariant with respect to the closed-loop dynamics. 
However, $\SC_{r-1}(t)$ is not necessarily a subset of the safe set $\SSS_\rms(t)$.
Therefore, the next result is useful because it shows that forward invariance of $\SC_{r-1}(t)$ implies forward invariance of $\SC(t)$, which is a subset of the safe set $\SSS_\rms(t)$.

\begin{proposition}\label{prop:Forward_Invariant}
\rm
Consider \eqref{eq:affine control}, where \ref{con3} is satisfied and $x_0 \in \SC(0)$. 
Assume there exists $\bar{t} \in (0,\infty]$ such that for all $t \in [0,\bar{t})$, $x(t) \in \SC_{r-1}(t)$.
Then, for all $t \in [0,\bar{t})$, $x(t) \in \SC(t)$.
\end{proposition}

\begin{proof}[\indent Proof]
We use induction on $i$ to show that for all $i\in \{1,2,\ldots,r\}$ and all $t \in \ST \triangleq [0,\bar{t})$, $x(t) \in \SC_{r-i}(t)$. 
First, note that for all $t \in [0,\bar{t})$, $x(t) \in \SC_{r-1}(t)$, which confirms the $i=1$ case. 
Next, let $\ell \in \{1,2,\ldots,r-1\}$, and assume for induction that for all $t \in \ST$, $x(t) \in \SC_{r-\ell}(t)$. 
Thus, \Cref{eq:HOCBF set} implies that for all $t \in \ST$, $\psi_{r-\ell}(t,x(t)) \ge 0$.
Since \Cref{prop.psi_i_properties} implies that $L_g \psi_{r-(\ell+1)} = 0$, it follows from \eqref{eq:HOCBF} that $\psi_{r-\ell}(t,x(t)) = \dot \psi_{r-(\ell+1)}(t,x(t)) + \alpha_{r-(\ell+1)}(\psi_{r-(\ell+1)}(t,x(t)))$. 
Thus, for all $t \in \ST$, $\dot \psi_{r-(\ell+1)}(t,x(t)) + \alpha_{r-(\ell+1)}(\psi_{r-(\ell+1)}(t,x(t))) \ge 0$. 
Since, in addition, $x_0 \in \SC(0) \subseteq \SC_{r-(\ell+1)}(0)$, it follows from \cite[Lemma~2]{glotfelter2017nonsmooth}  that for all $t \in \ST$, $\psi_{r-(\ell+1)}(t,x(t)) \ge 0$. 
Hence, \Cref{eq:safe set final,eq:HOCBF set} imply that for all $t \in \ST$, $x(t) \in \SC_{r-(\ell+1)}$, which confirms the $i=\ell+1$ case.
Finally, since for all $i\in \{1,2,\ldots,r\}$ and all $t \in \ST$, $x(t) \in \SC_{r-i}(t)$, it follows from \eqref{eq:Common set} that all $t \in \ST$, $x(t) \in \SC(t)$.
\end{proof}

\Cref{prop:Forward_Invariant} shows that if $\psi_{r-1}$ is nonnegative along the solution trajectory of \eqref{eq:affine control}, then the state remains in a subset of the safe set. 
The next section presents an optimal and safe control that is synthesized using a constraint that ensures $\psi_{r-1}$ is nonnegative along the closed-loop trajectories, which, in turn, from \Cref{prop:Forward_Invariant} ensures that $\SC(t)$ is invariant under the closed-loop dynamics.

\section{Optimal and Safe Control in an \textit{A Priori} Unknown and Potentially Dynamic Safe Set}
\label{sec:control}

This section uses the soft-maximum R-CBF $\psi_{r-1}$ to construct a closed-form optimal control that guarantees safety. 
Consider the R-CBF safety constraint function $b \colon [0,\infty) \times \BBR^n \times \BBR^m \times \BBR \to \BBR$ given by
\begin{align}
b(t,x,\hat{u},\hat{\mu}) &\triangleq \textstyle\frac{\partial \psi_{r-1}(t,x)}{\partial t} + L_f \psi_{r-1}(t,x) + L_g\psi_{r-1}(t,x)\hat{u} \nn \\
        & \qquad +\alpha(\psi_{r-1}(t,x)) + \hat{\mu}\psi_{r-1}(t,x),
        \label{eq:safety_constraint}
\end{align}
where $\hat u$ is the control variable; $\hat \mu$ is a slack variable; and $\alpha \colon \BBR \to \BBR$ is locally Lipschitz and nondecreasing such that $\alpha(0)=0$. 
Next, let $\gamma > 0$, and consider the cost function  $\SJ \colon [0,\infty) \times \BBR^n \times \BBR^m \times \BBR \to \BBR$ given by
\begin{align}
 \SJ(t,x,\hat{u},\hat{\mu}) &\triangleq J(t, x,\hat{u}) + \frac{\gamma}{2} \hat{\mu}^2\nn\\
 &= \frac{1}{2} \hat u^\rmT Q(t,x) \hat u + c(t,x)^\rmT \hat u + \frac{\gamma}{2} {\hat \mu}^2, \label{eq:SJ}
\end{align} 
which is equal to the cost \eqref{eq:J} plus a term that is quadratic in the slack variable $\hat{\mu}$.
The objective is to synthesize $(\hat u,\hat \mu)$ that minimizes the cost $\SJ(t,x,\hat{u},\hat{\mu})$ subject to the R-CBF safety constraint $b(t,x,\hat{u},\hat{\mu}) \ge 0$.

The slack-variable term $\hat{\mu}\psi_{r-1}$ in \eqref{eq:safety_constraint} ensures the constraint $b(t,x,\hat{u},\hat{\mu}) \ge 0$ is feasible. 
Specifically, the slack-variable term ensures that for all $(t,x) \in [0,\infty) \times \BBR^n$, there exists $\hat u$ and $\hat \mu$ such that $b(t,x,\hat{u},\hat{\mu}) \ge 0$.
Since $\psi_{r-1}$ is an R-CBF, it follows from \Cref{def:RCBF} that 
for all $(t,x) \in [0,\infty) \times \mbox{bd } \SC_{r-1}(t)$, there exists $\hat u$ such that the constraint without the slack variable is satisfied (i.e., $b(t,x,\hat{u},0) \ge 0$). 
However, the constraint without the slack variable may not be feasible for $x \not \in \mbox{bd } \SC_{r-1}(t)$. 
CBF-based constraints often do not include the slack term $\hat{\mu}\psi_{r-1}$.
In this case, additional assumptions are typically needed to guarantee feasibility. 
Note that the slack variable approach is also used in \cite{lindemann2018control,ames2014control,nguyen2016exponential,rabiee2024closed, compositionACC}.


For each $(t,x) \in [0,\infty) \times \BBR^n$, the minimizer of $\SJ(t,x,\hat u,\hat \mu)$ subject to $b(t,x,\hat{u},\hat{\mu}) \ge 0$ can be obtained from the first-order necessary conditions for optimality.
The derivation is in Appendix~\ref{appendix:CformControl} and yields the control $u_* \colon [0,\infty) \times \BBR^n$ defined by
\begin{equation}
\mspace{-10mu} u_* (t,x) \triangleq -Q(t,x)^{-1} \Big (c(t,x) - \lambda(t,x) L_g\psi_{r-1}(t,x)^\rmT \Big ), \label{eq:uclose}  
\end{equation}
where $\lambda,\omega,d \colon [0, \infty) \times \BBR^n \to \BBR$ are defined by 
\begin{align}
        \lambda(t,x) &\triangleq \max \, \left \{ 0,
            \tfrac{-\omega(t,x)}{d(t,x)} \right \} \label{eq:ulambda} \\
%
        %
\omega(t,x) &\triangleq b(t,x,-Q(t,x)^{-1}c(t,x),0) \nn \\ 
&= \textstyle\frac{\partial \psi_{r-1}(t,x)}{\partial t} + L_f \psi_{r-1}(t,x) + \alpha(\psi_{r-1}(t,x))\nn \\ &\qquad 
- L_g\psi_{r-1}(t,x)Q(t,x)^{-1}c(t,x), \label{eq:omegabar}\\
d(t,x) &\triangleq L_g \psi_{r-1}(t,x)Q(t,x)^{-1} L_g \psi_{r-1}(t,x)^\rmT \nn \\ &\qquad 
+ \gamma^{-1}\psi_{r-1}(t,x)^2. \label{eq:dxt}
\end{align}
Furthermore, the slack variable that satisfies the first-order necessary conditions for optimality is $\mu_* \colon [0,\infty) \times \BBR^n \to \BBR$ defined by
\begin{equation}
    \mu_*(t,x) \triangleq  \tfrac{\psi_{r-1}(t,x) \lambda(t,x)}\gamma. \label{eq:mu_close}
\end{equation}

The control $u_*$ and slack variable $\mu_*$ depend on the user-selected parameter $\gamma >0$ and the user-selected functions $\alpha$ and $\alpha_{0},\ldots,\alpha_{r-2}$. 
The extended class-$\SK$ functions $\alpha_{0},\ldots,\alpha_{r-2}$ are similar to those used in typical higher-order CBF approaches. 
These functions influence the size and shape of the forward invariant set $\SC$ as well as the aggressiveness of the control.
The parameter $\gamma >0$ weights the slack-variable term in the cost \eqref{eq:SJ} and influences the aggressiveness of the control $u_*$. 
If $\gamma > 0$ is small (e.g., as $\gamma \to 0$), then the optimal control $u_*$ is aggressive in that it allows the state to get close to the boundary of $\SC$.
Specifically, if $\gamma$ is small, then for all $x$ outside a small neighborhood of $\SB$, the optimal control $u_*$ is approximately equal to $-Q^{-1}c$, which is the unconstrained minimizer of $J$. 
On the other hand, if $\gamma>0$ is large (e.g., as $\gamma \to \infty$), then 
$u_*$ is approximately equal to the control generated by the quadratic program with zero slack variable unless nonzero slack is required for the safety constraint to be feasible (which occurs if $\omega(t,x) < 0$ and $L_g \psi_{r-1}(t,x) \approx 0$). 
Thus, large $\gamma$ effectively disables the slack unless it is required for feasibility.
If $\gamma>0$ is large, then the aggressiveness of $u_*$ is determined by $\alpha$, which is similar to the class-$\SK$ function used in a standard CBF-based constraint.

The next result demonstrates that $u_*$ and $\mu_*$ satisfy the safety constraint (i.e., $b(t,x,u_*(t,x),\mu_*(t,x)) \ge 0$).
The result also provides useful properties regarding the continuity of $u_*$, $\mu_*$, and $\lambda$ given by \Cref{eq:uclose,eq:ulambda,eq:omegabar,eq:dxt,eq:mu_close}.

\begin{proposition}\label{prop:ClosedForm}
\rm
Assume \ref{con6} is satisfied. 
Then, the following statements hold: 

\begin{enumerate}

\item \label{propE} 
For all $(t,x) \in [0,\infty) \times \BBR^n$, $b(t,x,u_*(t,x),\mu_*(t,x)) = \max \, \{0,\omega(t,x)\}$
    
\item \label{propC} 
$u_*$, $\mu_*$, and $\lambda$ are continuous on $[0, \infty) \times \BBR^n$.

\item\label{propD} 
Assume that $\psi_{r-1}^{\prime}$ is locally Lipschitz in $x$ on $\mathcal{D} \subseteq \mathbb{R}^{n}$. 
Then, $u_*$, $\mu_*$, and $\lambda$ are locally Lipschitz in $x$ on $\mathcal{D}$. 

\item \label{propA.1} 
Let $(t_1,x_1) \in [0,\infty)\times \BBR^n$ be such that $\omega(t_1,x_1) \le 0$.
Then, $d(t_1,x_1) > 0$.

\end{enumerate}
\end{proposition}

\begin{proof}[\indent Proof]

To prove \ref{propA.1}, assume for contradiction that $d(t_1,x_1) =0$. 
Since $\gamma >0$, and $Q(t_1,x_1)$ is positive definite, it follows from \eqref{eq:dxt} that $L_g \psi_{r-1}(t_1,x_1)=0$ and $\psi_{r-1}(t_1,x_1)=0$. 
Thus, \eqref{eq:omegabar} implies 
\begin{equation}\label{eq:omega_t1x1}
\omega(t_1,x_1) = \left . \textstyle\frac{\partial \psi_{r-1}(t,x_1)}{\partial t} \right |_{t=t_1} + L_f \psi_{r-1}(t_1,x_1).
\end{equation}
Since $L_g \psi_{r-1}(t_1,x_1)=0$ and $\psi_{r-1}(t_1,x_1)=0$, it follows from 
\ref{con6}, \eqref{eq: bd cr-1}, and \eqref{eq:omega_t1x1} that $\omega(t_1,x_1) > 0$, which is a contradiction. 
Therefore, $d(t_1,x_1) \ne 0$, which implies $d(t_1,x_1) > 0$.

To prove \ref{propC}, it follows from \eqref{eq:HOCBF} and \ref{prop.h.1} of \Cref{prop.h} that $\psi_{r-1}$ is continuously differentiable on $[0,\infty) \times \BBR^n$. 
Thus, $\frac{\partial \psi_{r-1}}{\partial t}$, $\psi_{r-1}$, $L_f \psi_{r-1}$ and $L_g \psi_{r-1}$ are continuous on $[0, \infty) \times \BBR^n$.
Since, in addition, $Q^{-1}$ and $c$ are continuous on $[0, \infty) \times \BBR^n$, it follows from \eqref{eq:omegabar} that $\omega$ is continuous on $[0, \infty) \times \BBR^n$.  
Since $\omega$ is continuous on $[0, \infty) \times \BBR^n$, it follows from  \ref{propA.1} and \Cref{eq:ulambda} that $\lambda$ is continuous on $[0, \infty) \times \BBR^n$, which combined with \Cref{eq:mu_close,eq:uclose} implies that $u_*$ and $\mu_*$ are continuous on $[0, \infty) \times \BBR^n$.

To prove \ref{propD}, note that $f$, $g$ are locally Lipschitz on $\BBR^n$ and  $Q^{-1}$, $c$, and $\psi_{r-1}$ are locally Lipschitz in $x$ on $\BBR^{n}$. Since, in addition, $\psi_{r-1}^{\prime}$ is locally Lipschitz in $x$ on $\mathcal{D} $, it follows from \Cref{eq:dxt,eq:omegabar,eq:safety_constraint} that $\omega$ and $d$ are locally Lipschitz in $x$ on $\mathcal{D}$.
Thus, \Cref{eq:ulambda} implies that $\lambda$ is locally Lipschitz in $x$ on $\mathcal{D}$, which combined with \Cref{eq:uclose,eq:mu_close} implies that $u_*$ and $\mu_*$ are locally Lipschitz in $x$ on $\mathcal{D}$.

To prove \ref{propE}, substituting \Cref{eq:mu_close,eq:uclose} into \Cref{eq:safety_constraint} and using \Cref{eq:omegabar,eq:dxt} yields
\begin{align*}
b(t,x,u_*,\mu_*) &= \textstyle\frac{\partial \psi_{r-1}}{\partial t} + L_f \psi_{r-1} +\alpha(\psi_{r-1}) - L_g\psi_{r-1}Q^{-1}c  \nn\\
&\qquad +\lambda \left ( L_g \psi_{r-1} Q^{-1} L_g \psi_{r-1}^\rmT + \gamma^{-1}\psi_{r-1}^2 \right ) \nn\\
&= \omega +\lambda d,
\end{align*}
where the arguments $(t,x)$ are omitted for brevity. 
Then, substituting \eqref{eq:ulambda} yields \ref{propE}. 
\end{proof}

Part~\ref{propE} of \Cref{prop:ClosedForm} demonstrates that the control $u_*$ and slack variable $\mu_*$ satisfy the R-CBF safety constraint; however, \Cref{prop:ClosedForm} does not address optimality. 
The next result demonstrates global optimality, specifically, $(u_*(t,x),\mu_*(t,x))$ is the unique global minimizer of $\SJ(t,x,\hat u,\hat \mu)$ subject to $b(t,x,\hat u,\hat \mu) \ge 0$.

\begin{theorem}\label{thm:global minimizer}
\rm
Assume \ref{con6} is satisfied. 
Let $t \ge 0$ and $x \in \BBR^n$.
Furthermore, let $u \in \BBR^m$ and $\mu \in \BBR$ be such that $b(t,x,u,\mu) \ge 0$ and $(u,\mu) \ne (u_*(t,x),\mu_*(t,x))$.
Then, 
\begin{equation} \label{eq:Global_Min}
    \SJ(t,x,u,\mu) > \SJ(t,x,u_*(t,x),\mu_*(t,x)).
\end{equation}
\end{theorem}

\begin{proof}[\indent Proof]

Define 
$\Delta \SJ \triangleq \SJ(t,x,u,\mu) - \SJ(t,x,u_*(t,x),$ $\mu_*(t,x))$,   
and it follows from \eqref{eq:SJ} that 
\begin{equation}\label{eq:Delta_SJ}
\Delta \SJ = \frac{1}{2} u^\rmT Q u + c^\rmT u + \frac{\gamma}{2} {\mu}^2  - \frac{1}{2} u_*^\rmT Q u_* - c^\rmT u_* - \frac{\gamma}{2} {\mu_*}^2,
\end{equation}
where $(t,x)$ are omitted for brevity. 
Adding and subtracting $\frac{1}{2} u_*^\rmT Q u_* - u^\rmT Q u_* + \frac{\gamma}{2} {\mu_*}^2 - \gamma \mu \mu_*$ to the right-hand side of \eqref{eq:Delta_SJ} yields
\begin{align}
\Delta \SJ &= \frac{1}{2} u^\rmT Q u - u^\rmT Q u_* + \frac{1}{2} u_*^\rmT Q u_*  \nn\\
&\qquad + \frac{\gamma}{2} {\mu}^2 - \gamma \mu \mu_* + \frac{\gamma}{2} {\mu_*}^2 + \Psi \nn\\
&= \frac{1}{2} (u-u_*)^\rmT Q (u-u_*) + \frac{\gamma}{2} (\mu-\mu_*)^2 +\Psi, \label{eq:Delta_SJ.2}
\end{align}
where
\begin{equation}
\Psi \triangleq c^\rmT u  - c^\rmT u_* - u_*^\rmT Q u_* + u^\rmT Q u_* - \gamma {\mu_*}^2 + \gamma \mu \mu_* . \label{eq:Psi}
\end{equation}
Substituting \Cref{eq:uclose,eq:mu_close} into \eqref{eq:Psi} yields 
\begin{align}
\Psi &= \lambda L_g\psi_{r-1} Q^{-1} c + \lambda L_g\psi_{r-1} u - \lambda^2 L_g\psi_{r-1} Q^{-1} L_g\psi_{r-1}^\rmT \nn\\
&\qquad - \lambda^2 \gamma^{-1} \psi_{r-1}^2 + \lambda \mu \psi_{r-1}\nn
\end{align}
and using \eqref{eq:dxt} and \eqref{eq:ulambda} implies that $\Psi =\lambda \left ( L_g\psi_{r-1} Q^{-1} c + L_g\psi_{r-1} u + \omega + \mu \psi_{r-1} \right )$.
Thus, \Cref{eq:safety_constraint,eq:omegabar} imply that $\Psi = \lambda b(t,x,u,\mu)$, which combined with \eqref{eq:Delta_SJ.2} yields
\begin{equation}
\Delta \SJ = \frac{1}{2} (u-u_*)^\rmT Q (u-u_*) + \frac{\gamma}{2} (\mu-\mu_*)^2 +\lambda b(t,x,u,\mu). \label{eq:Delta_SJ.3}
\end{equation}
Part~\ref{propA.1} of \Cref{prop:ClosedForm} and \eqref{eq:ulambda} imply that $\lambda \ge 0$. 
Since, in addition, $b(t,x,u,\mu) \ge 0$, it follows from \eqref{eq:Delta_SJ.3} that 
\begin{equation}
\Delta \SJ \ge \frac{1}{2} (u-u_*)^\rmT Q (u-u_*) + \frac{\gamma}{2} (\mu-\mu_*)^2. \label{eq:Delta_SJ.4}
\end{equation}
Since $Q$ is positive definite, $\gamma > 0$, and $(u,\mu) \ne (u_*,\mu_*)$, it follow from \eqref{eq:Delta_SJ.4} that $\Delta \SJ > 0$, which yields \eqref{eq:Global_Min}.
\end{proof}

The following theorem is the main result on safety using the  control $u_*$, which is constructed from the real-time perception information $b_k$ and the soft-maximum R-CBF $\psi_{r-1}$.

\begin{theorem}\label{Th:Main th}
\rm
Consider \eqref{eq:affine control}, where \ref{con1}--\ref{con6} are satisfied. 
Let $u=u_*$, where $u_*$ is given by~\Cref{eq:softmax_h,eq:HOCBF,eq:uclose,eq:ulambda,eq:omegabar,eq:dxt}.
Assume that $\psi^\prime_{r-1}$ is locally Lipschitz in $x$ on $\BBR^n$.
Then, for all $x_0 \in \SC(0)$, the following statements hold: 
\begin{enumerate}
\item \label{main.thm.1}
There exists a maximum value $t_{\rm m}(x_0) \in (0,\infty]$ such that \eqref{eq:affine control} with $u=u_*$ has a unique solution on $[0, t_{\rm m}(x_0))$. 

\item \label{main.thm.2}
For all $t \in [0, t_{\rm m}(x_0))$, $x(t) \in \SC(t) \subseteq \SSS_\rms(t)$.

\end{enumerate}
\end{theorem}

\begin{proof}[\indent Proof]

To prove \ref{main.thm.1}, it follows from \ref{propC}  and \ref{propD} of \Cref{prop:ClosedForm} that $u_*$ is continuous in $t$ on $[0,\infty)$ and locally Lipschitz in $x$ on $\BBR^n$. 
Since, in addition, $f$ and $g$ are locally Lipschitz on $\BBR^n$, it follows from \cite[Theorem~3.1]{khalil2002control} that \eqref{eq:affine control} with $u=u_*$ has a unique solution on $[0, t_{\rm m}(x_0))$.

To prove \ref{main.thm.2}, it follows from \eqref{eq:affine control} with $u=u_*$ that 
\begin{equation}
\dot{\tilde x} = \tilde f(\tilde x), \label{eq:combo_sys.1}
\end{equation}
where
\begin{equation}\label{eq:combo_sys.2}
    \tilde x \triangleq \begin{bmatrix} t \\ x \end{bmatrix}, \, \,
    \tilde f(\tilde x) \triangleq \begin{bmatrix} 1 \\ f(x) + g(x) u_*(t,x) \end{bmatrix}, 
    \, \,
    \tilde x_0 \triangleq \begin{bmatrix} 0 \\ x_0 \end{bmatrix}.
\end{equation}
Next, define
\begin{gather}
\tilde \psi_{r-1}(\tilde x) \triangleq \psi_{r-1}(t,x), \label{eq:combo_sys.3}\\
\tilde \SC_{r-1} \triangleq \{ \tilde x \in [0,\infty) \times \BBR^n \colon \tilde \psi_{r-1}(\tilde x) \ge 0 \}, \label{eq:combo_sys.4}
\end{gather}
and it follows from \Cref{eq:combo_sys.2,eq:combo_sys.3} that 
\begin{align}\label{eq:tilde_Lfpsi}
L_{\tilde f} \tilde \psi_{r-1}(\tilde x) &= \textstyle\frac{\partial \psi_{r-1}(t,x)}{\partial t} + L_f \psi_{r-1}(t,x)\nn\\
&\qquad+ L_g\psi_{r-1}(t,x) u_*(t,x).
\end{align}
Let $\tilde x_1 \in \mbox{bd } \tilde \SC_{r-1}$,  and let $(t_1,x_1) \in [0,\infty) \times \BBR^n$ be such that $\tilde x_1^\rmT = [ \, t_1 \quad x_1 \, ]^\rmT$. 
Since $\tilde x_1 \in \mbox{bd } \tilde \SC_{r-1}$, it follows from 
\Cref{eq:combo_sys.3,eq:combo_sys.4} that $\psi_{r-1}(t_1,x_1) = 0$. 
Thus, substituting \Cref{eq:uclose} into \eqref{eq:tilde_Lfpsi}, and using \Cref{eq:ulambda,eq:omegabar,eq:dxt} yields $L_{\tilde f} \tilde \psi_{r-1}(\tilde x_1) = \omega(t_1,x_1) + \lambda(t_1,x_1) d(t_1,x_1) = \max \, \{0, \omega(t_1,x_1) \}$.
Thus, for all $\tilde x \in \mbox{bd } \tilde \SC_{r-1}$, $ L_{\tilde f} \tilde \psi_{r-1}(\tilde x) \ge 0$.

Next, since $x_0 \in \SC(0) \subseteq \SC_{r-1}(0)$, it follows from \Cref{eq:HOCBF set,eq:combo_sys.3,eq:combo_sys.4} that $\tilde x_0 \in \tilde \SC_{r-1}$. 
Since, in addition, for all $\tilde x \in \mbox{bd } \tilde \SC_{r-1}$, $ L_{\tilde f} \tilde \psi_{r-1}(\tilde x) \ge 0$, it follows from 
Nagumo's theorem \cite[Corollary~4.8]{blanchini2008set} that for all $t \in [0, t_{\rm m}(x_0))$, $\tilde x(t) \in \tilde \SC_{r-1}$.
Thus, \Cref{eq:HOCBF set,eq:combo_sys.3,eq:combo_sys.4} imply that for all $t \in [0, t_{\rm m}(x_0))$, $x(t) \in \SC_{r-1}(t)$.
Therefore, \Cref{prop:Forward_Invariant} implies that for all $t \in [0, t_{\rm m}(x_0))$, $x(t) \in \SC(t)$, which combined with \Cref{fact:S(t)} proves \ref{main.thm.2}. 
\end{proof}

Together, \Cref{thm:global minimizer,Th:Main th} show that $(u_*,\mu_*)$ is the globally optimal solution to an constrained optimization problem and that the R-CBF constraint used in that optimization guarantees safety, specifically, forward invariance of $\SC(t) \subseteq \SSS_\rms(t)$.

The next two sections demonstrate the optimal and safe control~\Cref{eq:softmax_h,eq:HOCBF,eq:uclose,eq:ulambda,eq:omegabar,eq:dxt} on simulations of a ground robot and an quadrotor robot operating in an \textit{a priori} unknown environment. These sections present an approach for using raw perception data (e.g., LiDAR) and the soft minimum to synthesize the perception feedback function $b_k$.

\section{Application to a Ground Robot}
\label{sec:GroundRobot}

Consider the nonholonomic ground robot modeled by \eqref{eq:affine control}, where
\begin{equation*}
    f(x) = \begin{bmatrix}
     v\cos\theta \\
     v\sin\theta \\
    0 \\
    0
    \end{bmatrix}, 
    \,
    g(x) = \begin{bmatrix}
    0 & 0\\
    0 & 0\\
    1 & 0 \\
    0 & 1
    \end{bmatrix}, 
    \,
    x = \begin{bmatrix}
    q_\rmx\\
    q_\rmy\\
    v\\
    \theta
    \end{bmatrix}, 
    \,
    u = \begin{bmatrix}
    u_1\\
    u_2
    \end{bmatrix}, 
\end{equation*}
and $q \triangleq [ \, q_\rmx \quad q_\rmy \, ]^\rmT$ is the robot's position in an orthogonal coordinate system, $v$ is the speed, and $\theta$ is the direction of the velocity vector (i.e., the angle from $[ \, 1 \quad 0 \, ]^\rmT$ to $[ \, \dot q_\rmx \quad \dot q_\rmy \, ]^\rmT$).

The robot is equipped with a perception/sensing system (e.g., LiDAR) that detects up to $\bar \ell$ points on objects that are: (i) in line of sight of the robot; (ii) inside the field of view (FOV) of the perception system; and (iii) inside the detection radius $\bar r > 0$ of the perception system. 
Specifically, for all $k \in \BBN$, at time $t=kT$, the robot obtains raw perception feedback in the form of $\ell_k \in \{0,1,\ldots,\bar \ell \}$ points given by $(r_{1,k},\theta_{1,k}),\cdots,(r_{\ell_k,k},\theta_{\ell_k,k})$, which are the polar-coordinate positions of the detected points relative to the robot position $q(kT)$ at the time of detection.
For all $i\in \{1,2,\ldots,\ell_k \}$, $r_{i,k} \in [0,\bar r]$ and $\theta_{i,k} \in [0,2\pi)$.

Since $(r_{i,k},\theta_{i,k})$ is the position of the detected point relative to $q(kT)$, it follows that for all $i\in \{1,2,\ldots, \ell_k \}$, the location of the detected point is 
\begin{equation*}
   c_{i,k} \triangleq q(kT) + r_{i,k} \begin{bmatrix}
        \cos \theta_{i,k} \\ \sin \theta_{i,k}
    \end{bmatrix}.
\end{equation*}
Similarly, for all $i\in \{1,2,\ldots, \ell_k \}$,
\begin{equation*}
   d_{i,k} \triangleq q(kT) + \bar r \begin{bmatrix}
        \cos \theta_{i,k} \\ \sin \theta_{i,k}
    \end{bmatrix}
\end{equation*}
is the location of the point that is at the boundary of the detection radius and on the line between $q(kT)$ and $c_{i,k}$.

Next, for each point $(r_{i,k},\theta_{i,k})$, we consider a function whose zero-level set is an ellipse that encircles $c_{i,k}$ and $d_{i,k}$.
Specifically, for all $i\in \{1,2,\ldots,\ell_k \}$, consider $\sigma_{i,k} \colon \BBR^4 \to \BBR$ defined by 
\begin{align}
    \sigma_{i,k}(x) &\triangleq \left [ \chi(x) - \frac{1}{2} (c_{i,k}+d_{i,k}) \right ]^\rmT R^\rmT_{i,k} P_{i,k} R_{i,k} \nn\\
    &\qquad \times \left [ \chi(x) - \frac{1}{2} (c_{i,k}+d_{i,k}) \right ] - 1, \label{eq:ellipse}
\end{align}
where 
\begin{gather}
    R_{i,k} \triangleq \begin{bmatrix}
        \cos\theta_{i,k} & \sin\theta_{i,k} \\ -\sin\theta_{i,k}& \cos\theta_{i,k}
    \end{bmatrix}, \quad
    P_{i,k} \triangleq \begin{bmatrix}
        a_{i,k}^{-2}&0 \\ 0 & z_{i,k}^{-2}
    \end{bmatrix}, \label{eq:ellipse.2}\\
    a_{i,k} \triangleq \frac{\bar r-r_{i,k}}{2} + \varepsilon_a, \quad 
    z_{i,k} \triangleq \sqrt{a_{i,k}^2 - \left(\frac{\bar r-r_{i,k}}{2}\right)^2}, \label{eq:ellipse.3}
\end{gather}
where $\chi(x) \triangleq [ \, I_2 \quad 0_{2\times2} \, ]^\rmT x$ extracts robot position from the state, and $\varepsilon_a > 0$ determines the size of the ellipse $\sigma_{i,k}(x)=0$ (specifically, larger $\varepsilon_a$ yields a larger ellipse). 
The parameter $\varepsilon_a$ is used to introduce conservativeness that can, for example, address an environment with dynamic obstacles. 
Note that $a_{i,k}$ and $z_{i,k}$ are the lengths of the semi-major and semi-minor axes, respectively. 
The area outside the ellipse is the zero-superlevel set of $\sigma_{i,k}$.

Next, let $\xi_k \colon \BBR^4 \to \BBR$ be a continuously differentiable function whose zero-superlevel set models the perception system's detection area (i.e., the detection radius and FOV).
For example, if the perception system has a $360^\circ$ FOV with detection radius $\bar r >0$, then it is appropriate to let $\xi_k = \beta_k$, where $\beta_k \colon \BBR^4 \to \BBR$ is defined by
\begin{align}\label{eq:circle}
    \beta_k(x) \triangleq  \left(\bar r - \varepsilon_\beta \right)^2- \| \chi(x)-q(kT) \|^2,
\end{align}
which has a zero-superlevel that is a disk with radius $\bar r - \varepsilon_\beta$ centered at the robot position $q(kT)$ at the time of detection, where $\varepsilon_\beta \ge 0$ influences the size of the disk and plays a role similar to $\varepsilon_a$.

We construct the perception feedback function $b_k$ using the soft minimum to compose $\sigma_{1,k},\cdots,\sigma_{\ell_k,k}$ and $\xi_k$.
Specifically, let $\rho>0$ and consider 
\begin{equation} \label{eq:bk_ex1}
    b_k(x) = \begin{cases}
            \mbox{softmin}_{\rho} \left( \xi_k(x), \sigma_{1,k}(x),\ldots,\sigma_{\ell_k,k}(x)\right), & \ell_k > 0,\\
            \xi_k(x), & \ell_k =0.
        \end{cases}
\end{equation}
\Cref{prop:softmin_softmax_sets} implies that $\SSS_k$ (i.e., zero-superlevel set of $b_k$) is a subset of the intersection of the zero-superlevel sets of $\xi_k$ and $\sigma_{1,k},\ldots,\sigma_{\ell_k,k}$. 
\Cref{prop:softmin_softmax_sets} also shows that for sufficiently large $\rho >0$, $\SSS_k$ approximates this intersection, which models the area inside the perception detection area and outside the ellipses that contain the detected points. 
\textcolor{black}{
In fact, \Cref{prop:softmin_softmax_sets} shows that $\SSS_k$ converges to this intersection as $\rho \to \infty$. 
However, if $\rho >0$ is large, then $\textstyle\frac{\partial b_k(x)}{\partial x}$ has large magnitude at points where $\min \left \{ \xi_k(x), \sigma_{1,k}(x),\ldots,\sigma_{\ell_k,k}(x)\right \}$ is not differentiable. 
Thus, choice of $\rho$ is a trade off between the magnitude of $\textstyle\frac{\partial b_k(x)}{\partial x}$ and the how well $\SSS_k$ approximates the intersection of the zero-superlevel sets of $\xi_k$ and $\sigma_{1,k},\ldots,\sigma_{\ell_k,k}$. }
\Cref{fig:ex1_safeset} illustrates the approach to generate the perception feedback function $b_k$ and thus, $\SSS_k$. 
Specifically, \Cref{fig:ex1_safeset} shows the ellipses modeled by $\sigma_{i,k}$, circle modeled by $\xi_k=\beta_k$, and set $\SSS_k$ that approximates the intersection of the zero-superlevel sets of $\xi_k$ and $\sigma_{1,k},\ldots,\sigma_{\ell_k,k}$.
\textcolor{black}{
We note that \Cref{eq:ellipse,eq:ellipse.2,eq:ellipse.3,eq:circle,eq:bk_ex1} can be used to verify directly that $L_gb_k=0$.
Next, \Cref{eq:bk_ex1} implies that $b_k$ is constructed using the soft minimum.
Thus, arguments similar to those in the proof of parts \ref{prop.h.3} and \ref{prop.h.4} of \Cref{prop.h} can be used to show that $L_gL_f b_k$ is a convex combination of $L_gL_f\sigma_{1,k},\ldots,L_gL_f\sigma_{\ell_k,k}$ and $L_gL_f \xi_k$. 
Since, in addition, $L_gL_f\sigma_{i,k}(x)\ne0$ and $L_gL_f \xi_k(x)\ne0$ for almost all $x$, it follows that $L_gL_f b_k(x) \ne 0$ for almost all $x$.
Hence, $b_k$ satisfies \ref{con3} and~\ref{con4} with $r=2$. 
}

\begin{figure}[t!]
\center{\includegraphics[width=0.44\textwidth,clip=true,trim= 0.4in 0.3in 1.0in 0.85in] {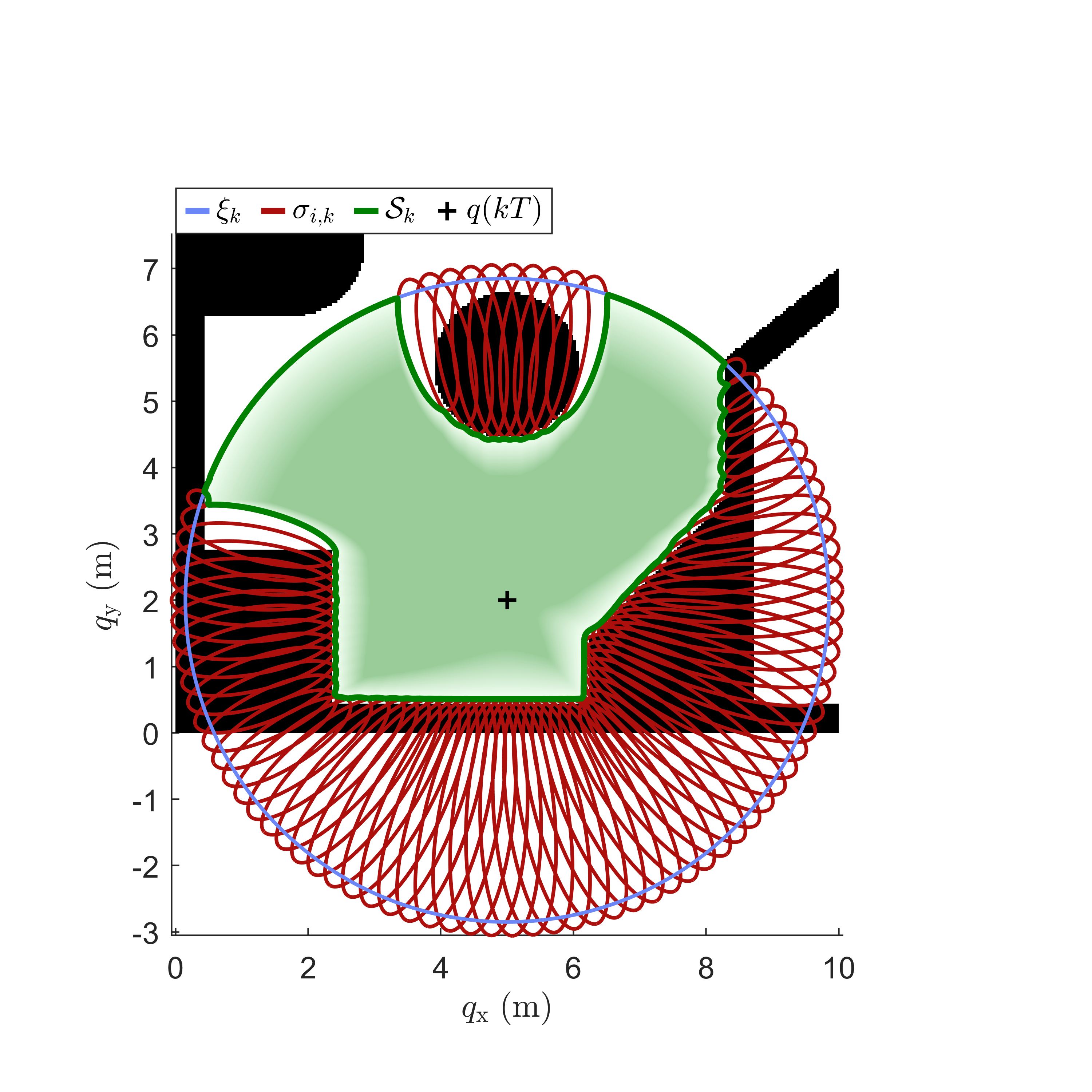}}
\caption{$\SSS_k$ that approximates the intersection of the zero-superlevel sets of $\xi_k$ and $\sigma_{1,k},\ldots,\sigma_{\ell_k,k}$.}\label{fig:ex1_safeset}
\end{figure}

The control objective is for the robot to move from its initial location to a goal location $q_\rmg = [ \, q_{\rmg,\rmx} \quad q_{\rmg,\rmy} \, ]^\rmT \in\BBR^2$ without violating safety (i.e., hitting an obstacle).
To accomplish this objective, we consider the desired control
\begin{equation*}
    u_\rmd(x)  \triangleq \begin{bmatrix} u_{\rmd_1}(x)\\u_{\rmd_2}(x) \end{bmatrix}, 
\end{equation*}
 where
\begin{align*}
u_{\rmd_1}(x) &\triangleq  -(k_1+k_3) v + (1+k_1k_3)\| q - q_\rmg \| \cos \delta(x)\\
&\qquad + k_1\left ( k_2 \| q - q_\rmg \| +v \right )\sin^2\delta(x),\\
u_{\rmd_2}(x) &\triangleq \left ( k_2+\frac{v}{\| q - q_\rmg \|} \right )\sin\delta(x),\\
\delta(x) &\triangleq\mbox{atan2}(q_\rmy-q_{\rmg,\rmy},q_\rmx-q_{\rmg,\rmx})-\theta + \pi,
\end{align*}
and $k_1,k_2,k_3 > 0$.
The desired control $u_\rmd$ drives $q$ to $q_\rmg$ but does not account for safety, that is, it does not incorporate information obtained from the perception system. 
The desired control is designed using a process similar to \cite[pp.~30--31]{de2002control}.

The simulations in this section use the minimum-intervention approach with the cost \eqref{eq:J}, where $Q(t,x) = I_2$, $c(t,x) = -u_\rmd(x)$.
Thus, the unconstrained minimizer of the cost is the desired control $u_\rmd$. 
For all examples in this section, the perception update period is $T = 0.2$~s and the gains for the desired control are $k_1 = 0.5$, $k_2 = 3$, and $k_3 = 3$.

We present simulations for 3 examples: $360^\circ$ FOV perception in a static environment; $360^\circ$ FOV perception in a dynamic environment; and $120^\circ$ FOV perception in a static environment.

\subsection{$360^\circ$ FOV in a Static Environment}
\label{sec:360_FOV}

We use the perception feedback $b_k$ given by \eqref{eq:bk_ex1}, where $\xi_k = \beta_k$, $\rho = 30$, $\bar r = 5$~m, and $\varepsilon_a = \varepsilon_\beta = 0.15$~m. 
The maximum number of detected points is $\bar \ell = 100$, and $N = 4$; however, similar results are obtained for other choices of $N$. 
\Cref{fig:ex1_map} shows a map of the unknown environment that the ground robot aims to navigate. 
\Cref{fig:ex1_safeset} is a close-up of the map at $k=0$ near the robot initial position (i.e., $q(0)$).

\begin{figure}[t!]
\center{\includegraphics[width=0.44\textwidth,clip=true,trim= 0.3in 0.3in 1in 1in] {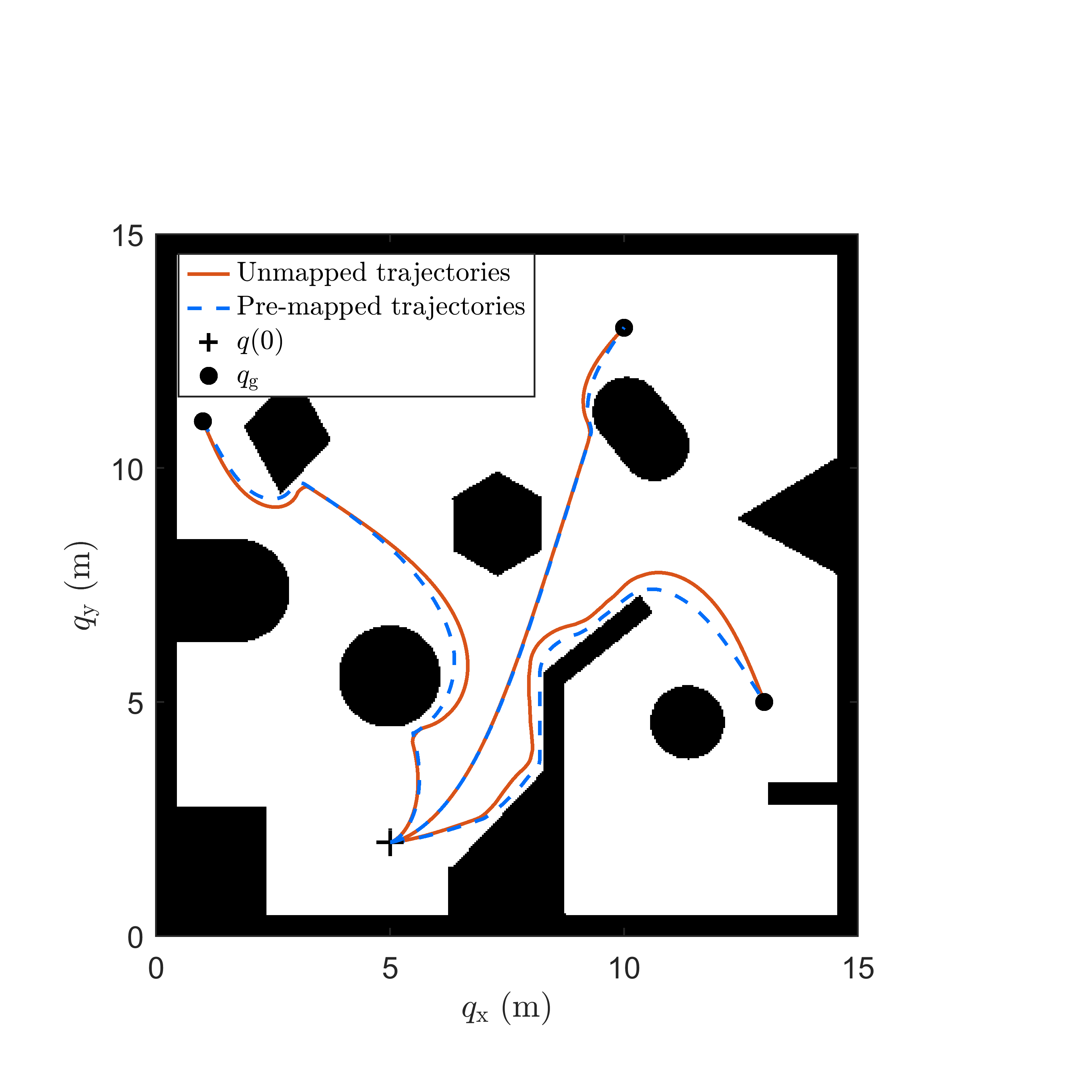}}
\caption{Three closed-loop trajectories using the control~\Cref{eq:softmax_h,eq:HOCBF,eq:uclose,eq:ulambda,eq:omegabar,eq:dxt} with the perception feedback $b_k$ generated from $360^{\circ}$ FOV perception in a static environment.} \label{fig:ex1_map}
\end{figure}

We implement the control~\Cref{eq:softmax_h,eq:HOCBF,eq:uclose,eq:ulambda,eq:omegabar,eq:dxt} with $\kappa = 30$, $\gamma = 200$, $\alpha_0(\psi_0) =35\psi_0$, $\alpha(\psi_1) =50\psi_1$, and $\eta$ given by Example~\ref{ex:g} where $r=2$ and $\nu = 2$. 
For sample-data implementation, the control is updated at $100$~Hz.

\Cref{fig:ex1_map} shows the closed-loop trajectories for $x_0 = [\,5\quad 2\quad0\quad0\,]^\rmT$ with 3 goal locations: $q_\rmg = [\,13\quad5\,]^\rmT$~m, $q_\rmg = [\,10\quad13\,]^\rmT$~m, and $q_\rmg = [\,1\quad11\,]^\rmT$~m. 
In all cases, the robot converges to the goal while satisfying safety. 
\Cref{fig:ex1_h,fig:ex1_state} show time histories of the relevant signals for the case where $q_\rmg = [\,13\quad5\,]^\rmT$~m. 
\textcolor{black}{
See the supplementary material for a video showing the time evolution of the robot and perception-based safe set $\SC_0(t)$.}
\Cref{fig:ex1_h} shows $\psi_0$ and $\psi_1$ are positive for all time, which demonstrates that for all time $t$, the trajectory is in $\SC(t) \subseteq \SSS_\rms(t)$. 
Thus, safety is satisfied. 
\Cref{fig:ex1_state} shows $u$ deviates from $u_\rmd$ in order to satisfy safety. 
Specifically, the turning-angle rate $u_2$ deviates from $u_{\rmd_2}$. 
\textcolor{black}{We note that the desired control $u_\rmd$ does not incorporate information from the perception system (i.e., $u_\rmd$ does not use knowledge of $\SSS_k$). 
Thus, $u$ may deviate significantly from $u_\rmd$, which is designed to drive the robot to the goal in an obstacle-free environment. 
Hence, performance could potentially be improved by using a more sophisticated desired control $u_\rmd$ that explicitly accounts for the information from the perception system. 
For example, $u_\rmd$ could be incorporate higher-level planners (e.g., \cite{lindqvist2021exploration}).
We also note that the user-selected functions $\alpha_0(\psi_0) =35\psi_0$ and $\alpha(\psi_1) =50\psi_1$ are selected to yield relatively aggressive behavior, where $u$ does not deviate from $u_\rmd$ until the robot is near the boundary of $\SC$ but then it may deviate significantly. 
Less aggressive deviations are accomplished by decreasing the slope of $\alpha_0$ and $\alpha$.}


\textcolor{black}{
For comparison, we present simulation results where the environment is perfectly known. 
For this case, we use the support vector machine classifier approach in \cite{srinivasan2020synthesis} to synthesize a single barrier function $h_*$ from the map in \Cref{fig:ex1_map}, which is assumed to be known \textit{a priori}. 
Next, we use the approach in \Cref{sec:control} with $\psi_0$ replaced by $h_*$ to construct a closed-form optimal control that guarantees safety for the pre-mapped environment described by $h_*$. 
\Cref{fig:ex1_map} shows closed-loop trajectories using the control designed for the pre-mapped environment; all control parameters are the same as used above with the unmapped implementation. 
\Cref{fig:ex1_map} shows that the trajectories with the unmapped control (i.e., using $\psi_0$ that is constructed from real-time perception feedback) are comparable to those with the pre-mapped control (i.e., using perfect map $h_*$ in place of $\psi_0$). 
\Cref{tab:trajectory_comparison} compares the settling time $T_\rms \triangleq \min \{ \hat{t} \, \colon \mbox{for all } t \geq \hat{t}, \| x(t) - q_\rmg \| \leq 0.1 \}$ and the root mean square (RMS) of $u-u_\rmd$ for the two cases. 
These metrics for the unmapped case are similar to the pre-mapped case, demonstrating that this new method for unmapped environments performs comparably to existing methods for perfectly known pre-mapped environments. 
}

\begin{table}[ht]
\centering
\small
\sisetup{detect-weight=true, detect-inline-weight=math}
\setlength{\tabcolsep}{3pt} 
\renewcommand{\arraystretch}{1.1} 
\caption{Performance metrics for unmapped and pre-mapped trajectories.}
\label{tab:trajectory_comparison}
\begin{tabular}{|c|c|c|c|c|}
\hline
Goal & Method & $T_\rms$ & $\operatorname*{RMS} u_1 - u_{\rmd_1}$ & $\operatorname*{RMS} u_2 - u_{\rmd_2}$ \\
\hline
\multirow{2}{*}{$[13 \; 5]^\rmT$} & Unmapped   & $5.2$ & $0.15$ & $2.88$ \\
                               & Pre-mapped & $5.0$ & $0.14$ & $2.67$ \\
\hline
\multirow{2}{*}{$[10 \; 13]^\rmT$} & Unmapped   & $4.7$ & $0.15$ & $0.77$ \\
                                & Pre-mapped & $4.6$ & $0.15$ & $0.75$ \\
\hline
\multirow{2}{*}{$[1 \; 11]^\rmT$} & Unmapped   & $5.2$ & $1.89$ & $3.37$ \\
                               & Pre-mapped & $4.8$ & $1.78$ & $3.11$ \\
\hline
\end{tabular}
\end{table}

\begin{figure}[t!]
\center{\includegraphics[width=0.44\textwidth,clip=true,trim= 0.4in 0.4in 1in 0.6in] {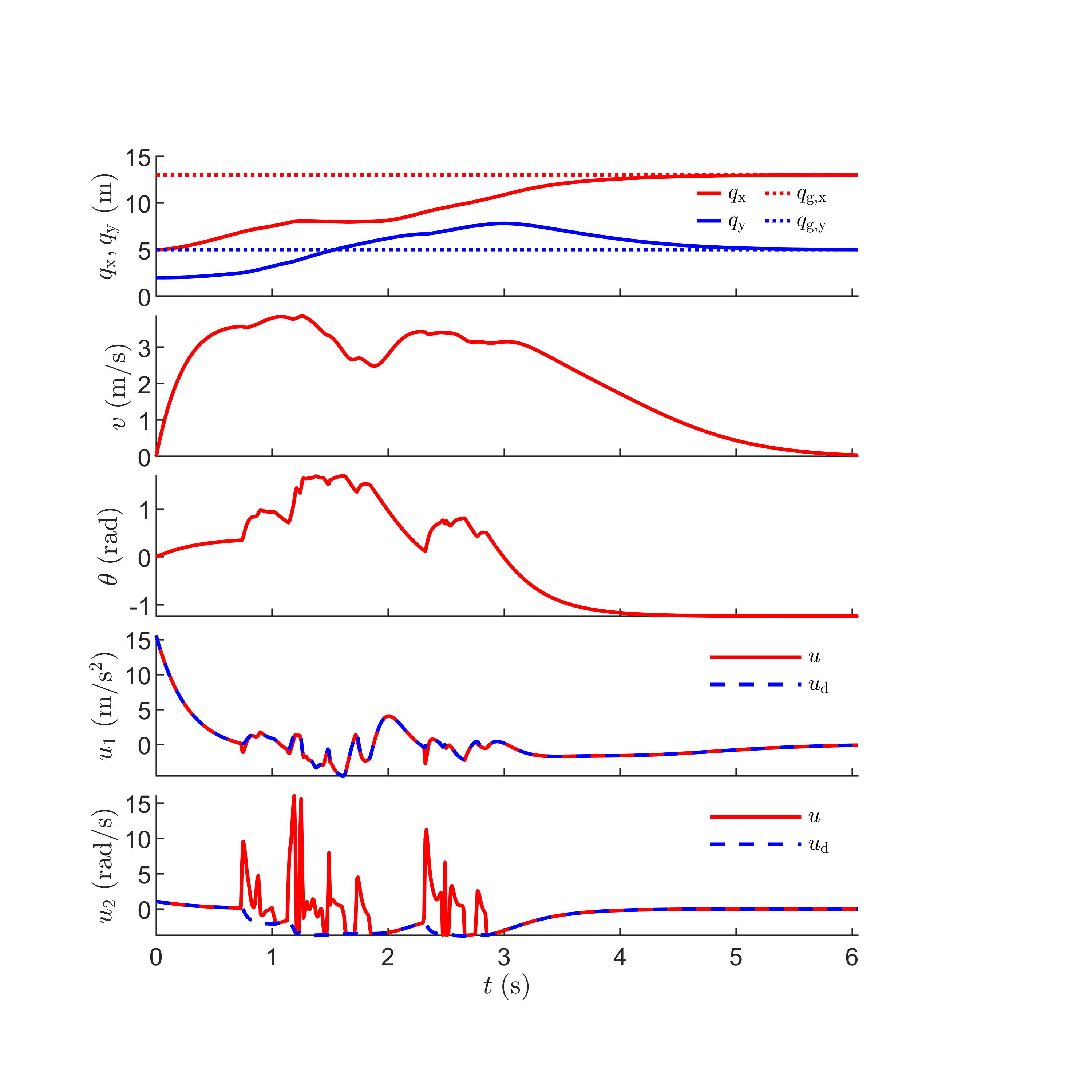}}
\caption{$q_\rmx$, $q_\rmy$, $v$, $\theta$, $u_\rmd$, and $u$ for $q_\rmg = [\,13\quad5\,]^\rmT$~m.}
\label{fig:ex1_state}
\end{figure} 

\begin{figure}[t!]
\center{\includegraphics[width=0.44\textwidth,clip=true,trim= 0.3in 0.27in 0.9in 0.5in] {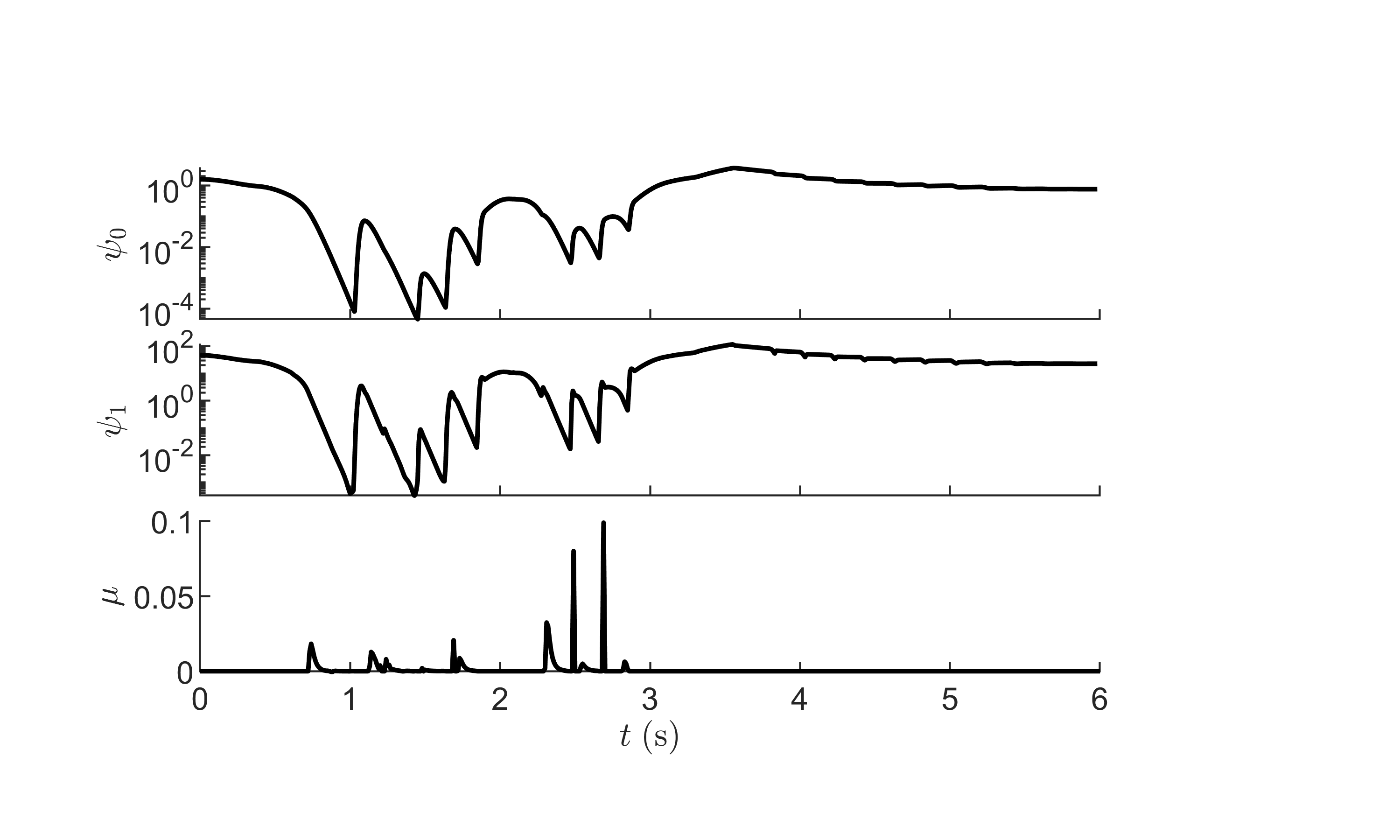}}
\caption{$\psi_0$, $\psi_1$, and $\mu$ for $q_\rmg = [\,13\quad5\,]^\rmT$~m.}\label{fig:ex1_h}
\end{figure}

\subsection{$360^\circ$ FOV in a Dynamic Environment}
\label{sec:Dynamic_env}

To address a dynamic environment, we assume that all objects in the environment move with a maximum speed $\bar v >0$. 
To satisfy \ref{con1}, we generate $\SSS_k$ that does not contain any of the possible positions that a point detected at time $t=kT$ can occupy during the time interval $[kT, (k+N+1)T]$. 
Specifically, for all $i\in \{1,2,\ldots, \ell_k \}$, $\SSS_k$ cannot contain the disk of radius $T(N+1) \bar v$ centered at 
the location $c_{i,k}$ of the detected point. 
Thus, for all $i\in \{1,2,\ldots, \ell_k \}$, we want the ellipse given by the zero-level set of $\sigma_{i,k}$ to encircle the disk of radius $T(N+1) \bar v$ centered at $c_{i,k}$.
This is accomplished by setting $\varepsilon_a \ge T(N+1) \bar v$. 
Similarly, the model of the detection area $\xi_k$ needs to account for the fact that objects outside the detection area can move into the detection area.
Similarly, this is accomplished by using $\xi_k = \beta_k$, where $\varepsilon_\beta \ge T(N+1) \bar v$. 
Thus, this section uses the perception feedback function $b_k$ given by \eqref{eq:bk_ex1}, where $\xi_k = \beta_k$, $\rho = 20$, $\bar r = 5$~m, and $\varepsilon_a = \varepsilon_\beta = T(N+1) \bar v$.
In this example, $\bar v = 0.5$~m/s and $\bar \ell = 100$. 
Since the environment is dynamic, we use $N=1$.
\Cref{fig:ex1_dynamic_map} is a map of the unknown and dynamic environment where the shaded circles are dynamic obstacles.
The dashed lines show the time-evolution of each dynamic obstacle, which are not known by the robot control system.

\begin{figure}[t!]
\center{\includegraphics[width=0.41\textwidth,clip=true,trim= 0.3in 0.3in 1in 1.03in] {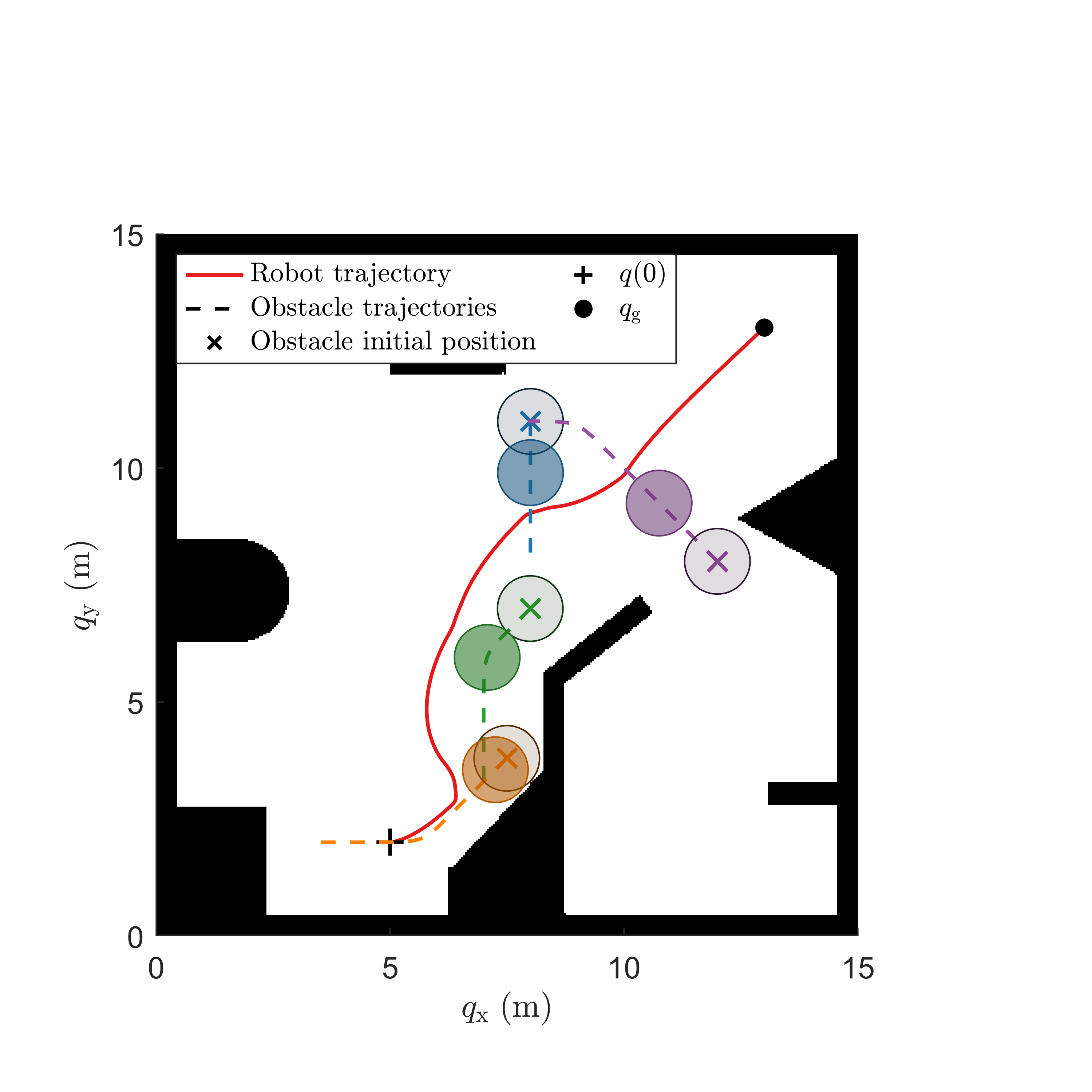}}
\caption{Map of the environment showing the motion of the dynamic obstacles, and a closed-loop trajectory using the control~\Cref{eq:softmax_h,eq:HOCBF,eq:uclose,eq:ulambda,eq:omegabar,eq:dxt} with $b_k$ generated from $360^{\circ}$~FOV perception.
Circles show the location of each dynamic obstacle at $t=0$ and when it is closest to the robot. }\label{fig:ex1_dynamic_map}
\end{figure}

\begin{figure}[t!]
\center{\includegraphics[width=0.44\textwidth,clip=true,trim= 0.35in 0.37in 0.6in 0.67in] {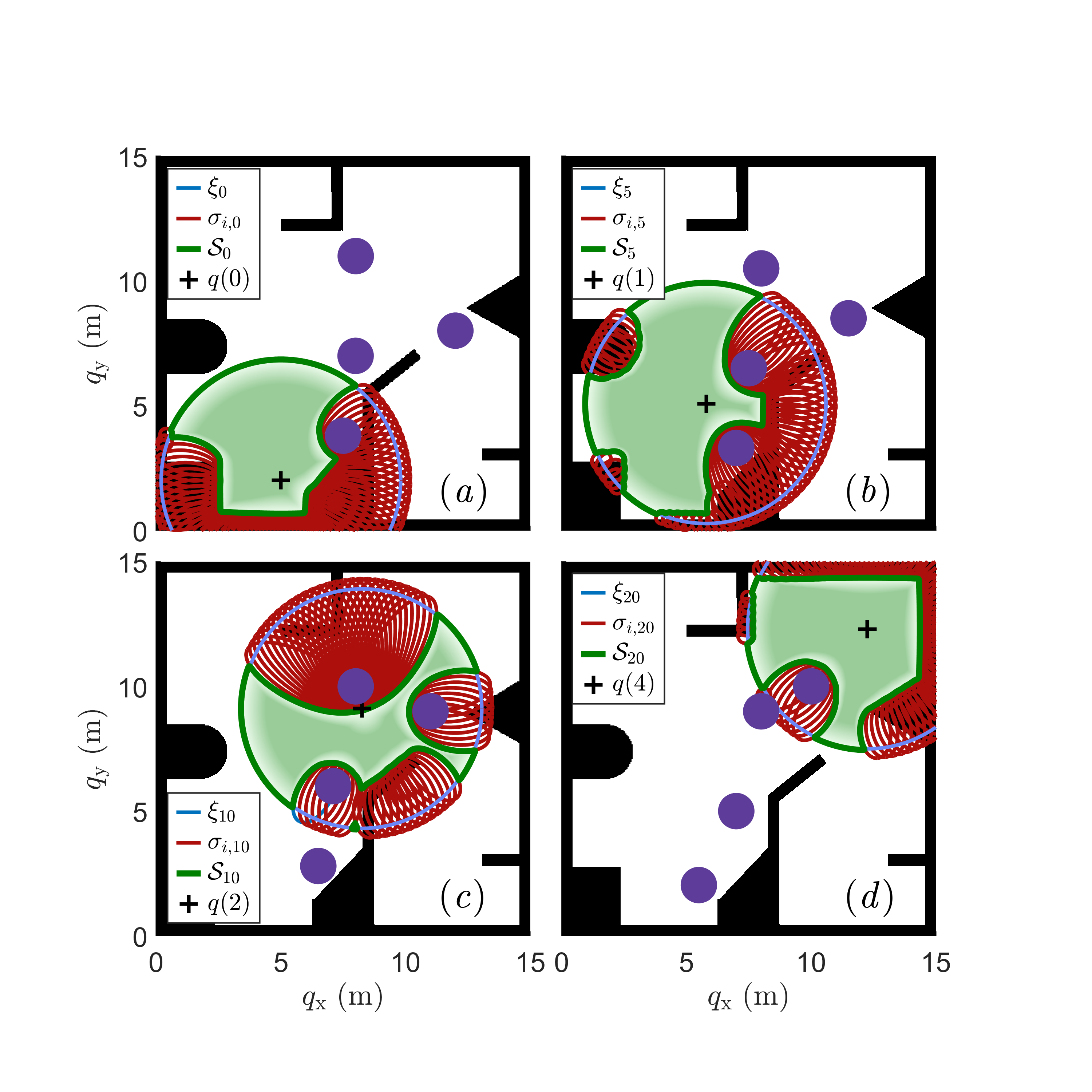}}
\caption{Ellipses modeled by $\sigma_{i,k}$ and circle modeled by $\xi_k=\beta_k$ at (a) $t=0$~s, (b) $t=1$~s, (c) $t=2$~s, and (d) $t=4$~s.}
\label{fig:ex1_dynamic_safeset}
\end{figure} 

We implement the control~\Cref{eq:softmax_h,eq:HOCBF,eq:uclose,eq:ulambda,eq:omegabar,eq:dxt} with $\kappa = 30$, $\gamma = 200$, $\alpha_0(\psi_0) =30\psi_0$, $\alpha(\psi_1) =30\psi_1$, and $\eta$ given by Example~\ref{ex:g} where $r=2$ and $\nu = 2$. 
For sample-data implementation, the control is updated at $100$~Hz.

\Cref{fig:ex1_dynamic_map} shows the closed-loop trajectory for $x_0 = [\,5\quad2\quad0\quad0\,]^\rmT$ with goal location $q_\rmg = [\,13\quad 13\,]^\rmT$~m. 
The figure shows each dynamic obstacle's initial position, trajectory, and location at the time when the obstacle was closest to the robot. 
\Cref{fig:ex1_dynamic_safeset} shows the ellipses modeled by $\sigma_{i,k}$ and circle modeled by $\xi_k=\beta_k$ at $t=0$~s, $t=1$~s, $t=2$~s, and $t=4$~s. 
\Cref{fig:ex1_dynamic_h,fig:ex1_dynamic_state} show time histories of the relevant signals. 
\Cref{fig:ex1_dynamic_h} shows that $\psi_0$ and $\psi_1$ are positive for all time, indicating that the safety constraint is satisfied. 
\textcolor{black}{
See the supplementary material for a video showing the time evolution of the robot and the obstacles.}

\begin{figure}[t!]
\center{\includegraphics[width=0.44\textwidth,clip=true,trim= 0.35in 0.4in 1.0in 0.6in] {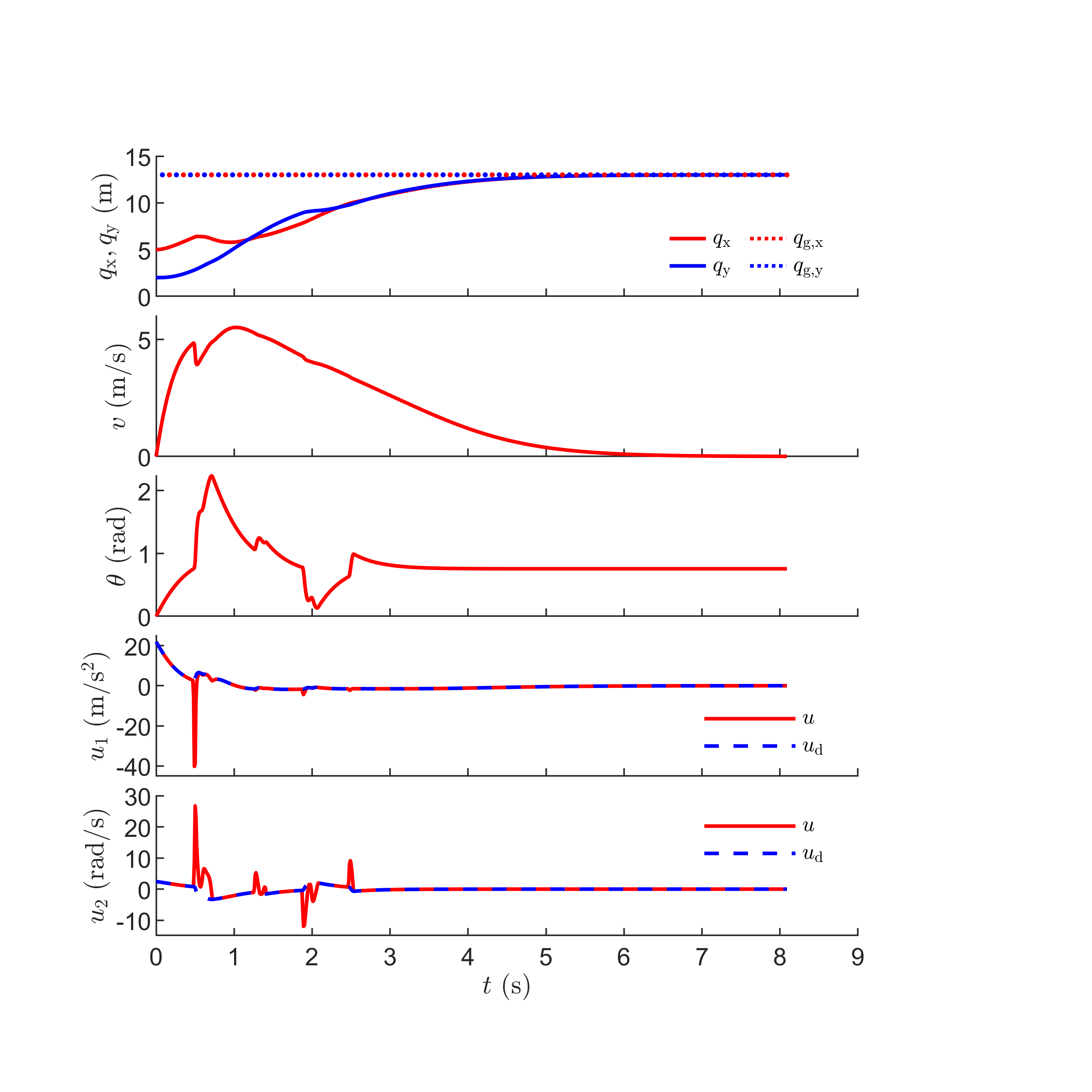}}
\caption{$q_\rmx$, $q_\rmy$, $v$, $\theta$, $u_\rmd$, and $u$ for $q_\rmg = [\,13\quad5\,]^\rmT$~m.}
\label{fig:ex1_dynamic_state}
\end{figure}

\begin{figure}[t!]
\center{\includegraphics[width=0.44\textwidth,clip=true,trim= 0.2in 0.27in 1.0in 0.5in] {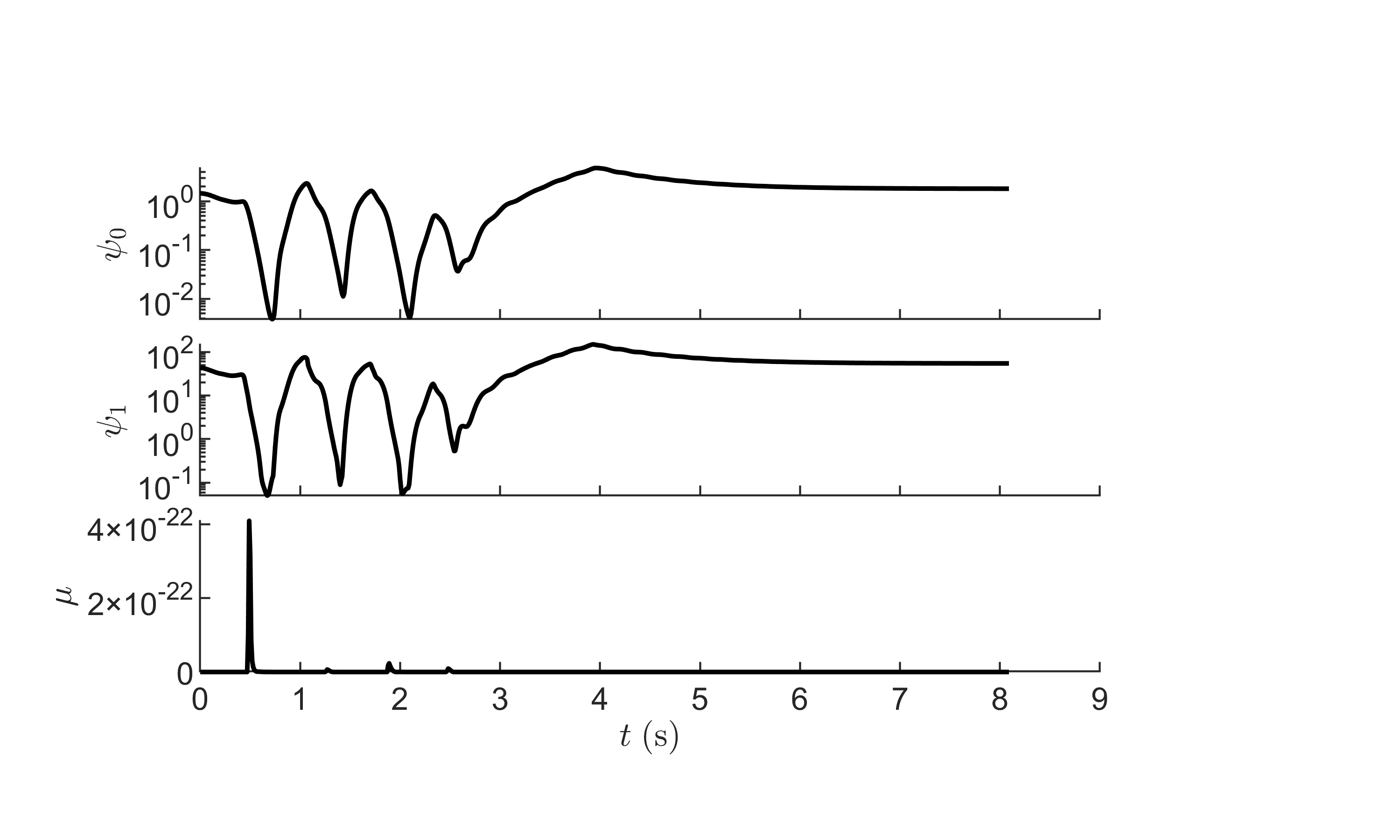}}
\caption{$\psi_0$, $\psi_1$, and $\mu$ for $q_\rmg = [\,13\quad5\,]^\rmT$~m.}
\label{fig:ex1_dynamic_h}
\end{figure} 

\textcolor{black}{
To examine control performance in more cluttered environments, we extend this simulation by introducing additional dynamic obstacles.
The baseline scenario in \Cref{fig:ex1_dynamic_map,fig:ex1_dynamic_safeset,fig:ex1_dynamic_h,fig:ex1_dynamic_state,fig:ex1_dynamic_h} has 4 dynamic obstacles. 
We examine scenarios with $5,6,\ldots,15$ dynamic obstacles where the map is the same as in \Cref{fig:ex1_dynamic_map}.
For each scenario, we conduct $1,000$ trials, where the initial condition $x(0)$, goal location $q_\rmg$, and trajectories of dynamic obstacles are random for each trial.
The speed of each dynamic obstacle is a uniform random variable on $(0, \bar v]$. 
All control parameters are the same as the baseline scenario, and each trial is 20-s long. 
For each scenario, \Cref{tab:dyn_obs} shows the percent of trials without collision (``Safe''), and percent of trials without collision and where the goal is reached by 20~s (``Successful'').
Note that collisions can occur because the density and random motion of dynamic obstacles can violate \ref{con2}; specifically, there may be time instants where the perceived safe set vanishes (i.e., $\SC_0$ may be empty because $\SSS_k \cap \SSS_{k-1}$ is empty). 
%
For each successful trial, we consider the following performance metrics: minimum value of the barrier function $\min_{t \in [0,20]} \psi_0(t)$; settling time $T_\rms$; and RMS of $u-u_\rmd$. 
\Cref{fig:boxplot_dynamic} shows the median, interquartile range ($25$th-to-$75$th percentiles), and whisker range ($10$th-to-$90$th percentiles) of each performance metric for each scenario. 
As expected, as the environment becomes more densely cluttered, the settling time increases, $\min_{t \in [0,20]} \psi_0(t)$ decreases, and the control perturbation $u-u_\rmd$ increases.
Note that all metrics are approximately linear in the number of obstacles. 
}
\begin{table}[t!]
\centering
\small
\sisetup{detect-weight=true, detect-inline-weight=math}
\setlength{\tabcolsep}{3pt} 
\renewcommand{\arraystretch}{1.1} 
\caption{Percent of trials for each scenario that are Safe and Successful.}
\label{tab:dyn_obs}
\begin{tabular}{|c|c|c|c|}
\hline
Number of Obstacles & Safe & Successful  \\
\hline
\multirow{1}{*}{$5$}   & 98.2 & 97.8   \\
\hline
\multirow{1}{*}{$6$}   & 96.4 & 95.6   \\
\hline
\multirow{1}{*}{$7$}   & 95.1 & 94.0   \\
\hline
\multirow{1}{*}{$8$}   & 91.3 & 90.0   \\
\hline
\multirow{1}{*}{$9$}   & 88.9 & 87.1   \\
\hline
\multirow{1}{*}{$10$}   & 85.8 & 83.1   \\
\hline
\multirow{1}{*}{$11$}   & 84.7 & 81.3   \\
\hline
\multirow{1}{*}{$12$}   & 80.2 & 76.2   \\
\hline
\multirow{1}{*}{$13$}   & 76.7 & 72.4   \\
\hline
\multirow{1}{*}{$14$}   & 72.7 & 68.5   \\
\hline
\multirow{1}{*}{$15$}   & 67.8 & 63.0   \\
\hline
\end{tabular}
\end{table}

\begin{figure}[t!]
\center{\includegraphics[width=0.46\textwidth,clip=true,trim= 0.3in 0.2in 0.7in 0.6in] {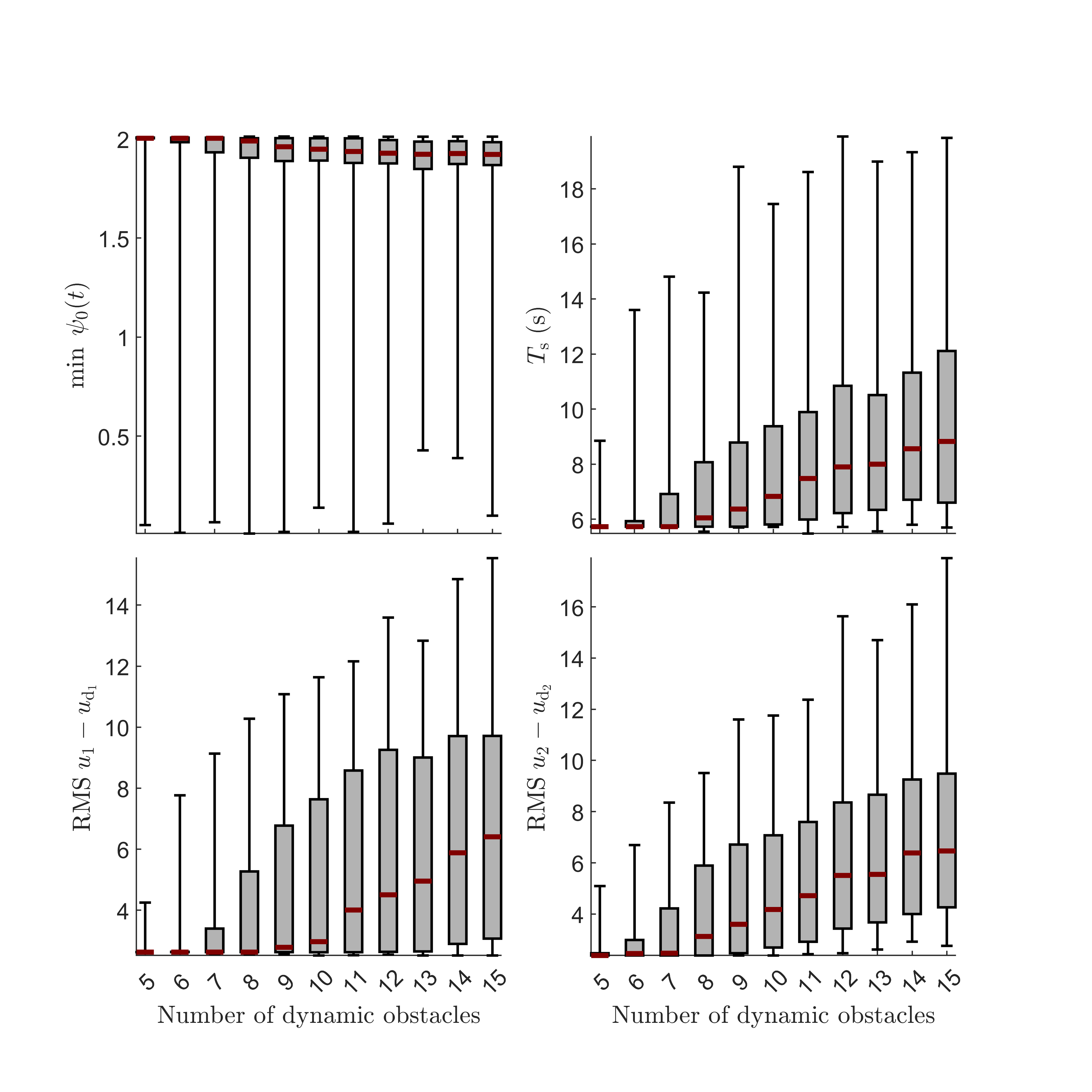}}
\caption{Median (red line), interquartile range (box), and $10$th-to-$90$th-percentile range (whisker) of each metric for each scenario.} 
\label{fig:boxplot_dynamic}
\end{figure}

\subsection{$120^\circ$ FOV in a Static Environment}
\label{sec:limited_FOV}

We now consider sensing capability with limited FOV. 
We use the perception feedback $b_k$ given by \eqref{eq:bk_ex1}, where 
\begin{align*}
    \xi_k(x) &= \mbox{softmin}_{\rho} \left( \beta_k(x), \ubar \tau_{k}(x), \bar \tau_{k}(x) \right), \\
    \ubar \tau_{k}(x) &\triangleq   e\left(\theta(kT) - \frac{\theta_f}{2}\right) \left[ \chi(x)-q(kT)\right]  - \varepsilon_k, \\
    \bar \tau_{k}(x) &\triangleq  - e\left(\theta(kT) + \frac{\theta_f}{2}\right) \left [\chi(x)-q(kT)\right]  - \varepsilon_k,  \\
    e(\theta) &\triangleq \begin{bmatrix} -\sin\theta &  \cos\theta \end{bmatrix},
\end{align*}
where $\rho > 0$, $\theta_\rmf \in (0,\pi]$~rad is the FOV, and $\epsilon_k > 0$ is chosen such that $b_{k}(q(kT)) = 0$. 
The detection area is modeled by the zero-superlevel set of $\xi_k$, which uses the soft minimum to approximate the intersection of the zero-superlevel sets of $\beta_k$, $\ubar \tau_k$, and $\bar \tau_k$. 
Note that $\beta_k$ models the disk based on detection radius, and $\ubar \tau_k$ and $\bar \tau_k$ model half planes based on the FOV angle $\theta_\rmf$. 
In this example, $\theta_f = 2\pi/3$~rad (i.e., $120^\circ$ FOV), $\rho =30$, $\bar \ell = 30$, $\bar r =5$~m, and $\varepsilon_a = \varepsilon_\beta = 0.15$~m. 
\Cref{fig:ex1_FOV_safeset} shows the ellipses modeled by $\sigma_{i,k}$, the detection areas modeled by $\xi_k$, and the associated $\SSS_k$ at $k=0$ near the robot initial position (i.e., $q(0)$).

\begin{figure}[t!]
\center{\includegraphics[width=0.41\textwidth,clip=true,trim= 0.35in 0.3in 1.0in 0.86in] {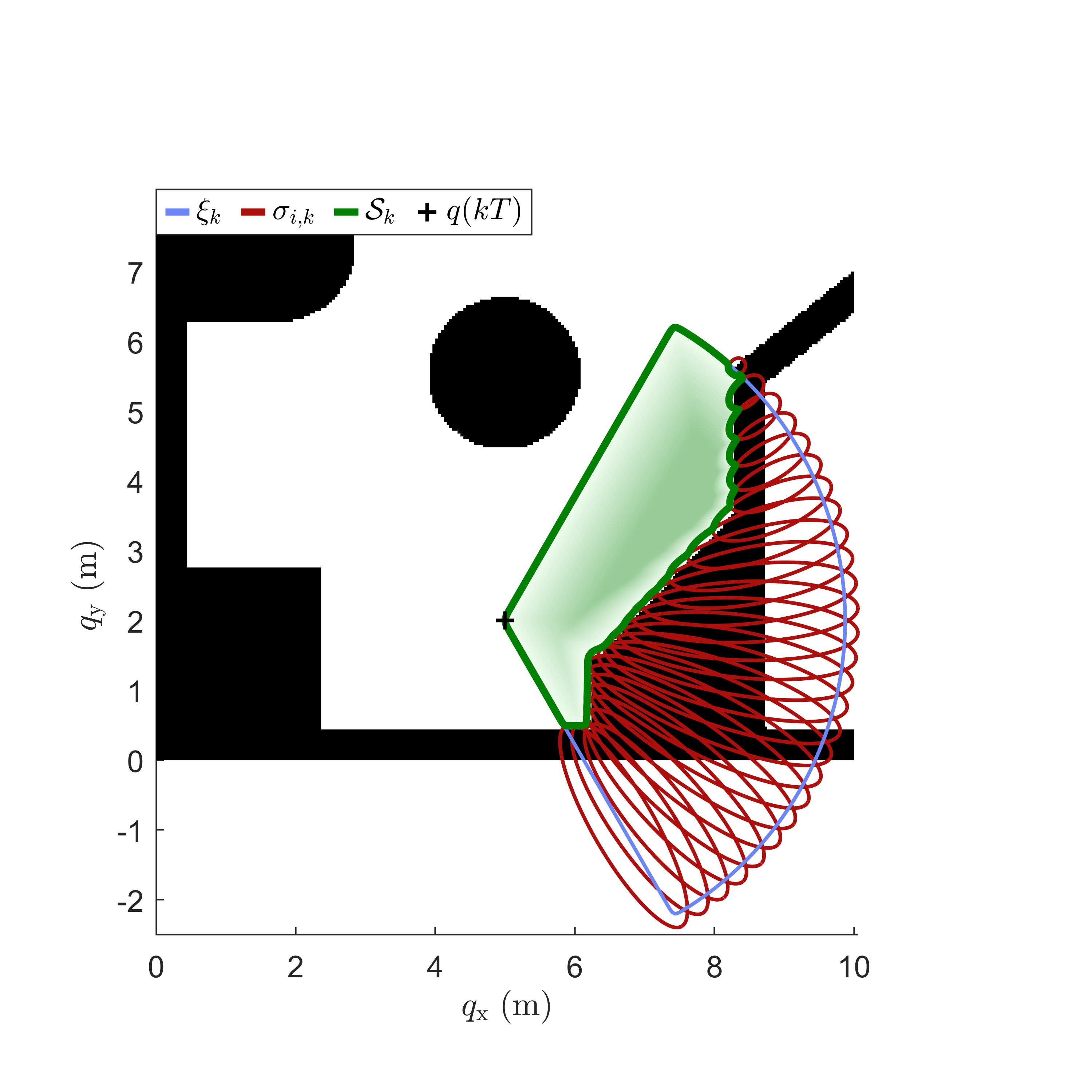}}
\caption{$\SSS_k$ that approximates the intersection of the zero-superlevel sets of $\xi_k$ and $\sigma_{1,k},\ldots,\sigma_{\ell_k,k}$.}\label{fig:ex1_FOV_safeset}
\end{figure} 

\begin{figure}[t!]
\center{\includegraphics[width=0.41\textwidth,clip=true,trim= 0.3in 0.3in 1in 1.03in] {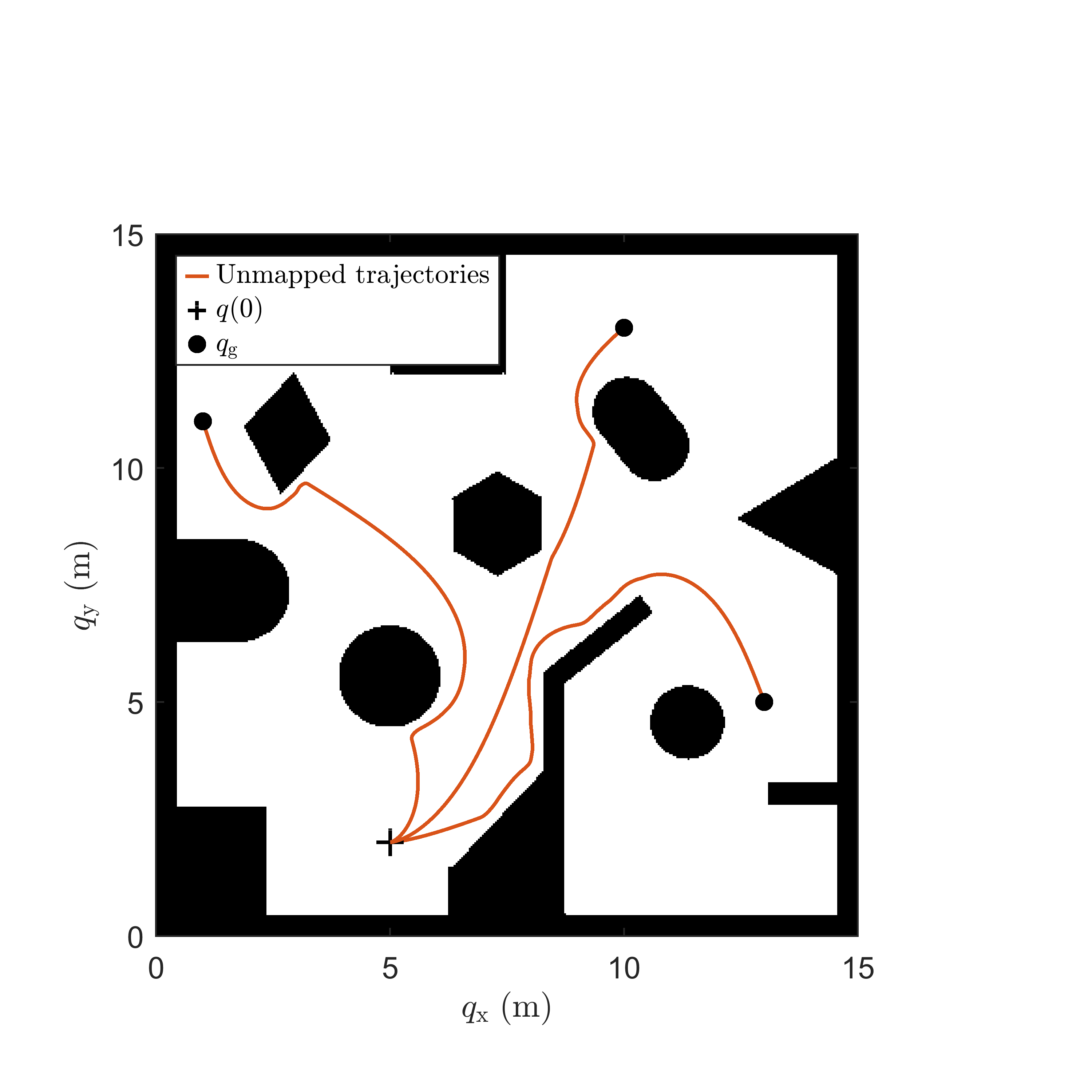}}
\caption{Three closed-loop trajectories using the control~\Cref{eq:softmax_h,eq:HOCBF,eq:uclose,eq:ulambda,eq:omegabar,eq:dxt} with the perception feedback $b_k$ generated from $120^{\circ}$ FOV perception in a static environment.}\label{fig:ex1_FOV_map}
\end{figure} 

We implement the control~\Cref{eq:softmax_h,eq:HOCBF,eq:uclose,eq:ulambda,eq:omegabar,eq:dxt} with $\kappa = 30$, $\gamma = 200$, $\alpha_0(\psi_0) =40\psi_0$, $\alpha(\psi_1) =65\psi_1$, and $\eta$ given by Example~\ref{ex:g} where $r=2$ and $\nu = 2$. 
Since this example considers a limited FOV, we use $N = 4$ to incorporate recent perception feedback functions and enlarge the percieved subset of the safe set $\SC_0(t)$.
For sample-data implementation, the control is updated at $100$~Hz.

\Cref{fig:ex1_FOV_map} shows the closed-loop trajectories  for $x_0 = [\,5\quad2\quad0\quad0\,]^\rmT$ with 3 different goals locations: $q_\rmg = [\,11\quad2\,]^\rmT$~m, $q_\rmg = [\,10\quad13\,]^\rmT$~m, and $q_\rmg = [\,1\quad11\,]^\rmT$~m.
In all cases, the robot converges to the goal while satisfying safety. 
\Cref{fig:ex1_FOV_h,fig:ex1_FOV_state} provide time histories of relevant signals for the case where $q_\rmg = [\,13\quad5\,]^\rmT$~m. 
\Cref{fig:ex1_FOV_h} shows that $\psi_0$ and $\psi_1$ are positive for all time, indicating that the safety constraint is satisfied.

\begin{figure}[t!]
\center{\includegraphics[width=0.44\textwidth,clip=true,trim= 0.35in 0.4in 1.0in 0.6in] {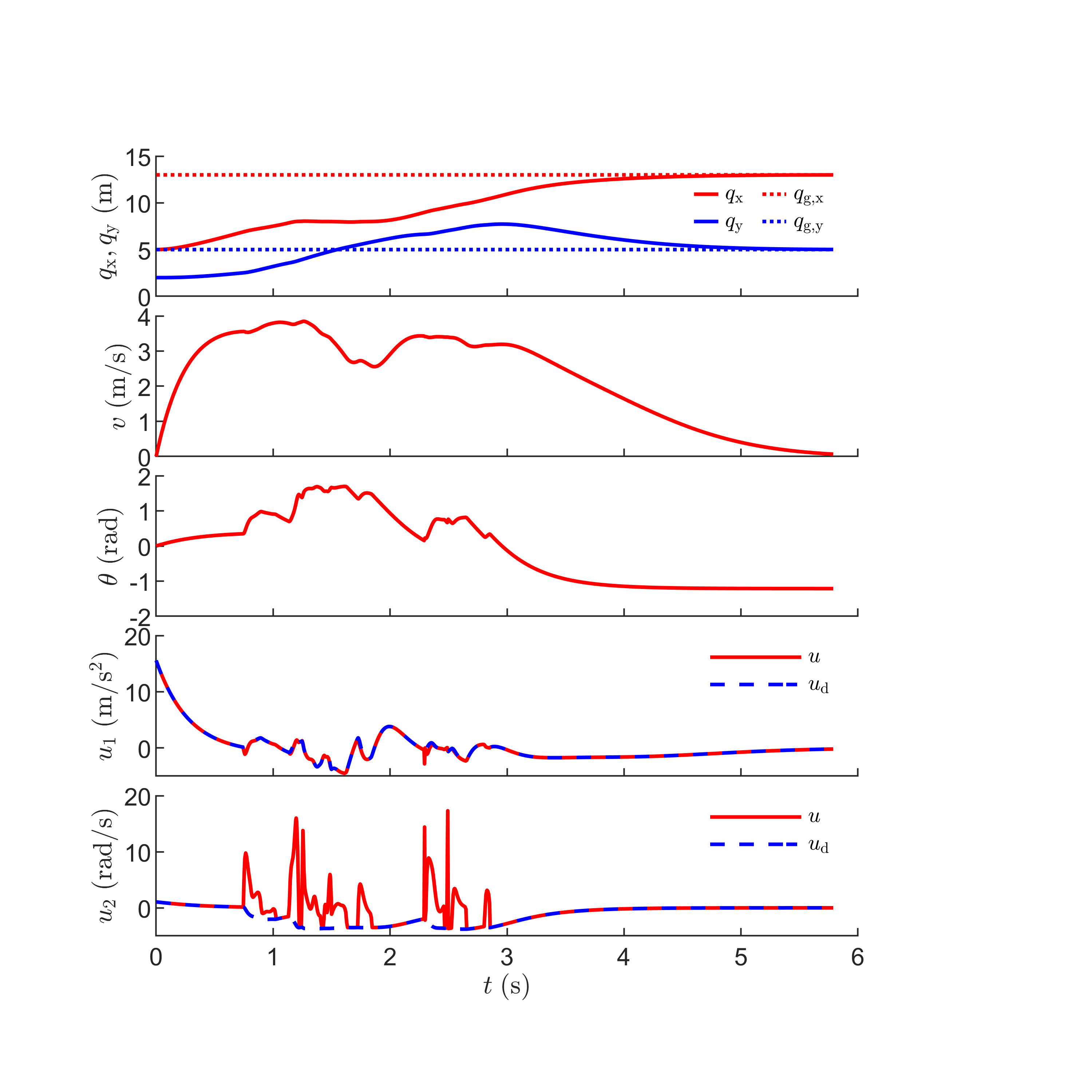}}
\caption{$q_\rmx$, $q_\rmy$, $v$, $\theta$, $u_\rmd$, and $u$ for $q_\rmg = [\,13\quad5\,]^\rmT$.}\label{fig:ex1_FOV_state}
\end{figure}

\begin{figure}[t!]
\center{\includegraphics[width=0.44\textwidth,clip=true,trim= 0.35in 0.3in 0.8in 0.6in] {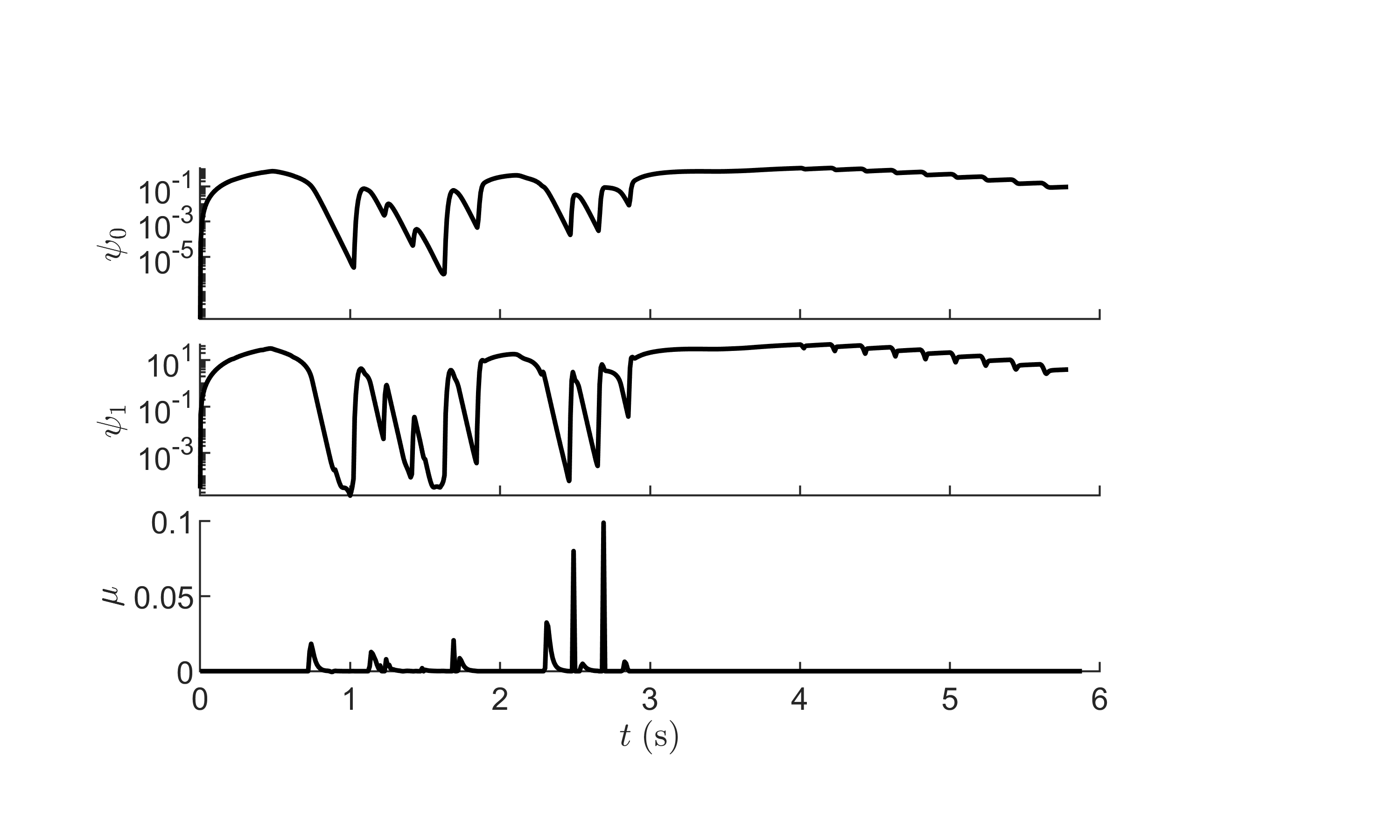}}
\caption{$\psi_0$, $\psi_1$, and $\mu$ for $q_\rmg = [\,13\quad5\,]^\rmT$ .}\label{fig:ex1_FOV_h}
\end{figure}

\section{Application to Quadrotor Aerial Vehicle}
\label{sec:Quadrotor}

Consider the attitude-stabilized quadrotor aerial vehicle modeled by 
\begin{align}
    &\dot{q} = p, \label{eq:uav dyn1}\\
    &m\dot{p} = FR\begin{bmatrix}
    0&0&1
\end{bmatrix}^\rmT - mg\begin{bmatrix}
    0&0&1
\end{bmatrix}^\rmT, \label{eq:uav dyn2}\\
    &\dot{R} = R[\omega]_\times, \label{eq:uav dyn3}\\ 
    &\dot{\omega} =
    \begin{bmatrix}
        1 & 0 & -\sin\theta \\
        0 & \cos\phi & \sin\phi\cos\theta \\ 
        0 & -\sin\phi & \cos\phi\cos\theta
    \end{bmatrix}
    \begin{bmatrix}
        k_1(\phi_\rmc-\phi) + k_2(-\dot{\phi}) \\ 
        k_1(\theta_\rmc-\theta) + k_2(-\dot{\theta}) \\
        k_3(\dot \psi_\rmc - \dot \psi)
    \end{bmatrix}, \label{eq:uav dyn4}\\
   &\dot{F} = k_4(F_\rmc - F),\label{eq:uav dyn5}
\end{align}
where $q(t)\in \mathbb{R}^3$ is the position of the quadrotor; 
$p(t) \in \mathbb{R}^3$ is the velocity; 
$R(t) \in \mbox{SO(3)}$ is the rotation matrix from a quadrotor's body frame to the inertial frame;
$\omega(t) \in \mathbb{R}^3$ is angular velocity; 
$[\omega(t)]_{\times} \in  \mbox{so(3)}$ is the skew-symmetric form of $\omega(t)$;  
$F(t) \in [0,\infty]$ is the scalar thrust force in the body frame;
and $\psi(t)$, $\theta(t)$, and $\phi(t)$ are the $3-2-1$ Euler angle sequence associated with $R(t)$. 
The quadrotor has mass $m = 0.1$ \si{\kilogram} and $g = 9.81$ \si{\meter\per\second^2} is the gravity acceleration. 
The inner-loop attitude controller has gains of $k_1 = \num{3.4e3} $, $k_2 = 116.67 $, and $k_3 = 1950$. 
The thrust gain is $k_4 = \num{3.9e3}$.
The command inputs to the attitude stabilized quadrotor \Cref{eq:uav dyn1,eq:uav dyn2,eq:uav dyn3,eq:uav dyn4,eq:uav dyn5} are  $\psi_\rmc$, $\theta_\rmc$, $\phi_\rmc$, $F_\rmc$ : $[0, \infty) \to \mathbb{R}$, which are the commanded yaw, pitch, roll, and thrust, respectively. 
These commanded inputs are given by $\psi_\rmc \equiv 0$ and 
\begin{gather}
\theta_\rmc \triangleq \arctan\left(\wfrac[1pt]{u_x}{u_z+g}\right),\quad 
\phi_\rmc \triangleq \arctan\left(\wfrac[1pt]{-u_y\cos\theta_\rmc}{u_z+g}\right), \label{eq:translational uav3}\\ 
F_\rmc \triangleq \wfrac[1pt]{(u_z+g)m}{\cos\phi_\rmc\cos\theta_\rmc},\label{eq:translational uav4}
\end{gather}
where $u = \begin{bmatrix} u_x & u_y & u_z \end{bmatrix}^\rmT$ is the translational acceleration command determined by the outer-loop control that is computed using the optimal and safe control~\Cref{eq:softmax_h,eq:HOCBF,eq:uclose,eq:ulambda,eq:omegabar,eq:dxt}.

Define $x \triangleq \begin{bmatrix} q^\rmT &  p^\rmT \end{bmatrix}^\rmT$.
Similar to the ground robot, the quadrotor is equipped with a perception system (e.g., LiDAR) that detects up to $\bar \ell$ points on objects that are: (i) in line of sight of the quadrotor; (ii) inside the field of view (FOV) of the perception system; and (iii) inside the detection radius $\bar r > 0$ of the perception system. 
However, the detection system is 3D rather than 2D. 
For all $k \in \BBN$, at time $t=kT$, the quadrotor obtains raw perception feedback in the form of $\ell_k \in \{0,1,\ldots,\bar \ell \}$ points given by $(r_{1,k},\theta_{1,k},\phi_{1,k}),\cdots,(r_{\ell_k,k},\theta_{\ell_k,k},\phi_{\ell_k,k})$, which are the spherical-coordinate positions of the detected points relative to the $q(kT)$ at the time of detection.
For all $i\in \{1,2,\ldots,\ell_k \}$, $r_{i,k} \in [0,\bar r]$, $\theta_{i,k} \in [0,2\pi)$, and $\phi_{i,k} \in [-\frac{\pi}{2},\frac{\pi}{2}]$.

For all $i\in \{1,2,\ldots, \ell_k \}$, the location of the detected point is 
\begin{equation*}
   c_{i,k} \triangleq q\left(kT\right) + r_{i,k} \begin{bmatrix}
        \sin\phi_{i,k}\cos\theta_{i,k} \\ \sin\phi_{i,k}\sin\theta_{i,k} \\ 
        \cos\phi_{i,k}
    \end{bmatrix}, 
\end{equation*}
and 
\begin{equation*}
   d_{i,k} \triangleq q(kT) +  \bar r \begin{bmatrix}
        \sin\phi_{i,k}\cos\theta_{i,k} \\ \sin\phi_{i,k}\sin\theta_{i,k} \\ 
        \cos\phi_{i,k}
    \end{bmatrix}
\end{equation*}
is the location of the point that is at the boundary of the detection radius and on the line between $q(kT)$ and $c_{i,k}$.

For each point $(r_{i,k},\theta_{i,k}, \phi_{i,k})$, we consider a function whose zero-level set is an ellipsoid that encircles $c_{i,k}$ and $d_{i,k}$.
Specifically, for all $i\in \{1,2,\ldots,\ell_k \}$, consider $\sigma_{i,k} \colon \BBR^6 \to \BBR$ given by \eqref{eq:ellipse}, where $\chi(x) \triangleq [ \, I_3 \quad 0_{3\times3} \, ]^\rmT x$, $a_{i,k}$ and $z_{i,k}$ are given by \eqref{eq:ellipse.3}, 
\begin{gather*}
    R_{i,k} \triangleq \begin{bmatrix}
        \cos\theta_{i,k}\sin\phi_{i,k} & \sin\theta_{i,k}\sin\phi_{i,k} & \cos\phi_{i,k} \\ -\sin\theta_{i,k}& \cos\theta_{i,k}&0 \\ 
        -\cos\theta_{i,k}\cos\phi_{i,k} & -\sin\theta_{i,k}\cos\phi_{i,k} & \sin\phi_{i,k}
    \end{bmatrix},\\
    P_{i,k} \triangleq \begin{bmatrix}
        a_{i,k}^{-2}&0&0 \\ 0 & z_{i,k}^{-2}&0 \\ 0 & 0&z_{i,k}^{-2}
    \end{bmatrix}, 
    %
\end{gather*}
and $\varepsilon_a > 0$ determines the size of the ellipsoid $\sigma_{i,k}(x)=0$. 
The area outside the ellipsoid is the zero-superlevel set of $\sigma_{i,k}$.

Let $\xi_k \colon \BBR^6 \to \BBR$ be a continuously differentiable function whose zero-superlevel set models the perception system's detection area (i.e., the detection radius and FOV).
In this section, we consider a perception system that has a $360^\circ$ azimuth FOV, $180^\circ$ elevation FOV, and detection radius $\bar r >0$.
Thus, we let $\xi_k(x) = \beta_k(x)$, where $\beta_k$ is given by \eqref{eq:circle}.
Note that limited FOV can be addressed with a method similar to that used in 
\Cref{sec:Dynamic_env}.

Finally, the perception feedback $b_k$ is constructed similar to in \Cref{sec:GroundRobot}. 
Specifically, $b_k$ is given by \eqref{eq:bk_ex1}, where $\rho>0$.

The control objective is for the quadrotor to move from its initial location to a goal location $q_\rmg = [ \, q_{\rmg,\rmx} \quad q_{\rmg,\rmy} \quad q_{\rmg,\rmz} \, ]^\rmT \in\BBR^3$ without violating safety (i.e., hitting an obstacle).
To accomplish this objective, we consider the desired control $u_\rmd(x)  \triangleq k_5 \tanh ( q_\rmg - q ) - k_6 p$, where $k_5 = 3$ and $k_6 = 2$.
Note $u_\rmd$ is a proportional-derivative control with saturation on the position-error feedback term.

This section uses the minimum-intervention approach with the cost \eqref{eq:J}, where $Q(t,x) = I_2$, $c(t,x) = -u_\rmd(x)$.
The perception update period is $T = 0.2$~s, and we use the perception feedback \eqref{eq:bk_ex1}, where $\xi_k = \beta_k$, $\rho = 30$, $\bar r = 5$~m, and $\varepsilon_a = \varepsilon_\beta = 0.15$~m. 
The maximum number of detected points is $\bar \ell = 300$, and $N = 2$.

We note that the attitude-stabilized quadrotor \cref{eq:uav dyn1,eq:uav dyn2,eq:uav dyn3,eq:uav dyn4,eq:uav dyn5} combined with \cref{eq:translational uav3,eq:translational uav4} can be approximated as the double integrator\cite{beard2008quadrotor}. 
The double integrator can be modeled by \eqref{eq:affine control}, where 
\begin{equation} \label{eq:approx2}
    f(x) = \begin{bmatrix}
    p \\
    0
    \end{bmatrix},
    \qquad
    g({x}) = \begin{bmatrix}
    0_{3\times3} \\
    I_3
    \end{bmatrix}. 
\end{equation}
\textcolor{black}{Using \eqref{eq:approx2}, it can be verified by direct calculation that $L_gb_k=0$.
Next, using an argument similar to the one used for the ground robot, it follows that $L_gL_f b_k(x) \ne 0$ for almost all $x \in\ \BBR^6$.
Hence, $b_k$ satisfies \ref{con3} and~\ref{con4} with $r=2$ for the approximate system model~\eqref{eq:approx2}.}
We implement the control~\Cref{eq:softmax_h,eq:HOCBF,eq:uclose,eq:ulambda,eq:omegabar,eq:dxt} with $\kappa = 30$, $\gamma = 200$, $\alpha_0(\psi_0) =40\psi_0$, $\alpha(\psi_1) =2.5\psi_1$, $f$ and $g$ given by \eqref{eq:approx2}, and $\eta$ given by Example~\ref{ex:g} where $r=2$ and $\nu = 2$. 
For sample-data implementation, the control is updated at $100$~Hz.

\Cref{fig:UAV_map_3D,fig:UAV_map_2D} show the map of unknown environment that the quadrotor aims to navigate and closed-loop trajectories for $q(0) = [\,8\quad -10 \quad 5\,]^\rmT$~m and $p(0) = 0$~m/s with 3 different goals locations $q_\rmg = [\,0 \quad 10 \quad 5\,]^\rmT$~m, $q_\rmg = [\,-10 \quad 10 \quad 8\,]^\rmT$~m, and $q_\rmg = [\,-10 \quad 0 \quad 3 \,]^\rmT$~m. 
In all cases, the quadrotor's position converges to the goal location while satisfying safety constraints.  
\Cref{fig:UAV_h,fig:UAV_pos,fig:UAV_state,fig:UAV_input} provide time histories of relevant signals for the case where $q_\rmg = [\,0\quad 10 \quad 5 \,]^\rmT$~m. 
\Cref{fig:UAV_h} shows that $\psi_0$ and $\psi_1$ are positive for all time, which demonstrates that for all time $t$, the trajectory is in $\SC(t) \subseteq \SSS_\rms(t)$. \Cref{fig:UAV_input} shows $u$ deviates from $u_\rmd$ in order to satisfy safety.
 
\begin{figure}[t!]
\center{\includegraphics[width=0.41\textwidth,clip=true,trim= 0.2in 0.95in 0.9in 1.7in] {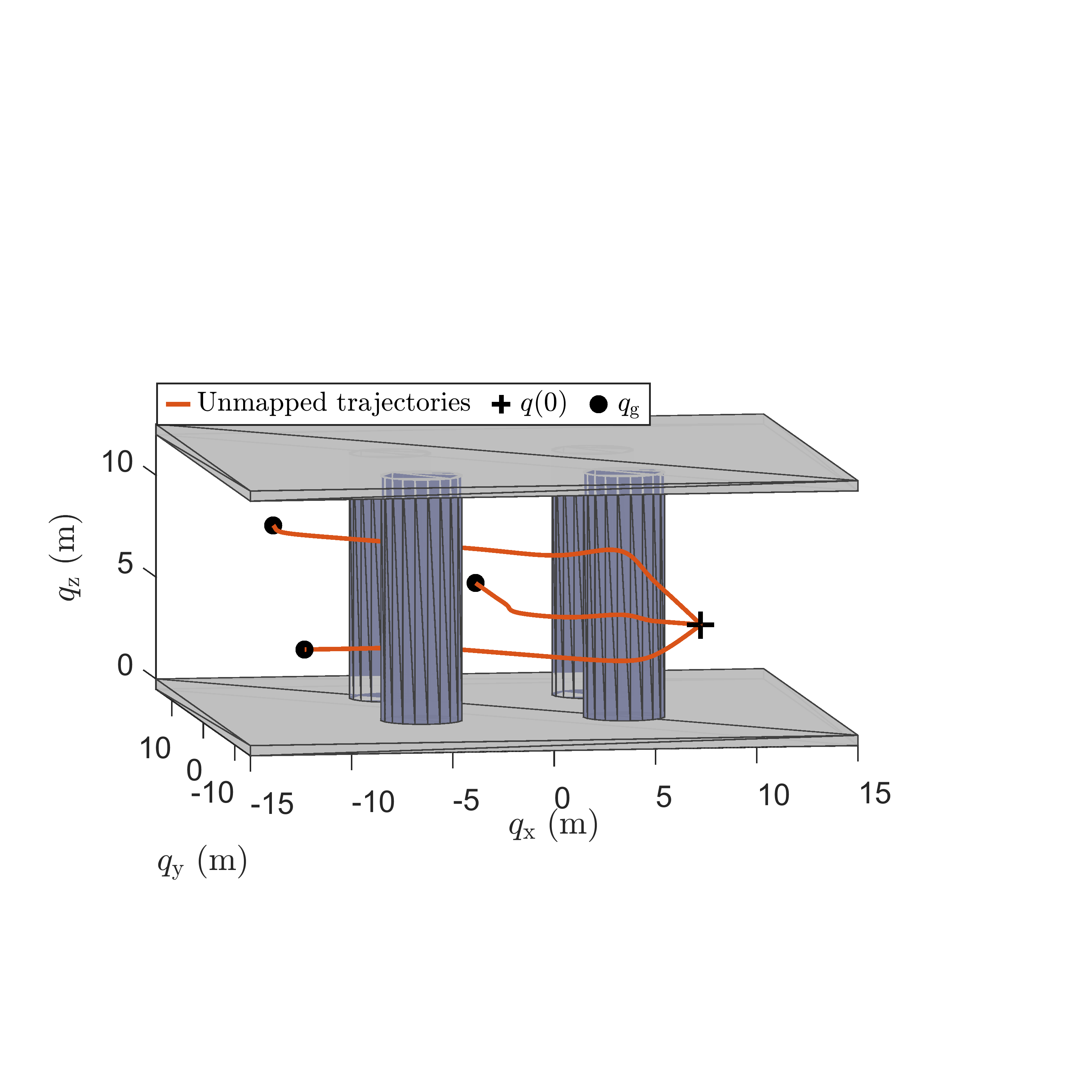}}
\caption{Three closed-loop trajectories using the control~\Cref{eq:softmax_h,eq:HOCBF,eq:uclose,eq:ulambda,eq:omegabar,eq:dxt} with the perception feedback $b_k$.}\label{fig:UAV_map_3D}
\end{figure}

\begin{figure}[t!]
\center{\includegraphics[width=0.41\textwidth,clip=true,trim= 0.25in 0.3in 1in 1.0in] {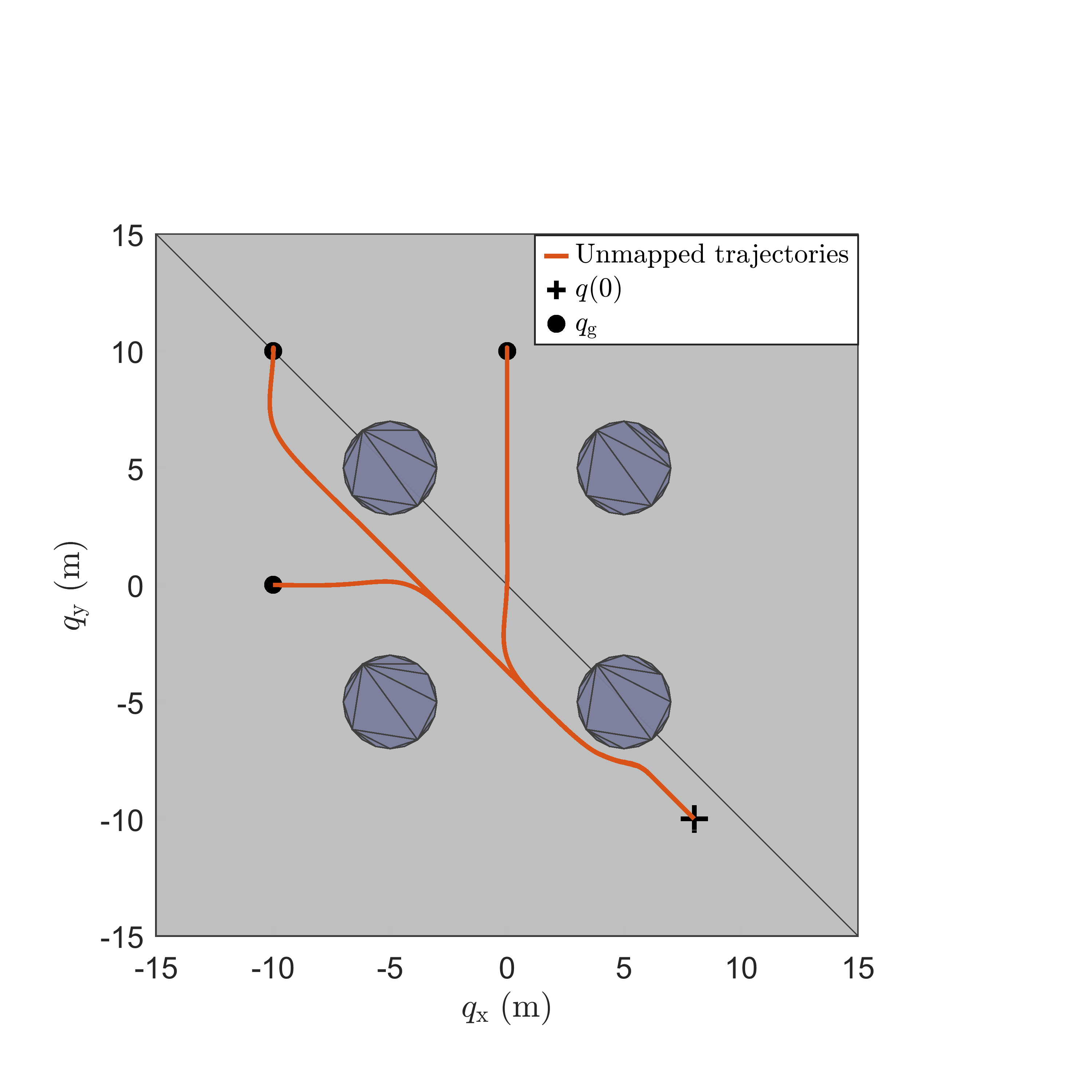}}
\caption{Top-down view of three closed-loop trajectories using the control~\Cref{eq:softmax_h,eq:HOCBF,eq:uclose,eq:ulambda,eq:omegabar,eq:dxt} with the perception feedback $b_k$.}\label{fig:UAV_map_2D}
\end{figure}

\begin{figure}[t!]
\center{\includegraphics[width=0.44\textwidth,clip=true,trim=0.3in 0.27in 1.0in 0.5in] {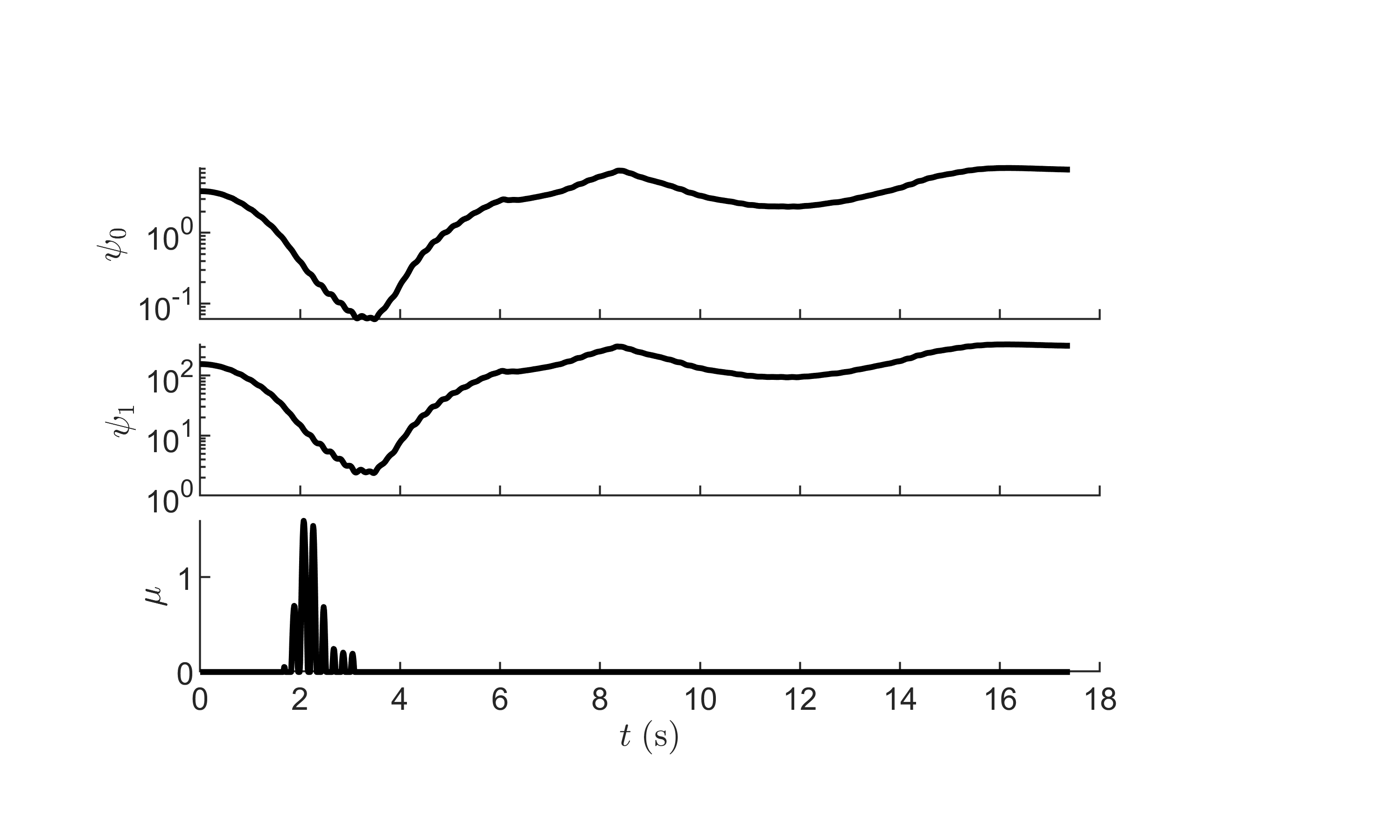}}
\caption{$\psi_0$, $\psi_1$, and $\mu$ for $q_\rmg = [\,0\quad 10 \quad 5 \,]^\rmT$~m.}\label{fig:UAV_h}
\end{figure} 
\begin{figure}[t!]
\center{\includegraphics[width=0.44\textwidth,clip=true,trim= 0.35in 0.27in 1.0in 0.5in] {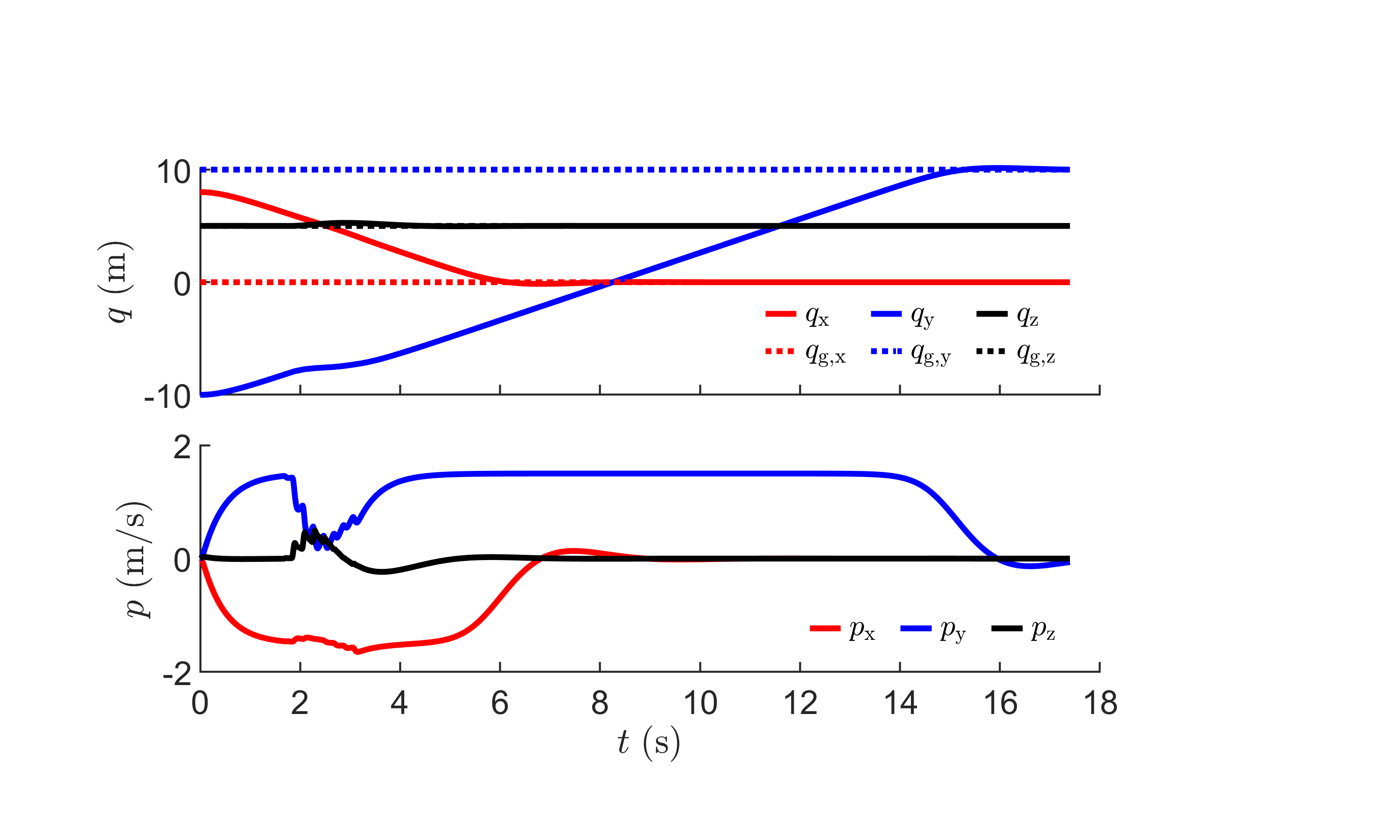}}
\caption{$q$ and $p$ for $q_\rmg = [\,0\quad 10 \quad 5 \,]^\rmT$~m.}\label{fig:UAV_pos}
\end{figure} 
\begin{figure}[t!]
\center{\includegraphics[width=0.44\textwidth,clip=true,trim= 0.35in 0.35in 1.0in 0.5in] {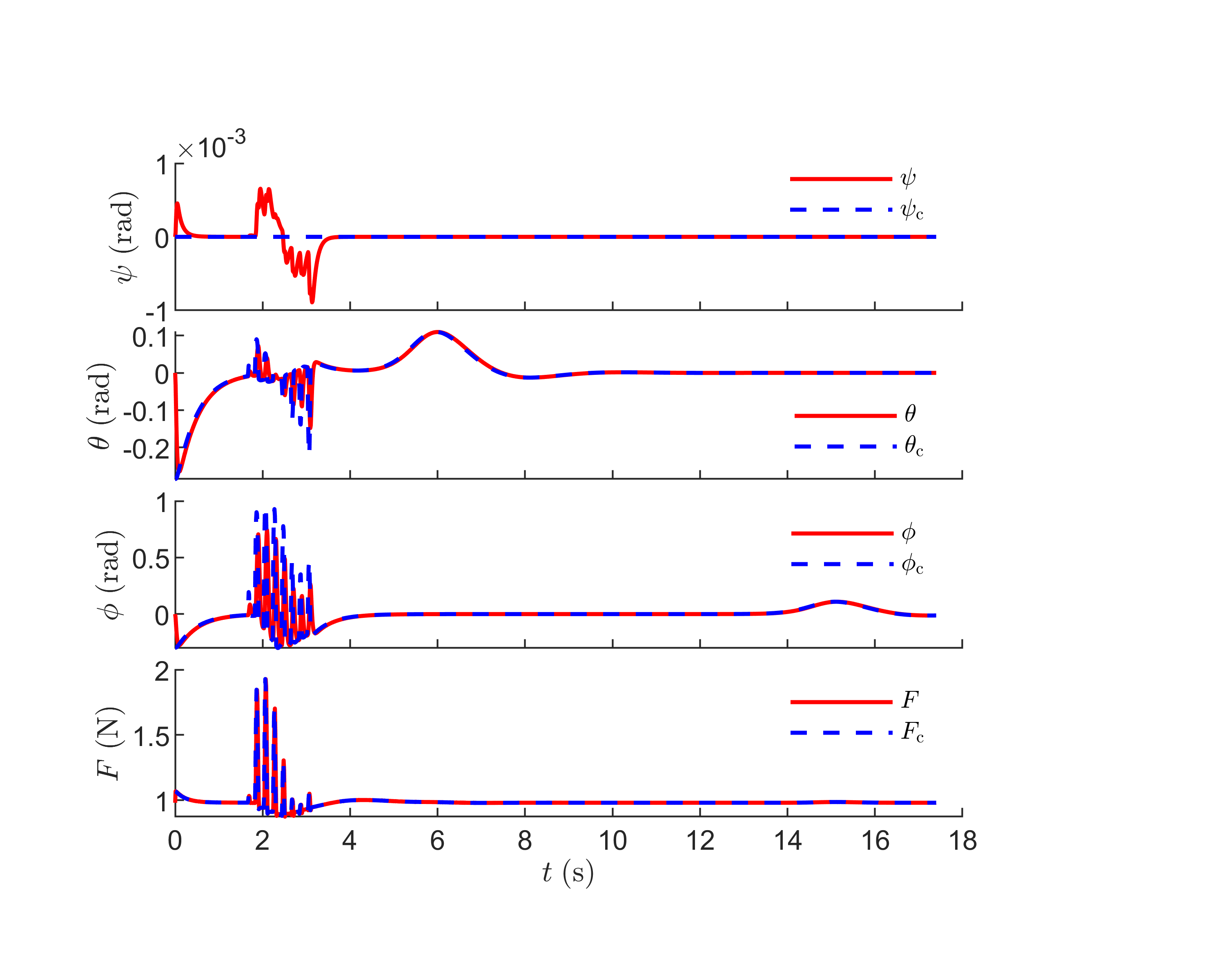}}
\caption{$\psi$, $\theta$, $\phi$, and $F$ for $q_\rmg = [\,0\quad 10 \quad 5 \,]^\rmT$~m.}\label{fig:UAV_state}
\end{figure} 
\begin{figure}[t!]
\center{\includegraphics[width=0.44\textwidth,clip=true,trim= 0.35in 0.35in 1.0in 0.5in] {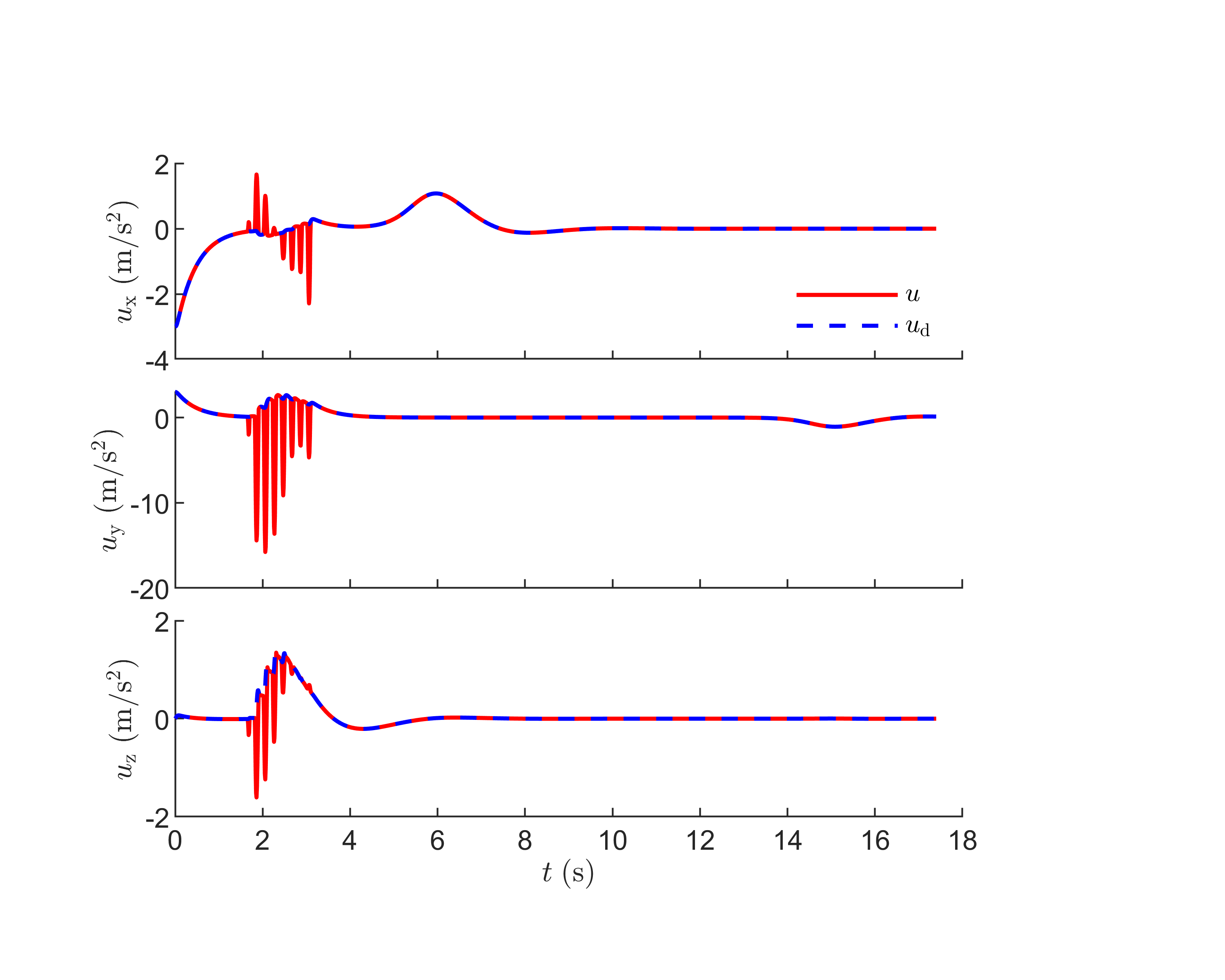}}
\caption{$u_x$, $u_y$, and $u_z$ for $q_\rmg = [\,0\quad 10 \quad 5 \,]^\rmT$~m.}\label{fig:UAV_input}
\end{figure}

\section{Concluding Remarks}

This article presents a new closed-form optimal feedback control method that ensures safety in an \textit{a priori} unknown and potentially dynamic environment.
A key enabling element is the smooth time-varying soft-maximum barrier function $\psi_0$ given by  \eqref{eq:softmax_h}, whose zero-superlevel set $\SC_0(t)$ is a subset of the safe set $\SSS_\rms(t)$. 
The function $\psi_0$ composes the $N+1$ most recently obtained perception feedback $b_k,\ldots,b_{k-N}$ and uses the convex combination $\eta\left(\textstyle\frac{t}{T}-k\right)b_k  + \left[ 1-\eta\left(\textstyle\frac{t}{T}-k\right)\right ] b_{k-N}$ to incorporated the newest perception feedback $b_k$ and remove the oldest perception feedback $b_{k-N}$. 
Other choices of $\psi_0$ also have these required properties.
For example, 
\begin{align}
    \psi_0(t,x) &\triangleq  \left[ 1-\eta\left(\textstyle\frac{t}{T}-k\right)\right]\mbox{softmax}_\kappa \left (b_{k-1}(x), \cdots, b_{k-N}(x) \right )\nn \\
    &\quad +\eta\left(\textstyle\frac{t}{T}-k\right) \mbox{softmax}_\kappa \left ( b_{k}(x), \cdots, b_{k-N+1}(x) \right ) \label{eq:alternative}. 
\end{align}
can be used in place of \eqref{eq:softmax_h}. 
If $N=1$, then \eqref{eq:alternative} is equivalent to \eqref{eq:softmax_h}; however, these choices of $\psi_0$ differ for $N>1$. 
The main results of this article (\Cref{prop.psi_i_properties,prop:CFI,prop:ClosedForm,prop:Forward_Invariant} and \Cref{thm:global minimizer,Th:Main th}) still hold if $\psi_0$ is given by \eqref{eq:alternative} instead of \eqref{eq:softmax_h}.
For the simulation in this article, $\psi_0$ given by \eqref{eq:alternative} performs very similar to $\psi_0$ given by \eqref{eq:softmax_h}.
However, the relative advantages and disadvantages of different $\psi_0$ is an open question. 

The main results on the optimal and safe feedback control~\Cref{eq:uclose,eq:ulambda,eq:omegabar,eq:dxt} rely on Assumption~\ref{con6} that $L_g \psi_{r-1}$ is nonzero on a subset of the boundary of the zero-superlevel set of $\psi_{r-1}$ (specifically, on $[0,\infty) \times \SB(t)$). 
\Cref{remark:Lgpsi} discusses sufficient conditions for~\ref{con6}. 
However, determining necessary and sufficient conditions for~\ref{con6} is an open question.

\appendices

\section{Derivation of Optimal Control \cref{eq:uclose,eq:mu_close,eq:ulambda,eq:omegabar,eq:dxt}}
\label{appendix:CformControl}

This appendix uses the first-order necessary conditions to derive the expressions \cref{eq:uclose,eq:mu_close,eq:ulambda,eq:omegabar,eq:dxt} for the unique global minimizer of $\SJ(t,x,\hat u,\hat \mu)$ subject to $b(t,x,\hat{u},\hat{\mu}) \ge 0$, where $\SJ$ and $b$ are given by \cref{eq:SJ,eq:safety_constraint}.

To derive the minimizer \Cref{eq:uclose,eq:mu_close}, let $(t,x) \in [0,\infty) \times \BBR^n$, and define $u_\rmg \triangleq-Q(t,x)^{-1} c(t,x)$, which is the unique global minimizer of the unconstrained cost $J$ given by \eqref{eq:J}.
Thus, $( u_\rmg, 0 )$ is the unique global minimizer of the unconstrained cost $\SJ$ given by \eqref{eq:SJ}.

Let $\left(u_*, \mu_*\right) \in \mathbb{R}^{m} \times \mathbb{R}$ denote the unique global minimizer of $\SJ\left(t, x_, \hat{u}, \hat{\mu} \right)$ subject to $b\left(t,x, \hat{u}, \hat{\mu}\right) \geq 0$. 
Note \cref{eq:omegabar,eq:safety_constraint} imply that 
\begin{equation}
    b(t,x,u_\rmg, 0) = \omega(t,x). \label{eq:b=omega}
\end{equation}
We consider 2 cases: (i) $\omega(t,x) \ge 0$, which implies $( u_\rmg, 0 )$ satisfies the constraint; and (ii) $\omega(t,x) < 0$, which implies $( u_\rmg, 0 )$ does not satisfy the constraint.

First, consider case (i) $\omega(t,x) \ge 0$, and it follows from \eqref{eq:b=omega} that $b(t,x,u_\rmg, 0) \ge 0$. 
Since, in addition, $\left(u_\rmg, 0\right)$ is the unique global minimizer of the unconstrained cost $\SJ\left(t, x, \hat{u}, \hat{\mu}\right)$, it follows that $u_*=u_\rmg$ and $\mu_*= 0$, which confirms \Cref{eq:uclose,eq:mu_close} for case (i) $\omega(t,x) \ge 0$.

Next, consider case (ii) $\omega(t,x) < 0$, and it follows from \eqref{eq:b=omega} that $b(t,x,u_\rmg, 0) < 0$. 
Thus, $b(t,x, u_*, \mu_* )=0$. 
Define the Lagrangian $\mathcal{L}(\hat{u}, \hat{\mu}, \hat{\lambda}) \triangleq \SJ\left(t, x, \hat{u}, \hat{\mu}\right)-\hat{\lambda} b\left(t, x, \hat u, \hat \mu \right)$, and let $\lambda_* \in \mathbb{R}$ be such that $(u_*, \mu_*, \lambda_* )$ is a stationary point of $\SL$. 
Next, we evaluate ${\partial \mathcal{L}}/{\partial \hat{u}}$, ${\partial \mathcal{L}}/{\partial \hat{\mu}}$, and ${\partial \mathcal{L}}/{\partial \hat{\lambda}}$ at $(u_*, \mu_*, \lambda_* )$; set equal to zero; and solve for $u_*, \mu_*, \lambda_*$ to obtain
\begin{gather*}
u_*=-Q(t, x)^{-1}\left(c(t, x)-\lambda_* L_g\psi_{r-1}\left(t,x\right)^{\mathrm{T}} \right), \\
\mu_*=\frac{\psi_{r-1}\left(t,x\right) \lambda_*}{\gamma},\qquad
\lambda_*=\frac{-\omega\left(t,x\right)}{d\left(t,x\right)},
\end{gather*}
where it follows from \ref{propA.1} of \Cref{prop:ClosedForm} that $d(t,x) > 0$.
This confirms \Cref{eq:uclose,eq:mu_close,eq:ulambda} for case (ii) $\omega(t,x) < 0$.

\section{Proofs of \Cref{fact:softmin_limit,prop:softmin_softmax_sets,prop:softmin_max_numerical,prop.h,prop.psi_i_properties}}
\label{appendix:proposition proofs}

\begin{proof}[\indent Proof of \Cref{fact:softmin_limit}]
Define $\ubar{z} \triangleq \min \, \{z_1,\cdots,z_N\}$. Since the exponential is nonnegative and strictly increasing, it follows that $e^{-\kappa \ubar{z}} \leq \sum_{i=1}^{N} e^{-\kappa z_i} \leq Ne^{-\kappa \ubar{z}}$, and taking the logarithm yields $-\kappa \ubar{z} \leq \log\sum_{i=1}^{N} e^{-\kappa z_i} \leq \log N -\kappa \ubar{z}$.
Then, dividing by $-\kappa$ confirms \eqref{eq:softmin_inequality}.

Next, define $\bar z \triangleq \max \, \{z_1,\cdots,z_N\}$. 
Since the exponential is nonnegative and strictly increasing, it follows that $e^{\kappa \bar{z}} \leq \sum_{i=1}^{N} e^{\kappa z_i} \leq Ne^{\kappa \bar{z}}$, and taking the logarithm yields $ \kappa \bar{z} \leq \log\sum_{i=1}^{N} e^{\kappa z_i} \leq \log N + \kappa \bar{z}$.
Then, subtracting $\log N$ and dividing by $\kappa$ confirms \eqref{eq:softmax_inequality}.
\end{proof} 

\begin{proof}[\indent Proof of \Cref{prop:softmin_softmax_sets}]
To prove $\SX_{\kappa} \subseteq \bigcap_{i=1}^N \SD_i$, let $x \in \SX_\kappa$. 
Since $\kappa > 0$, it follows from \Cref{eq:softmin,prop2.2} that $\sum_{i=1}^Ne^{-\kappa \zeta_i(x)}\le 1$, which implies that for all $i \in \{1,\cdots N \}$, $e^{-\kappa \zeta_i(x)} \le 1$.
Thus, for all $i \in \{1,\cdots N \}$, $\zeta_i(x) \ge 0$,  which implies that $x \in \SD_i$. 
Hence, $\SX_{\kappa} \subseteq \bigcap_{i=1}^N \SD_i$.

Next, to prove $\SY_{\kappa} \subseteq \bigcup_{i=1}^N \SD_i$, let $y \in \SY_\kappa$. 
Since $\kappa > 0$, it follows from \Cref{eq:softmin,prop2.3} that $\sum_{i=1}^Ne^{\kappa \zeta_i(y)}\ge N$, which implies that there exists $i_* \in \{1,\cdots N \}$ such that $e^{\kappa \zeta_{i_*}(y)} \ge 1$. 
Thus, $\zeta_{i_*}(y) \ge 0$, which implies that $y \in \SD_{i_*}$. 
Hence, $\SY_{\kappa} \subseteq \bigcup_{i=1}^N \SD_i$.

To prove the final statement of the result, note that \Cref{fact:softmin_limit} implies that $\lim_{\kappa \to \infty} \mbox{softmin}_\kappa(z_1,\cdots,z_N) = \min \, \{z_1,\cdots,z_N\}$ and $\lim_{\kappa \to \infty} \mbox{softmax}_\kappa(z_1,\cdots,z_N) = \max \, \{z_1,\cdots,z_N\}$, which combined with \cref{prop2.1,prop2.2,prop2.3} implies that as $\kappa \to \infty$, $\SX_{\kappa} \to \bigcap_{i=1}^N \SD_i$ and $\SY_{\kappa} \to \bigcup_{i=1}^N \SD_i$.
\end{proof}

\begin{proof}[\indent Proof of \Cref{prop:softmin_max_numerical}]
It follows from \eqref{eq:softmin}
\begin{align*}
   \mbox{softmin}_\kappa (z_1,\cdots,z_N) &= -\frac{1}{\kappa}\log\sum_{i=1}^Ne^{-\kappa z_i} \\
   & = -\frac{1}{\kappa}\log\sum_{i=1}^Ne^{-\kappa (z_i - \ubar{z} + \ubar{z} )} \\
   & = -\frac{1}{\kappa}\log e^{-\kappa \ubar{z}} -\frac{1}{\kappa}\log\sum_{i=1}^Ne^{-\kappa (z_i - \ubar{z})} \\ 
   & = \ubar{z} + \mbox{softmin}_\kappa(z_1-\ubar{z},\cdots,z_N-\ubar{z}).
\end{align*}
The same steps are used to prove $\mbox{softmax}_\kappa(z_1, \cdots, z_N) = \bar{z} + \mbox{softmax}_\kappa(z_1-\bar{z}, \cdots, z_N-\bar{z})$.
\end{proof}

\begin{proof}[\indent Proof of \Cref{prop.h}]
To prove \ref{prop.h.1}, note that $b_k$ is $r$-times continuously differentiable on $\BBR^n$, and $\eta$ is $r$-times continuously differentiable on $[0,\infty)$. 
Thus, \Cref{eq:softmax,eq:softmax_h} and \ref{con:con1_g}--\ref{con: con4_g} imply that $\psi_0$ is $r$-times continuously differentiable on $[0,\infty) \times \BBR^n$, which proves \ref{prop.h.1}.

Next, it follows from \eqref{eq:softmax_h} and direct computation that for all $i \in \{0,1,\cdots, r-1 \}$,
\begin{align}
L_f^i\psi_0(t,x) &= F_i (t, b_k(x),\cdots,b_{k-N}(x),L_fb_k(x),\cdots,\nn\\
    &\qquad L_fb_{k-N}(x),\cdots,L_f^ib_{k}(x) ,\cdots,L_f^ib_{k-N}(x)).\label{eq:induction.prop2b}
\end{align}
where $F_0 : [0,\infty) \times \BBR^{N+1}\to \BBR$ is defined by
\begin{equation}
F_0(t, b_k(x),\cdots,b_{k-N}(x)) \triangleq \psi_0(t,x), \label{eq:induction.prop2b.1}
\end{equation}
and for all $i \in \{1,\cdots, r-1 \}$, $F_i : [0,\infty) \times \BBR^{(N+1)(i+1)}\to \BBR$ is defined by
\begin{equation}
F_{i} \triangleq  \sum_{j = 0}^{N}\sum_{p=0}^{i-1}\textstyle\frac{\partial F_{i-1}}{\partial L_f^p b_{k-j}}L_f^{i}b_{k-j}(x), \label{eq:induction.prop2b.2}
\end{equation}
where the arguments of $F_i$ are omitted from \eqref{eq:induction.prop2b.2} for brevity.

To prove \ref{prop.h.2}, let  $i \in \{0,1,\cdots,r-2\}$, and \cref{eq:induction.prop2b} implies
 \begin{equation*}
     L_gL_f^i\psi_0(t,x) = L_g F_i = \sum_{j = 0}^{N}\sum_{p=0}^{i}\textstyle\frac{\partial F_i}{\partial L_f^p b_{k-j}}L_gL_f^{p}b_{k-j}(x).
 \end{equation*}
Since \ref{con3} implies that for all $j \in \{0,\cdots ,N \}$ and all $p \in \{0,\cdots,i\}$, $L_gL_f^{p}b_{k-j}(x) = 0$, it follows that $L_gL_f^i\psi_0(t,x) =0$, which proves \ref{prop.h.2}.

To prove \ref{prop.h.3}, it follows from \eqref{eq:induction.prop2b} and \ref{con3} that 
 \begin{equation}
     L_g L_f^{r-1}\psi_0(t,x) = \sum_{j = 0}^{N} \textstyle\frac{\partial F_{r-1}}{\partial L_f^{r-1} b_{k-j}} L_g L_f^{r-1}b_{k-j}(x). \label{eq:induction.prop2b.3}
 \end{equation}
Next, \eqref{eq:induction.prop2b.2} and \eqref{eq:induction.prop2b.1} imply that for all $j \in \{0,\ldots,N\}$,
 \begin{align*}
\textstyle\frac{\partial F_{r-1}}{\partial L_f^{r-1} b_{k-j}} = \textstyle\frac{\partial F_{r-2}}{\partial L_f^{r-2} b_{k-j}} = \cdots = \textstyle\frac{\partial F_{0}}{\partial L_f^{0} b_{k-j}} = \textstyle\frac{\partial \psi_{0}}{\partial b_{k-j}},
 \end{align*}
 which combined with \eqref{eq:induction.prop2b.3} yields
 \begin{align*}
     L_g L_f^{r-1}\psi_0(t,x) &= \sum_{j = 0}^{N} \textstyle\frac{\partial \psi_0}{\partial b_{k-j}} L_g L_f^{r-1}b_{k-j}(x)\nn\\
         & = \sum_{j=0}^{N} \mu_{j}(t,x) L_gL_f^{r-1} b_{k-j}(x),
 \end{align*}
which proves \ref{prop.h.3}.

To prove \ref{prop.h.4}, since the exponential is nonnegative, and $\eta\left(\frac{t}{T}-k\right)$ and $1-\eta\left(\frac{t}{T}-k\right)$ are nonnegative, it follows from \Cref{eq:mu0,eq:muN,eq:muj} that for $j \in \{0,1,\cdots,N\}$, $\mu_j(t,x) \ge 0$.
Next, using \Cref{eq:mu0,eq:muN,eq:muj} and direct calculation yields $\sum_{j=0}^N \mu_j(t,x)~=1$.
\end{proof}


The next lemma is needed for the proof of \Cref{prop.psi_i_properties}.

\begin{lemma}\label{lemma:Fij}\rm
Assume $r \ge 2$.
Then, for all $i \in \{0,1,\cdots, r-2 \}$ and all $j \in \{0,1,\cdots,r-2-i\}$, there exists continuously differentiable $F_{i,j} : \BBR^{j+1}\to \BBR$ such that 
\begin{equation}\label{eq:Fij.1}
    L_f^{j}\alpha_i(\psi_i(t,x)) = F_{i,j} (\psi_i(t,x), L_f\psi_i(t,x),\cdots,L_f^j\psi_i(t,x)).
\end{equation}
\end{lemma}

\begin{proof}[\indent Proof]
Let $i \in \{0,1,\cdots, r-2 \}$, and we use induction on $j$.
First, let $F_{i,0}(\psi_i) = \alpha_i(\psi_i)$, which is continuously differentiable because $\alpha_i$ is $(r-1-i)$-times continuously differentiable. 
Note that $L_f^0\alpha_i(\psi_i) = F_{i,0}(\psi_i)$, which confirms \eqref{eq:Fij.1} for $j=0$.
Next, assume for induction that \eqref{eq:Fij.1} holds for $j = \ell \in \{0,1,\cdots,r-3-i\}$. 
Thus, 
\begin{equation}
    L_f^{\ell+1}\alpha_i(\psi_i) =L_f L_f^{\ell}\alpha_i(\psi_i) = L_f F_{i,\ell} = \sum_{p=0}^{\ell}\textstyle\frac{\partial F_{i,\ell}}{\partial L_f^p \psi_i} L_f^{p+1} \psi_i, \label{eq:Fij.2}
\end{equation}
where the arguments of $F_{i,\ell}$ are omitted for brevity. 
Let 
\begin{equation}
F_{i,\ell+1} =  \sum_{p=0}^{{\ell}}\textstyle\frac{\partial F_{i,\ell}}{\partial L_f^p \psi_l} L_f^{p+1} \psi_l, \label{eq:Fij.3}
\end{equation}
and \Cref{eq:Fij.2,eq:Fij.3} imply that $L_f^{\ell+1}\alpha_i(\psi_i) = F_{i,\ell+1}$. 
Since, in addition, $\alpha_i$ is $(r-1-i)$-times continuously differentiable, it follows that $F_{i,\ell+1}$ is continuously differentiable, which confirms \eqref{eq:Fij.1} for $j=\ell +1$. 
\end{proof}

\begin{proof}[\indent Proof of \Cref{prop.psi_i_properties}]
We use induction on $i$ to show that for all $i\in \{0,1,\ldots,r-2\}$ and all $j \in \{0,1,\ldots,r-2-i\}$, 
\begin{equation}
    L_gL_f^j\psi_i = 0. \label{prop.psi_i_properties.1}
\end{equation}
First, \ref{prop.h.2} of \Cref{prop.h} implies that for all $j \in \{0,1,\ldots,r-2\}$, $L_g L_f^{j}  \psi_0 = 0$, which confirms \eqref{prop.psi_i_properties.1} for $i=0$.
Next, let $i=\ell \in \{0,1,\cdots, r-3 \}$, and assume for induction that for all $j \in \{0,1,\ldots,r-2-\ell\}$, $L_gL_f^{j} \psi_\ell = 0$. 
Thus, using \Cref{eq:HOCBF} implies that for all $j \in \{0,1,\ldots,r-3-\ell\}$,
\begin{align}
    L_gL_f^{j} \psi_{\ell+1} &= L_gL_f^{j} \left [ \textstyle \frac{\partial \psi_{\ell}}{\partial t} + L_f \psi_{\ell} +\alpha_\ell(\psi_\ell) \right ]\nn\\
    &= \textstyle \frac{\partial}{\partial t} L_gL_f^{j} \psi_{\ell}+ L_gL_f^{j+1} \psi_{\ell} + L_gL_f^{j} \alpha_\ell(\psi_\ell)\nn\\
    &= L_gL_f^{j} \alpha_\ell(\psi_\ell). \label{prop.psi_i_properties.2}
\end{align}
\Cref{lemma:Fij} implies that for  $j \in \{0,1,\cdots,r-3-\ell\}$, there exists continuously differentiable $F_{\ell,j} : \BBR^{j+1}\to \BBR$ such that $L_f^{j}\alpha_\ell(\psi_\ell(t,x)) = F_{\ell,j} (\psi_\ell(t,x), L_f\psi_\ell(t,x), \cdots, L_f^j\psi_\ell(t,x))$, which combined with \eqref{prop.psi_i_properties.2} implies that for all $j \in \{0,1,\ldots,r-3-\ell\}$,
\begin{equation}
    L_gL_f^{j} \psi_{\ell+1} = L_g F_{\ell,j} (\psi_\ell,\cdots,L_f^j\psi_\ell) = \sum_{p=0}^{j}\textstyle\frac{\partial F_{\ell,j}}{\partial L_f^p \psi_\ell} L_g L_f^{p} \psi_\ell,\label{prop.psi_i_properties.4}
 \end{equation}
 where the arguments of $F_{\ell,j}$ are omitted for brevity. 
Since for all $j \in \{0,1,\ldots,r-2-\ell\}$, $L_gL_f^{j} \psi_\ell = 0$, it follows from \eqref{prop.psi_i_properties.4} that 
 for all $j \in \{0,1,\ldots,r-3-\ell\}$, $L_gL_f^{j} \psi_{\ell+1} = 0$, which confirms \eqref{prop.psi_i_properties.1} for $i=\ell+1$.
Thus, for all $i\in \{0,1,\ldots,r-2\}$, $L_g \psi_i=0$.


To prove that $L_g \psi_{r-1} = L_gL_f^{r-1}\psi_{0}$, note that since $L_g \psi_{r-2} =0$, it follows from \Cref{eq:HOCBF} that 
\begin{align}
L_g \psi_{r-1} &= \textstyle\frac{\partial}{\partial t} L_g \psi_{r-2} + L_gL_f\psi_{r-2} + \alpha_{r-2}^\prime (\psi_{r-2}) L_g \psi_{r-2} \nn\\
    &= L_gL_f\psi_{r-2}, \label{eq:prop.psi_i_properties.5}
\end{align}
Since \Cref{prop.psi_i_properties.1,prop.psi_i_properties.2} imply that for all $i \in \{0,1,\cdots,r-3 \}$, $L_gL_f^{r-2-i} \psi_i =0$ and $L_g L_f^{r-2-i}\alpha_{i}(\psi_{i}) = 0$, it follows from \Cref{eq:prop.psi_i_properties.5,eq:HOCBF} that 
\begin{align*}
    L_g \psi_{r-1} &= L_g L_f \left [ \textstyle\frac{\partial}{\partial t} \psi_{r-3} + L_f \psi_{r-3} + \alpha_{r-3} (\psi_{r-3}) \right ] \\
    &= L_gL_f^2 \psi_{r-3} \\
    &= L_g L_f^2 \left [\textstyle\frac{\partial}{\partial t} \psi_{r-4} + L_f \psi_{r-4} + \alpha_{r-4} (\psi_{r-4}) \right ] \\
    &= L_gL_f^3 \psi_{r-4} \\
    &\quad \vdots\\
    & = L_gL_f^{r-1}\psi_{0},
\end{align*}
which confirms the result. 
\end{proof}

\bibliographystyle{ieeetr}
\bibliography{Softmax_LiDAR.bib} 

\begin{thebibliography}{10}

\bibitem{hudson2021heterogeneous}
N.~Hudson {\em et~al.}, ``Heterogeneous ground and air platforms, homogeneous sensing: Team {CSIRO} {Data61}'s approach to the {DARPA} subterranean challenge,'' {\em arXiv preprint arXiv:2104.09053}, 2021.

\bibitem{kress2009temporal}
H.~Kress-Gazit, G.~E. Fainekos, and G.~J. Pappas, ``Temporal-logic-based reactive mission and motion planning,'' {\em IEEE Trans. Robotics}, vol.~25, no.~6, pp.~1370--1381, 2009.

\bibitem{schwarting2018planning}
W.~Schwarting, J.~Alonso-Mora, and D.~Rus, ``Planning and decision-making for autonomous vehicles,'' {\em Annual Review of Contr., Robotics, Autom. Sys.}, vol.~1, pp.~187--210, 2018.

\bibitem{tang2010novel}
L.~Tang, S.~Dian, G.~Gu, K.~Zhou, S.~Wang, and X.~Feng, ``A novel potential field method for obstacle avoidance and path planning of mobile robot,'' in {\em Int. Conf. Comp. Sci. Info. Tech.}, pp.~633--637, IEEE, 2010.

\bibitem{kirven2021autonomous}
T.~Kirven and J.~B. Hoagg, ``Autonomous quadrotor collision avoidance and destination seeking in a gps-denied environment,'' {\em Autonomous Robots}, vol.~45, pp.~99--118, 2021.

\bibitem{sunkara2019collision}
V.~Sunkara, A.~Chakravarthy, and D.~Ghose, ``Collision avoidance of arbitrarily shaped deforming objects using collision cones,'' {\em IEEE Robotics and Autom. Lett.}, vol.~4, no.~2, pp.~2156--2163, 2019.

\bibitem{chen2018hamilton}
M.~Chen and C.~J. Tomlin, ``Hamilton--jacobi reachability: Some recent theoretical advances and applications in unmanned airspace management,'' {\em Annual Review of Contr., Robotics, Autom. Sys.}, pp.~333--358, 2018.

\bibitem{prajna2007framework}
S.~Prajna, A.~Jadbabaie, and G.~J. Pappas, ``A framework for worst-case and stochastic safety verification using barrier certificates,'' {\em IEEE Trans. Autom. Contr.}, pp.~1415--1428, 2007.

\bibitem{panagou2015distributed}
D.~Panagou, D.~M. Stipanovi{\'c}, and P.~G. Voulgaris, ``Distributed coordination control for multi-robot networks using {Lyapunov}-like barrier functions,'' {\em IEEE Trans. Autom. Contr.}, pp.~617--632, 2015.

\bibitem{tee2009barrier}
K.~P. Tee, S.~S. Ge, and E.~H. Tay, ``Barrier {Lyapunov} functions for the control of output-constrained nonlinear systems,'' {\em Automatica}, pp.~918--927, 2009.

\bibitem{ames2014control}
A.~D. Ames, J.~W. Grizzle, and P.~Tabuada, ``Control barrier function based quadratic programs with application to adaptive cruise control,'' in {\em Proc. Conf. Dec. Contr.}, pp.~6271--6278, 2014.

\bibitem{ames2016control}
A.~D. Ames, X.~Xu, J.~W. Grizzle, and P.~Tabuada, ``Control barrier function based quadratic programs for safety critical systems,'' {\em IEEE Trans. Autom. Contr.}, pp.~3861--3876, 2016.

\bibitem{ames2019control}
A.~D. Ames, S.~Coogan, M.~Egerstedt, G.~Notomista, K.~Sreenath, and P.~Tabuada, ``Control barrier functions: Theory and applications,'' in {\em Proc. Europ. contr. conf.}, pp.~3420--3431, 2019.

\bibitem{wabersich2022predictive}
K.~P. Wabersich and M.~N. Zeilinger, ``Predictive control barrier functions: Enhanced safety mechanisms for learning-based control,'' {\em IEEE Trans. Autom. Contr.}, vol.~68, no.~5, pp.~2638--2651, 2022.

\bibitem{seiler2021control}
P.~Seiler, M.~Jankovic, and E.~Hellstrom, ``Control barrier functions with unmodeled input dynamics using integral quadratic constraints,'' {\em IEEE Contr. Sys. Lett.}, vol.~6, pp.~1664--1669, 2021.

\bibitem{breeden2023robust}
J.~Breeden and D.~Panagou, ``Robust control barrier functions under high relative degree and input constraints for satellite trajectories,'' {\em Automatica}, vol.~155, p.~111109, 2023.

\bibitem{wieland2007constructive}
P.~Wieland and F.~Allg{\"o}wer, ``Constructive safety using control barrier functions,'' {\em IFAC Proc.}, vol.~40, no.~12, pp.~462--467, 2007.

\bibitem{tan2021high}
X.~Tan, W.~S. Cortez, and D.~V. Dimarogonas, ``High-order barrier functions: Robustness, safety, and performance-critical control,'' {\em IEEE Trans. Autom. Contr.}, vol.~67, no.~6, pp.~3021--3028, 2021.

\bibitem{nguyen2016exponential}
Q.~Nguyen and K.~Sreenath, ``Exponential control barrier functions for enforcing high relative-degree safety-critical constraints,'' in {\em Proc. Amer. Contr. Conf.}, pp.~322--328, IEEE, 2016.

\bibitem{xiao2021high}
W.~Xiao and C.~Belta, ``High-order control barrier functions,'' {\em IEEE Trans. Autom. Contr.}, vol.~67, no.~7, pp.~3655--3662, 2021.

\bibitem{zeng2021safety}
J.~Zeng, B.~Zhang, and K.~Sreenath, ``Safety-critical model predictive control with discrete-time control barrier function,'' in {\em Proc. Amer. Contr. Conf.}, pp.~3882--3889, IEEE, 2021.

\bibitem{taylor2020adaptive}
A.~J. Taylor and A.~D. Ames, ``Adaptive safety with control barrier functions,'' in {\em Proc. Amer. Contr. Conf.}, pp.~1399--1405, IEEE, 2020.

\bibitem{srinivasan2020synthesis}
M.~Srinivasan, A.~Dabholkar, S.~Coogan, and P.~A. Vela, ``Synthesis of control barrier functions using a supervised machine learning approach,'' in {\em Int. Conf. Int. Robots and Sys.}, pp.~7139--7145, IEEE, 2020.

\bibitem{abuaish2023geometry}
A.~Abuaish, M.~Srinivasan, and P.~A. Vela, ``Geometry of radial basis neural networks for safety biased approximation of unsafe regions,'' in {\em Proc. Amer. Contr. Conf.}, pp.~1459--1466, 2023.

\bibitem{long2021learning}
K.~Long, C.~Qian, J.~Cort{\'e}s, and N.~Atanasov, ``Learning barrier functions with memory for robust safe navigation,'' {\em IEEE Robotics and Autom. Lett.}, vol.~6, no.~3, pp.~4931--4938, 2021.

\bibitem{backupautomatica}
P.~Rabiee and J.~B. Hoagg, ``Soft-minimum and soft-maximum barrier functions for safety with actuation constraints,'' {\em Automatica}, vol.~171, no.~111921, pp.~1--11, 2025.

\bibitem{rabiee2024closed}
P.~Rabiee and J.~B. Hoagg, ``A closed-form control for safety under input constraints using a composition of control barrier functions,'' {\em arXiv preprint arXiv:2406.16874}, 2024.

\bibitem{compositionACC}
P.~Rabiee and J.~B. Hoagg, ``Composition of control barrier functions with differing relative degrees for safety under input constraints,'' in {\em Proc. Amer. Contr. Conf.}, 2024.

\bibitem{lindemann2018control}
L.~Lindemann and D.~V. Dimarogonas, ``Control barrier functions for signal temporal logic tasks,'' {\em IEEE Contr. Sys. Lett.}, vol.~3, no.~1, pp.~96--101, 2018.

\bibitem{srinivasan2018control}
M.~Srinivasan, S.~Coogan, and M.~Egerstedt, ``Control of multi-agent systems with finite time control barrier certificates and temporal logic,'' in {\em Proc. Conf. Dec. Contr.}, pp.~1991--1996, IEEE, 2018.

\bibitem{glotfelter2017nonsmooth}
P.~Glotfelter, J.~Cort{\'e}s, and M.~Egerstedt, ``Nonsmooth barrier functions with applications to multi-robot systems,'' {\em IEEE Contr. Sys. Lett.}, vol.~1, no.~2, pp.~310--315, 2017.

\bibitem{glotfelter2019hybrid}
P.~Glotfelter, I.~Buckley, and M.~Egerstedt, ``Hybrid nonsmooth barrier functions with applications to provably safe and composable collision avoidance for robotic systems,'' {\em IEEE Robotics and Autom. Lett.}, vol.~4, no.~2, pp.~1303--1310, 2019.

\bibitem{long2024sensor}
K.~Long, Y.~Yi, Z.~Dai, S.~Herbert, J.~Cort{\'e}s, and N.~Atanasov, ``Sensor-based distributionally robust control for safe robot navigation in dynamic environments,'' {\em arXiv preprint arXiv:2405.18251}, 2024.

\bibitem{lutkus2024incremental}
P.~Lutkus, D.~Anantharaman, S.~Tu, and L.~Lindemann, ``Incremental composition of learned control barrier functions in unknown environments,'' {\em arXiv preprint arXiv:2409.12382}, 2024.

\bibitem{cortez2022compatibility}
W.~S. Cortez and D.~V. Dimarogonas, ``On compatibility and region of attraction for safe, stabilizing control laws,'' {\em IEEE Trans. Autom. Contr.}, vol.~67, no.~9, pp.~4924--4931, 2022.

\bibitem{safari2023time}
A.~Safari and J.~B. Hoagg, ``Time-varying soft-maximum control barrier functions for safety in an a priori unknown environment,'' in {\em Proc. Amer. Contr. Conf.}, 2024.

\bibitem{blanchini2008set}
F.~Blanchini {\em et~al.}, {\em Set-theoretic methods in control}.
\newblock Springer, 2008.

\bibitem{khalil2002control}
H.~K. Khalil, {\em Control of nonlinear systems}.
\newblock Prentice Hall, 2002.

\bibitem{de2002control}
A.~De~Luca, G.~Oriolo, and M.~Vendittelli, ``Control of wheeled mobile robots: An experimental overview,'' {\em RAMSETE: articulated and mobile robotics for services and technologies}, pp.~181--226, 2002.

\bibitem{lindqvist2021exploration}
B.~Lindqvist, A.-A. Agha-Mohammadi, and G.~Nikolakopoulos, ``Exploration-{R}{R}{T}: A multi-objective path planning and exploration framework for unknown and unstructured environments,'' in {\em Proc. Int. Conf. Intelligent Robots and Systems}, pp.~3429--3435, 2021.

\bibitem{beard2008quadrotor}
R.~W. Beard, ``Quadrotor dynamics and control rev 0.1,'' {\em Faculty Publications Brigham Young University}, 2008.

\end{thebibliography}



\end{document}